\documentclass[accepted]{uai2026} % after acceptance, for a revised version; 
\usepackage[american]{babel}
\usepackage{natbib} % has a nice set of citation styles and commands
\usepackage{booktabs} % commands to create good-looking tables
\usepackage{tikz} % nice language for creating drawings and diagrams
\usepackage{multirow}
\usepackage{bbm}
\usepackage{pifont}
\usepackage{tabularx}
\usepackage[most]{tcolorbox}
\usepackage{wrapfig}
\usepackage{enumitem}
\usepackage{amsthm}
\usepackage{mathtools}
\usepackage{subfigure}
\usepackage{float}
\usepackage[utf8]{inputenc} % allow utf-8 input
\usepackage[T1]{fontenc}    % use 8-bit T1 fonts
\usepackage{hyperref}       % hyperlinks
\usepackage{url}            % simple URL typesetting
\usepackage{amsfonts}       % blackboard math symbols
\usepackage{nicefrac}       % compact symbols for 1/2, etc.
\usepackage{microtype}      % microtypography
\usepackage{xcolor}         % colors
\usepackage{graphicx}       % graphics
\usepackage{amsmath}   
\usepackage{algorithm}
\usepackage{algpseudocode}
\usepackage{dsfont}
\usepackage{makecell}
\usepackage{adjustbox}
\usepackage{soul}
\theoremstyle{plain} 
\usepackage{colortbl}
\usepackage[dvipsnames]{xcolor}

\newtheorem{theorem}{Theorem}
\newtheorem{definition}{Definition}
\newtheorem{corollary}{Corollary}[theorem]
\newtheorem{lemma}{Lemma}

\newtheorem{remark}{Remark}
\newtheorem{claim}{Claim}
\DeclareMathOperator*{\argminA}{arg\,min}
\DeclareMathOperator*{\argmaxA}{arg\,max}
\definecolor{lightblue}{RGB}{160,200,255}

\title{Provably Efficient Personalized Multi-Objective Bandits with Proactive Conversational Queries}

\author[1]{Linfeng Cao}
\author[2]{Ming Shi}
\author[1,3]{Ness B. Shroff}
\affil[1]{%
    Department of Computer Science and Engineering\\
    The Ohio State University
}
\affil[2]{%
    Department of Electrical Engineering\\
    University at Buffalo
}
\affil[3]{%
    Department of Electrical and Computer Engineering\\
    The Ohio State University
  }
  
\begin{document}
\maketitle

\begin{abstract}
Personalized decision-making in multi-objective bandits requires learning user-specific trade-offs among competing objectives. Since arm utility depends on both unknown rewards and unknown preferences, existing methods infer preferences only from utility feedback, entangling preference learning with reward exploration. In practice, however, users often reveal their priorities through proactive conversational queries (e.g., “cheap and clean hotel”), yet this structured signal is not leveraged.
We formalize a proactive query-based framework in which user queries provide structured preference signals. Modeling these signals via a Plackett–Luce subset choice model, we show that query-only learning is insufficient due to a fundamental shift-invariance barrier. To resolve this, we introduce MO-PQUCB, a hybrid algorithm that integrates query-based preference anchoring with bandit feedback through shift-invariant regularization and dual-exploration UCB.
We prove that proactive queries accelerate preference estimation and yield improved regret scaling over prior preference-aware MO-MAB methods. Under corrupted queries, we further characterize statistical limits and design a robust estimator achieving near-optimal performance when the corruption is sparse. Experiments validate both theoretical and practical gains.
\end{abstract}

\section{Introduction}

Multi-objective decision-making arises in many applications such as personalized recommendation, dialogue planning and so on, where competing objectives (e.g., cost, quality, latency) must be balanced. The multi-objective multi-armed bandit (MO-MAB) framework \citep{drugan2013designing} provides a principled model for sequential optimization under uncertainty. However, most existing MO-MAB methods focus on identifying Pareto-optimal arms \citep{drugan2013designing, turgay2018multi, lu2019multi, drugan2018covariance, balef2023piecewise} without accounting for user-specific trade-offs, leading to one-size-fits-all solutions.

Preference-aware MO-MAB formulations address this limitation by modeling user preferences explicitly \citep{cao2025provably}. In these models, each user is associated with a latent preference vector over objectives, and utility is defined as the inner product between preferences and objective rewards. However, existing methods estimate preferences solely from bandit feedback, coupling reward exploration with preference learning and limiting statistical efficiency.

In practical interactive systems, users often reveal high-level priorities through proactive conversational queries (e.g., ``cheap and clean hotel'', Fig.\ref{fig:intro}), which provide structured signals about objective trade-offs \citep{rui2025action, dao2024broadening, zhang2023user}. Nevertheless, such signals are not explicitly leveraged in existing MO-MAB frameworks.
In this work, we formalize a proactive query-based preference elicitation framework for MO-MAB. We model user queries as top-$m$ objective rankings under a Plackett--Luce subset choice model parameterized by the user's latent preference. We first show that query-only preference learning suffers from a fundamental shift-invariance identifiability barrier, making additional feedback necessary for optimal decision-making. 
To overcome this limitation, we propose \textbf{MO-PQUCB}, a hybrid algorithm that integrates query-based preference anchoring with bandit feedback via shift-invariant regularization and a dual-exploration UCB strategy. This design leverages structured query feedback to accelerate preference estimation while retaining principled exploration of objective rewards.

Our contributions are summarized as follows:
\begin{itemize}[leftmargin=*]
\item \textbf{Framework.} We introduce a proactive preference-aware MO-MAB formulation that incorporates structured query feedback. Our algorithm integrates query-based preference estimation with bandit learning to shrink confidence regions over personalized utilities.

\item \textbf{Theory.} We characterize the statistical benefit of proactive queries and prove that MO-PQUCB achieves near-optimal regret, improving regret scaling of $\sqrt{\log T}$ over prior preference-aware MO-MAB methods.

\item \textbf{Robustness.} We analyze learning under corrupted queries, establish a fundamental lower bound identifying the learnability regime, and design a robust estimator achieving near-optimal performance under sparse corruption.

\item \textbf{Empirical validation.} Experiments on synthetic and real-world datasets demonstrate improved personalization and robustness, and integration with large language models highlights practical applicability.
\end{itemize}

\begin{figure}[t]
    \centering    
    \includegraphics[width=0.98\linewidth]{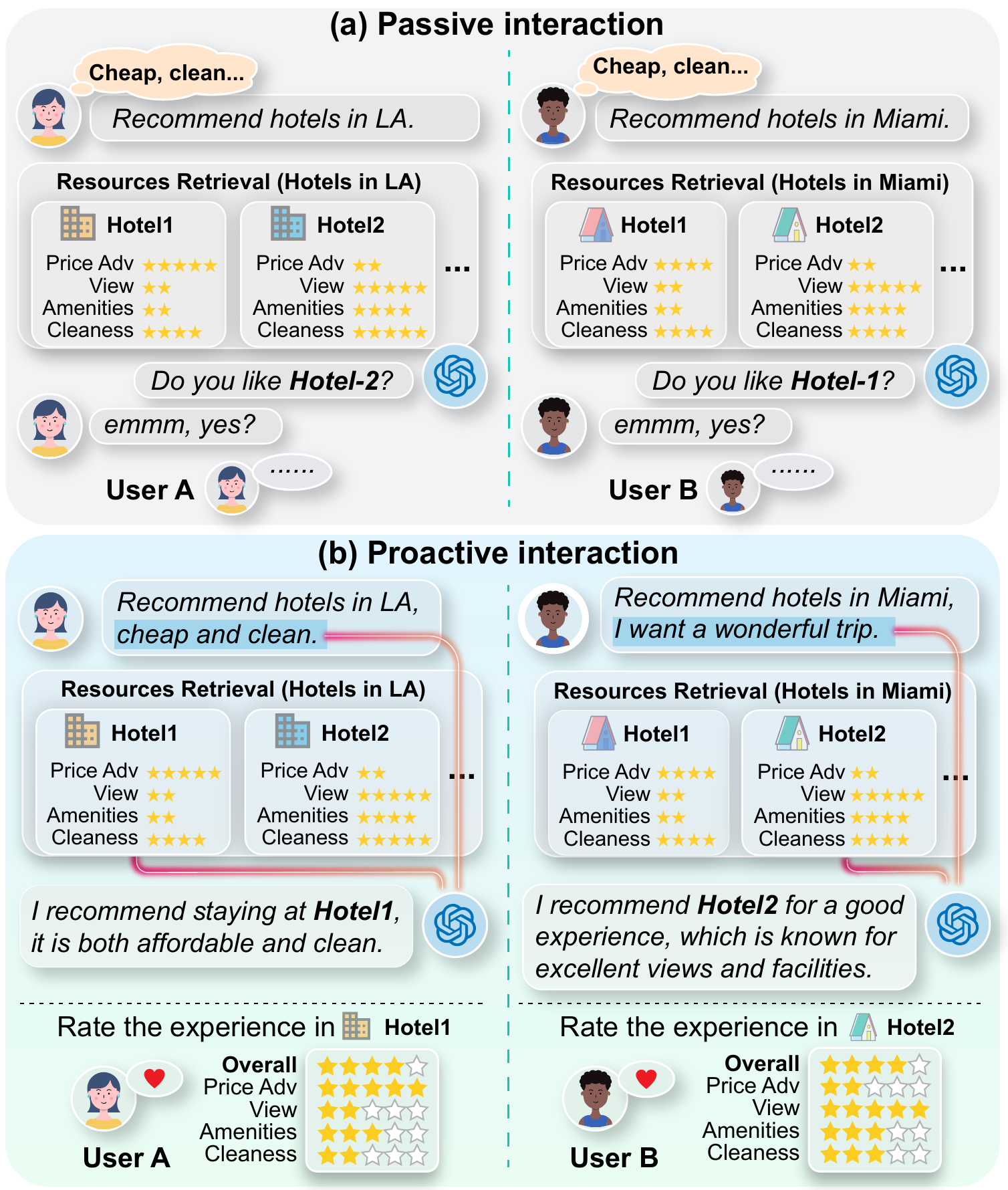}
    \caption{
Illustration of interaction styles in conversational recommendation.
(a) Passive interaction: the system proposes items or keywords and receives binary feedback, providing limited information about user trade-offs.
(b) Proactive interaction: the user expresses high-level priorities in the initial query, revealing structured preference signals that enable more informative objective trade-off inference.
}
    \label{fig:intro}
\end{figure}

\section{Related Work}

The multi-objective multi-armed bandit (MO-MAB) framework extends classical bandits to vector-valued rewards. Early work primarily focused on Pareto-optimal arm identification and Pareto regret minimization~\citep{drugan2013designing, turgay2018multi, lu2019multi, drugan2018covariance, balef2023piecewise}. While these methods provide principled global optimization guarantees, they do not account for user-specific trade-offs across objectives.

Preference-aware MO-MAB formulations incorporate user-specific trade-offs across objectives. Some works assume known structured orders (e.g., lexicographic hierarchies) and optimize without learning preferences~\citep{huyuk2021multi, cheng2024hierarchize}. Another line of work models preferences as latent vectors and defines utility as the inner product between preference and objective reward vectors, learning them jointly~\citep{cao2025provably}. In this regime, preferences are inferred solely from bandit feedback, coupling preference learning with reward exploration.

Conversational and interactive bandit methods leverage user interaction to guide exploration~\citep{christakopoulou2016towards, zhang2020conversational, xie2021comparison, zhao2022knowledge}. However, most existing frameworks are system-driven, relying on keyword prompts or binary feedback, which provide limited structural information about objective trade-offs.
In contrast, our method supports proactive user-initiated queries beyond system-driven binary interaction and achieves improved regret scaling.
% In contrast, we model proactive user queries as structured top-$m$ objective rankings under a Plackett--Luce subset choice model, enabling direct preference elicitation. 
% As summarized in Table~\ref{tab:related_work}, our framework integrates structured query feedback with bandit learning and provides regret guarantees under both clean and corrupted preference signals. 
We leave the detailed discussion and systematic comparison including regret bound with these methods in Appendix~\ref{app:related_work}.
\section{Problem Formulation}
\subsection{Preference-Aware MO-MAB}

We adopt the preference-aware MO-MAB framework proposed in \cite{cao2025provably}. 
Specifically, we consider a multi-objective multi-armed bandit (MO-MAB) setting with $K$ arms, $N$ users, and $D$ objective dimensions. At each round $t \in [T]$, user $n \in [N]$ is presented with an available arm set $\mathcal{A}_t^n \subseteq [K]$. The learner selects $a_{n,t} \in \mathcal{A}_t^n$ and observes both the objective reward vector $\boldsymbol{r}_{a_{n,t},t} \in [0,1]^D$ and the resulting scalar utility $g_{a_{n,t},t}$.

\textbf{Objective Rewards.}
For each arm $i$, the objective reward vector $\boldsymbol r_{i,t}$ is drawn i.i.d. from an unknown distribution supported on $[0,1]^D$, with mean $\boldsymbol \mu_i$.

\textbf{User Preferences.}
Each user $n$ is associated with an unknown mean preference vector $\overline{\boldsymbol{c}}_n \in \mathbb{R}^{D}$, indicating preference on $D$ objectives. At each round $t$, the instantaneous preference vector $\boldsymbol{c}_{n,t} \in [0,B]^{D}$ is drawn i.i.d. from an unknown distribution supported on $[0,B]^{D}$ with mean $\overline{\boldsymbol{c}}_n$.

\textbf{Independence.}
For all $t,n,i$, we assume the reward vector $\boldsymbol r_{i,t}$ and preference vector $\boldsymbol c_{n,t}$ are independent. This is realistic in many applications since $\boldsymbol{c}_t$ and $\boldsymbol{r}_t$ are \emph{inherently driven} by independent factors of user characteristics and arm attributes \citep{cao2025provably}.
For instance, an individual user’s personal preferences do not affect a hotel’s location, environment, or pricing, and vice versa.

\textbf{Utility.}
At round $t$, when arm $a_{n,t}$ is selected for user $n$, the observed utility is 
$g_{a_{n,t},t}^n = \langle \boldsymbol c_{n,t}, \boldsymbol r_{a_{n,t},t} \rangle$.
Under the independence assumption, we have $\mathbb{E}[g_{i,t}^n] = \langle \overline{\boldsymbol c}_{n}, \boldsymbol \mu_i \rangle$.

\textbf{Regret.}
Let $a_{n,t}^{*} = \arg\max_{i \in \mathcal A_t^n}
\langle \overline{\boldsymbol c}_{n}, \boldsymbol \mu_i \rangle$ 
denote the optimal arm for user $n$ at round $t$.
We define the cumulative regret over horizon $T$ as
\begin{equation}
\label{eq:regret_def}
% \textstyle
R(T)
= \sum_{t=1}^T \sum_{n=1}^N
\left\langle \overline{\boldsymbol c}_n,
\boldsymbol \mu_{a_{n,t}^{*}}
- \boldsymbol \mu_{a_{n,t}}
\right\rangle.
\end{equation}
The goal is to minimize $R(T)$ by efficiently matching all
users with arms that align with their individual preferences.

\subsection{Proactive Preference Query Elicitation (QE)}

At each round $t$, prior to arm selection, each user $n \in [N]$ proactively provides high-level preference feedback via query interactions. 
We represent this feedback as a top-$m_{n,t}$ ordered subset
$S_t^n = [\sigma_{n,t}(1), \ldots, \sigma_{n,t}(m_{n,t})]$,
where $m_{n,t} \le D$ may vary across rounds.
The ranking $S_t^n$ indicates the user’s most preferred $m_{n,t}$ objectives, ordered from most to least preferred
\footnote{
Our study does not fundamentally require every round query. The key quantity is the effective query information. If informative queries occur with linear frequency (i.e., $\sum_{t=1}^T m_t = \Theta(T)$), the same bound holds up to constants
}.
In practice, such rankings can be extracted from natural-language queries using a language interface (e.g., BERT or GPT).
For theoretical analysis, we assume the ranked subset $S_t^n$ is directly observed.

We model the latent full ranking $\sigma_t$ using the Plackett–Luce (PL) model \citep[Sec.~5.6.1]{marden1996analyzing} parameterized by the user’s latent mean preference vector $\overline{\boldsymbol c} \in \mathbb{R}^D$. Under the induced PL subset choice model \citep{saha2019pac}, the probability of observing the top-$m_{n,t}$ ranking $S_t^n$ is
\begin{equation}
\label{eq: preference_model}
% \textstyle
\mathbb{P}(S_t^n \mid \overline{\boldsymbol c})
=
\prod_{j=1}^{m_{n,t}}
\frac{\exp(\overline{c}_{S_t^n(j)})}
{\sum_{k \in [D] \setminus \{S_t^n(1),\ldots,S_t^n(j-1)\}}
\exp(\overline{c}_k)}.
\end{equation}
This corresponds to sequentially selecting $m_{n,t}$ objectives without replacement under the PL distribution. 
The model generalizes several classical cases: when $m_{n,t}=1$, $D=2$, it reduces to the Bradley–Terry–Luce (BTL) model \citep{bradley1952rank}; when $m_{n,t}=D$, it recovers the full ranking PL model \citep{zhu2023principled}. 
\emph{The top-$m$ formulation reflects practical settings where users disclose only their most prioritized objectives rather than a complete ranking.}

\section{Warm up: Statistical Benefit of Proactive QE}
\label{sec:warmup}

We analyze the statistical benefit of proactive query elicitation for preference estimation under the PL subset model. The analysis reveals a fundamental shift-invariance barrier and motivates our hybrid algorithm design.

\subsection{QE-Based Preference Estimator}

For each user $n \in [N]$, given the observed ranking sequence 
$\mathcal{S}_t^n = \{S_i^n\}_{i=1}^t$,
we estimate the latent preference vector under the PL subset model via classic maximum likelihood estimation (MLE), the classic algorithm in learning to rank and RLHF \citep{christiano2017deep, ouyang2022training, wang2024rlhf, xiong2024iterative}.

Let $\ell_{\mathrm{PL}}(S;\boldsymbol c)$ denote the negative log-likelihood of a top-$m$ ranking $S$ under the PL model:
\begin{equation}
\label{eq:PL_log_likehood}
\ell_{\mathrm{PL}}(S ; \boldsymbol c)
=
- \sum_{j=1}^{m}
\log
\frac{\exp(\boldsymbol c_{S(j)})}
{\sum_{k \in [D] \setminus \{S(1),\ldots,S(j-1)\}}
\exp(\boldsymbol c_k)},
\end{equation}
and the MLE estimator aims at minimizing
\[
% \textstyle
\mathcal{L}_{\mathcal{S}_t^n}(\boldsymbol c)
=
\sum_{i=1}^{t}
\ell_{\mathrm{PL}}(S_i^n ; \boldsymbol c).
\]

However, the PL model is invariant to additive shifts of the parameter vector \citep{shah2016estimation}:
for any $a \in \mathbb{R}$, $\boldsymbol c$ and $\boldsymbol c + a\mathbf{1}$ induce the same distribution. To ensure identifiability, we work with the mean-centered preference vector
$
\boldsymbol{\overline{c}}^0_n = \boldsymbol{\overline{c}}_n - \frac{\Vert \boldsymbol{\overline{c}}_n \Vert_1}{D} \boldsymbol{e}
$, and restrict estimation to the zero-mean subspace
\[
% \textstyle
\Theta_C^0
=
\left\{
\boldsymbol c \in \mathbb{R}^D :
\|\boldsymbol c\|_\infty \le B,
\ \boldsymbol c^\top \mathbf{1} = 0
\right\}.
\]
The QE-based estimator of user $n$ is therefore defined as
\begin{equation}
\label{eq:qe_estimator}
% \textstyle
\hat{\boldsymbol c}^{(\mathrm{QE})}_{n,t}
=
\arg\min_{\boldsymbol c \in \Theta_C^0}
\mathcal{L}_{\mathcal{S}_t^n}(\boldsymbol c).
\end{equation}

The following lemma establishes a convergence guarantee for the QE-based estimator to the mean-centered preference vector $\overline{\boldsymbol{c}}^0_n$. The proof is provided in Appendix \ref{sec:pf_lemma_c_QE}.

\begin{lemma}
\label{lemma:c_QE}
Let $\overline{m}_{n,t} = \frac{1}{t} \sum_{i=1}^t m_{n,i}$
denote the average number of elicited objectives up to round $t$.
For all $t \geq  C_0 \log ( T )$ with $C_0 = 2 D(16 e^{2B}+8)^2$, there exists a constant $C_1>0$ such that the QE-based estimator uniformly holds with high probability that:
\[
% \textstyle
\Vert \boldsymbol{\hat{c}}^{(\text{QE})}_{n,t} - \boldsymbol{\overline{c}}^0_n \Vert_2
\leq 
C_1 D \sqrt{ \frac{ \log (t) }{ \overline{m}_{n,t} t } }.
\] 
\end{lemma}

% \noindent\underline{\textbf{Insight 1: Efficiency of Proactive QE.}}
% Proactive queries yield an estimator with error $\widetilde{O} \big((\overline{m}_{n,t} t)^{-\frac{1}{2}} \big)$ on mean-centered preference, showing that the effective sample size scales with $\overline{m}_{n,t} t$, i.e., richer elicitation improves learning.

\noindent
\colorbox{gray!10}{%
\begin{minipage}{0.96\linewidth}
\textcolor{orange!80!black}{\textbf{Insight 1: Efficiency of Proactive QE.}}
Proactive queries yield an estimator with error 
$\widetilde{O} \big((\overline{m}_{n,t} t)^{-\frac{1}{2}} \big)$ 
on mean-centered preference, showing that the effective sample size scales with 
$\overline{m}_{n,t} t$, i.e., richer elicitation improves learning.
\end{minipage}
}

\subsection{Is QE Alone Sufficient?}
\label{sec:QE_lower_bd}

While proactive QE enables statistically efficient estimation of the
mean-centered preference vector, an important question remains:
does query feedback alone suffice for optimal decision-making in MO-MAB?

We show that the answer is negative. Query-only learning is fundamentally insufficient, as formalized below.

\begin{theorem}[Lower Bound for QE-Only Learning]
\label{theorem: lower_bd}
Consider a preference-aware MO-MAB environment with $D \ge 2$
objectives and at least two Pareto-optimal arms whose objective rewards differ.
For any bandit algorithm that relies solely on QE-based preference estimates, there exist latent preference vectors
$\{\overline{\boldsymbol c}^n\}_{n \in [N]}$
such that $R(T) = \Omega(T)$.
\end{theorem}

The proof is deferred to Appendix \ref{sec:pf_theorem_lower_bd}.
The lower bound arises from the shift-invariance of the PL model.
QE-based estimators can only recover preferences up to an additive constant.
However, MO-MAB utilities depend on the absolute preference vector: $\langle \boldsymbol c, \boldsymbol \mu_i \rangle$
and additive shifts modify arm rankings whenever
$\sum_d \boldsymbol{\mu}_i(d)$ differs across arms.
Thus, indistinguishable preference vectors under QE can induce
different optimal arms, leading to linear regret.

\noindent
\colorbox{gray!10}{%
\begin{minipage}{0.98\linewidth}
\textcolor{orange!80!black}{\textbf{Insight 2: Necessity of Additional Feedback.}}
Query-only learning is inherently non-identifiable under shift-invariance, rendering optimal decision-making impossible.
\end{minipage}
    }

% This motivates a principled hybrid framework in which bandit feedback anchors the absolute preference scale while QE provides structured relative information, together enabling identifiable learning and provably sub-linear regret.
\section{MO-PQUCB: Learning with Queries and Bandit Feedback}

The preceding analysis and insights reveal a structural tension:
while QE enables efficient relative preference estimation,
it is insufficient for optimal decision-making.
We now introduce \emph{MO-PQUCB} (Multi-Objective Proactive Query UCB), a unified framework that integrates proactive query elicitation with bandit feedback, as shown in Algorithm \ref{alg:MO_PQUCB}.

The central idea is to use QE to anchor the relative preference structure,
while leveraging bandit utility observations to calibrate
the absolute preference scale and guide exploration.
Algorithm~\ref{alg:MO_PQUCB} consists of two tightly coupled components:
(1) a collaborative preference estimator that regularizes bandit-based
preference learning using QE anchors, and
(2) a preference-aware UCB policy that balances uncertainty
in both reward and preference estimation.

\begin{algorithm}[t]
\caption{MO-PQUCB}
\label{alg:MO_PQUCB}
\begin{algorithmic}[1]

\State \textbf{Input:} $\alpha \in (0,1)$, $\lambda > 0$, $\{\beta_t\}$, $\rho$

\State \textbf{$\triangleright$ Initialization:}
\For{each user $n \in [N]$}
    \State $\boldsymbol U_\lambda \gets (\lambda + 1)\boldsymbol I - \frac{1}{D}\boldsymbol 1 \boldsymbol 1^\top$;
    \State $\boldsymbol V_{n,0} \gets (1-\alpha)\boldsymbol U_\lambda$; 
    \State $\mathcal S_0^n \gets \emptyset$
\EndFor

\For{each arm $i \in [K]$}
    \State $N_{i,1} \gets 0, \quad \hat{\boldsymbol r}_{i,1} \gets \boldsymbol 0, \quad \gamma_{i,1} \gets 0$
\EndFor

\For{$t = 1$ to $T$}

    \For{each interacting user $n \in [N]$}

        \State \textbf{$\triangleright$ Query Phase:}
        \State Receive user query and obtain ranked subset $S_t^n$
        \State $\mathcal S_t^n \gets \mathcal S_{t-1}^n \cup \{S_t^n\}$

        \State \textbf{$\triangleright$ Preference Estimation:}
        % \State $\hat{\boldsymbol c}_t^{n(\text{QE})}
        % \gets \arg\min_{\boldsymbol c \in \Theta_C^0}
        % \mathcal L_{\mathcal S_t^n}(\boldsymbol c)$
        \State Estimate preference anchor $\hat{\boldsymbol c}_{n,t}^{(\text{QE})}$ by Eq. \eqref{eq:qe_estimator} 
        \label{alg:MO_PQUCB_PE}
        \State Update collaborative preference estimate $\boldsymbol{\hat{c}}_{n,t}^{(\text{HB})}$ by \eqref{eq:c_BF}

        \State \textbf{$\triangleright$ Arm Selection:}
        \State Select $a_{n,t}$ with dual-exploration UCB policy by \eqref{eq:arm_policy}

        \State Observe $\boldsymbol r_{a_{n,t},t}$ and $g_{a_{n,t},t}^n$

        \State \textbf{$\triangleright$ Bandit Sufficient Statistics Update:}
        \State $\boldsymbol V_{n,t+1} \gets 
        \boldsymbol V_{n,t} + \alpha\, \boldsymbol r_{a_{n,t},t}\boldsymbol r_{a_{n,t},t}^\top$
        % \State $\boldsymbol b_{n,t+1} \gets 
        % \boldsymbol b_{n,t} + g_{a_{n,t},t}^n \boldsymbol r_{a_{n,t},t}$

    \EndFor

    \State \textbf{$\triangleright$ Arm Statistics Update:}
    \For{each arm $i \in [K]$}
        \State Update $N_{i,t+1}$ and $\hat{\boldsymbol r}_{i,t+1}$ by \eqref{eq:N_number} and \eqref{eq:r_estimate}
        \State $\gamma_{i,t+1} \gets
        \sqrt{\frac{\log(t/\rho)}{\max\{1,N_{i,t+1}\}}}$
    \EndFor

\EndFor
\end{algorithmic}
\end{algorithm}

\subsection{QE-Anchored Preference Estimator}

The QE estimate $\hat{\boldsymbol c}_t^{(\text{QE})}$
identifies preferences only up to an additive constant.
To recover the full preference vector,
we incorporate bandit feedback through a regularized least-squares formulation, using the QE estimate as a structural regularization anchor.

For each user $n$, we define the hybrid estimator
\[
\begin{aligned}
\hat{\boldsymbol c}_{n,t}^{(\text{HB})}
&=
\arg\min_{\boldsymbol c \in \mathbb{R}^D}
\Big(
  \alpha
  \sum_{\ell=1}^{t-1}
  \big(
    \langle \boldsymbol c, \boldsymbol r_{a_\ell^n,\ell} \rangle
    - g_{a_\ell^n,\ell}^n
  \big)^2
\\
&\qquad
  +
  (1-\alpha)
  \big\|
    \boldsymbol c
    - \hat{\boldsymbol c}_{n,t}^{(\text{QE})}
  \big\|_{\boldsymbol U_\lambda}^2
\Big).
\end{aligned}
\]
where $\alpha \in (0,1)$ balances the alignment trade-off between bandit driven feedback and QE anchoring, and
\[
\boldsymbol U_\lambda
=
\boldsymbol U + \lambda \boldsymbol I,
\quad
\boldsymbol U
=
\boldsymbol I - \frac{1}{D} \boldsymbol 1 \boldsymbol 1^\top.
\]
Here $\boldsymbol U$ is the projector onto the subspace orthogonal to $\boldsymbol 1$.

% \paragraph{Structural Design.}
The key structural design lies in aligning the regularization geometry 
with the identifiability structure induced by QE. 
Since QE identifies preferences only up to an additive shift 
along $\boldsymbol 1$, the projector $\boldsymbol U$ 
restricts regularization to the orthogonal subspace, 
while leaving the shift direction unconstrained. 
Consequently, bandit feedback uniquely determines 
the previously unidentifiable shift component, 
leading to consistent recovery of the full preference vector.
The optimization admits the closed-form solution:
\begin{equation}
\label{eq:c_BF}
% \textstyle
\boldsymbol{\hat{c}}^{(\text{HB})}_{n,t} = \boldsymbol{V}_{n,t}^{-1} \big( \alpha \sum_{\ell=1}^{t-1} g^n_{a_{n,\ell}, \ell} \boldsymbol{r}_{a_{n,\ell},\ell} + (1-\alpha) \boldsymbol{U}_{\lambda} \boldsymbol{\hat{c}}^{(\text{QE})}_{n,t} \big),
\end{equation}
where $\boldsymbol{V}_{n,t} = \alpha \sum_{\ell=1}^{t-1} \boldsymbol{r}_{a_{n,\ell},\ell} \boldsymbol{r}_{a_{n,\ell},\ell}^{\top} + (1-\alpha) \boldsymbol{U}_{\lambda}$.

We characterize the estimation error of the QE-anchored hybrid
preference estimator as follows.
\begin{lemma}
% [UCB of preference estimate error]
\label{lemma:ucb_C_BF}
For any user $n \in [N]$, for all $t \geq  C_0 \log\big(\frac{T}{\rho}\big)$, the estimation error of the proposed QE-anchored preference estimator (Eq.\ref{eq:c_BF}) 
admits a high-probability upper confidence bound with some $C_2 > 0$:
\[
\begin{aligned}
% \textstyle
& \Vert \boldsymbol{\hat{c}}^{(\text{HB})}_{n,t} - \overline{\boldsymbol{c}}_n \Vert_{ \boldsymbol{V}_{n,t}} \\
& \quad \leq C_2 \Bigg(
\underbrace{\alpha DB\sqrt{D \log\big(\alpha Dt \big)}}_{\text{bandit noise term}}
+ 
\underbrace{D \sqrt{ \frac{ (1-\alpha) \log(t)}{\overline{m}_{n,t} t} }}_{\text{QE anchoring noise term}}
\Bigg).
\end{aligned}
\]
\end{lemma}

The proof is deferred to Appendix \ref{sec:pf_lemma_ucb_C_BF}.
This result shows that the hybrid estimator improves over prior methods by leveraging proactive QE. The first term, scaled by $\alpha$, grows as $O(\sqrt{\log t})$, matching the confidence bounds in standard conversational contextual bandits \citep{zhang2020conversational, zhao2022knowledge} and preference-aware MO-MAB \citep{cao2025provably}. In contrast, the QE-based term shrinks at rate $\tilde{O} (1 / (\sqrt{\overline{m}_{n,t} t}))$, enabling faster shrinkage of confidence region. 
The parameter $\alpha$ thus governs the trade-off between proactive QE anchoring and bandit-driven feedback, directly controlling how the algorithm balances exploration and preference regularization. Such trade-off can be optimized by choosing $\alpha$ (and $\lambda$) to achieve near-optimal regret (see Corollary \ref{cor:optimized_regret}).

\subsection{Dual-Exploration UCB Policy}

Given the hybrid preference estimate
$\boldsymbol{\hat{c}}^{(\text{HB})}_{n,t}$ for each user $n$,
our goal is to construct an upper confidence bound (UCB) policy
for the expected utility
$\langle \overline{\boldsymbol{c}}_{n}, \boldsymbol{\mu}_i \rangle$
of each arm $i \in [K]$.

\noindent \textbf{Dual exploration challenge.}
Preference-aware MO-MAB exhibits a fundamental
\emph{global–local exploration} dilemma~\citep{cao2025provably}.
For each user $n$, the learner must simultaneously:
\textbf{(1)} \emph{ Global exploration:} sample diverse arms to ensure that the Gram matrix $\boldsymbol V_{n,t}$ is well-conditioned, enabling accurate preference recovery; 
\textbf{(2)} \emph{Local exploration:} repeatedly sample a specific arm to reduce uncertainty in its reward estimate.
% \begin{itemize}
%     \item \textbf{Global exploration:}
%     sample diverse arms to ensure that the Gram matrix
%     $\boldsymbol V_{n,t}$ is well-conditioned,
%     enabling accurate preference recovery;
%     \item \textbf{Local exploration:}
%     repeatedly sample a specific arm to reduce uncertainty
%     in its reward estimate and tighten its utility confidence bound.
% \end{itemize}

Thus, the UCB must explicitly account for uncertainty in both reward estimation and preference estimation.
Inspired by the analysis of PRUCB \cite{cao2025provably}, we adopt their dual-exploration framework as our bandit backbone (i.e., dual-exploration policy eq.\eqref{eq:arm_policy} with the corresponding reward and preference uncertainties designs).

\noindent \textbf{Objective reward estimation.}
We estimate the objective reward of each arm $i \in [K]$ via the collaborative empirical mean across all users:
\begin{equation}
\label{eq:r_estimate}
% \textstyle
\hat{\boldsymbol{r}}_{i,t}
=
\frac{
\sum_{j=1}^{t-1}
\sum_{n \in [N]}
\boldsymbol{r}_{a_j^n,j}
\mathds{1}_{\{a_j^n = i\}}
}{
\max\{1, N_{i,t}\}
},
\end{equation}
\begin{equation}
N_{i,t}
=
\sum_{j=1}^{t-1}
\sum_{n \in [N]}
\mathds{1}_{\{a_j^n = i\}},
\label{eq:N_number}
\end{equation}
% \textstyle
where
$N_{i,t}$
is the number of times arm $i$ selected up to round $t-1$.

\noindent \textbf{Utility confidence bound.}
By the concentration of the reward estimator
and the confidence region of the hybrid preference estimator,
we can then bound the utility estimate as follows. The proof is deferred to Appendix \ref{sec:pf_lemma_g_estimator_upper_conf_bd}.

\begin{lemma}
\label{lemma:g_estimator_upper_conf_bd}
For any user $n \in [N]$, any arm $i \in [K]$,
with probability at least $1 - \frac{D\rho^2\pi^2}{3}$, it holds uniformly for all
$t \ge C_0 \log\!\left(\frac{T}{\rho}\right)$ that
\[
\langle \overline{\boldsymbol c}_n, \boldsymbol \mu_i \rangle
\le
\langle \hat{\boldsymbol c}_{n,t}^{\text{(HB)}}, \hat{\boldsymbol r}_{i,t} \rangle
+
B_{i,t}^{n(r)}
+
B_{i,t}^{n(c)}.
\]

Here the two bonus terms are defined as follows.

\textbf{Reward uncertainty.}
\[
% \textstyle
B_{i,t}^{n(r)}
=
\gamma_{i,t}
\|\hat{\boldsymbol c}_{n,t}^{\text{(HB)}}\|_1,
\quad
\gamma_{i,t}
=
\sqrt{
\log(t/\rho)/\max\{1, N_{i,t}\}
}.
\]
This term captures uncertainty arising from estimation of
the arm reward vector.

\textbf{Preference uncertainty.}
\[
B_{i,t}^{n(c)}
=
\beta_t^n
\left\|
\hat{\boldsymbol r}_{i,t}
+
\gamma_{i,t}\boldsymbol 1
\right\|_{\boldsymbol V_{n,t}^{-1}},
\]
where
$\beta_t^n
=
C_2
\left(
\alpha \sqrt{D^3 \log(\alpha D t)}
+
D \sqrt{
\frac{(1-\alpha)\log t}
{\overline m_{n,t}\, t}
}
\right)$,
denotes the radius of the confidence region for
$\hat{\boldsymbol c}_{n,t}^{\text{(HB)}}$
(see Eq.~\eqref{eq:c_ucb_size} for the explicit constant).
This term captures uncertainty from hybrid preference estimation.
\end{lemma}

\paragraph{Arm selection.}
By the principle of optimism in the face of uncertainty~\citep{auer2002finite},
the user-specific policy is
\begin{equation}
\label{eq:arm_policy}
% \textstyle
a_{n,t}
=
\argmaxA_{i \in \mathcal{A}_{t}^n}
\Big(
\langle \boldsymbol{\hat{c}}_{n,t}^{(\text{HB})}, \hat{\boldsymbol{r}}_{i,t} \rangle
+
B_{i,t}^{n(r)}
+
B_{i,t}^{n(c)}
\Big).
\end{equation}
The two bonus terms explicitly reflect the dual-exploration structure:
$B_{i,t}^{n(r)}$ promotes local exploration
by reducing arm-level reward uncertainty,
while $B_{i,t}^{n(c)}$ encourages global exploration
to improve conditioning of the preference estimator.
Together, they enable efficient learning
under heterogeneous and latent user preferences.

\subsection{Theoretical Guarantees}

We state the main regret results in a compact form, highlighting the dependence on key parameters; the full versions and proof are deferred to Appendix \ref{sec:pf_thm_user_spe_reg}.

\begin{theorem}[User-Specific Regret]
\label{thm:main_regret}
For any user $n \in [N]$, with high probability, the regret of MO-PQUCB satisfies
\[
\begin{aligned}
R^n(T)
& =
O \Big(
\underbrace{
\big(
\sqrt{\tfrac{D^3}{\overline m_{n,T}}}
+
% B\sqrt{\tfrac{\lambda}{\alpha}}
% +
BD^{\frac{3}{2}}\sqrt{\alpha \log T}
\big)
\sqrt{T \log T}
}_{\text{preference estimation error}} \\
& \quad \quad +
\underbrace{
\tfrac{\alpha BD^{3/2}}{\sqrt{\lambda}}
% \Big)
\sqrt{K T} \log T
}_{\text{reward estimation error}}
\Big).
\end{aligned}
\]
\end{theorem}
The regret decomposes into two sources of uncertainty:
(i) \emph{preference uncertainty}, arising from recovering the latent
preference vector, and 
(ii) \emph{reward uncertainty}, arising from estimating arm utilities.
The first group of terms captures the trade-off between QE anchoring and bandit-driven calibration, while the second group reflects multi-armed exploration.
The parameter $\alpha$ controls this balance and can be tuned to optimize the overall regret.

\begin{corollary}
\label{cor:optimized_regret}
By choosing $\alpha = \lambda = \Theta (\log^{-1}(T))$,
all leading multiplicative factors of $\sqrt{T \log T}$ in
Theorem~\ref{thm:main_regret} remain bounded by constants, yielding
\[
% \textstyle
R^{n}(T) = 
O \left(
\left(
\sqrt{ \frac{D^{3}}{\overline{m}_{n,T}}} + BD^{\frac{3}{2}}\sqrt{K} 
\right) \sqrt{T \log T} \right).
\]
% \vspace{-3pt}
Aggregating all users and absorbing the problem-dependent constants of $D,K$ yield the compact final regret
\[
% \textstyle
R(T) = \sum_{n=1}^{N} R^n(T) = O\left( N \sqrt{T \log T} \right).
\]
\end{corollary}

\paragraph{Discussion.}
Compared to PRUCB~\citep{cao2025provably}, which attains $O(N\sqrt{T}\log T)$ regret in preference-aware MO-MAB,
our proactive QE-based framework improves the logarithmic dependence on $T$, achieving $O(N\sqrt{T \log T})$ regret. This improvement stems from structured preference elicitation to reduce uncertainty in preference estimation.
% The regret bound $R^{n}(T)$ decomposes into two terms corresponding to reward uncertainty and preference uncertainty, matching the dual-exploration structure.
% Notably, our method achieves near-optimal regret scaling across the entire user population, improving logarithmic dependence on $T$ compared to $O(N\sqrt{T}\log T)$ bounds of PRUCB~\citep{cao2025provably} and conventional conversational bandits~\citep{zhang2020conversational, zhao2022knowledge}.

\section{Learning under Corrupted Proactive Queries}

In practical deployments, proactive query elicitation is often mediated through natural-language interaction with conversational agents (e.g., large language models). 
Inference errors, contextual ambiguity, or imperfect alignment may cause such systems to output preference rankings that deviate from the user's true intent.
To account for this realistic uncertainty, we study a more challenging setting in which the preference feedback obtained from proactive queries is partially corrupted, and analyze the robustness of MO-PQUCB under such perturbations.

\subsection{Contamination Model}
\label{sec:corruption_model}

We model query corruption via an
$\epsilon$-\emph{stochastic adversarial contamination model}.
At each round $t$ for user $n$, the observed top-$m_{n,t}$ list
$\tilde S_t^n$ is generated from a Plackett--Luce (PL)
subset choice model (Eq.~\ref{eq: preference_model})
parameterized by a possibly corrupted preference vector
$\tilde{\boldsymbol c}_{n,t}$:

With probability $1-\epsilon$,
the feedback is clean and
$\tilde{\boldsymbol c}_{n,t} = \overline{\boldsymbol c}_n$.
With probability $\epsilon$,
the preference is corrupted as
\[
\tilde{\boldsymbol c}_{n,t} = P_{n,t} \overline{\boldsymbol c}_n,
\]
where $P_{n,t}$ is an arbitrary permutation matrix over the $D$ objectives at round $t$ for user $n$.
Such corruption relabels objective dimensions, producing rankings
misaligned with the true preference vector while preserving
preference magnitudes.
We allow $\{P_{n,t}\}$ to be chosen arbitrarily and potentially adaptively
with respect to past observations.
The history of possibly corrupted observations is denoted by
$\tilde{\mathcal S}_t^{n}$.

\subsection{A General Lower Bound}

We first characterize the fundamental statistical limits of
preference estimation under $\epsilon$-corrupted PL feedback.
The following result reveals a sharp transition at $\epsilon = 1/2$, beyond which consistent estimation becomes impossible.

\begin{theorem}
\label{theorem:lw_bd_corrupt}
Assume each query reveals the full objective ranking ($m = D$)
and $L$ corrupted ranking samples are observed.
Then for any estimator $\hat{\boldsymbol c}_L$,
\[
\inf_{\hat{\boldsymbol c}_L}
\sup_{\overline{\boldsymbol c}^0 \in \Theta_C^0}
\mathbb{E}
\big[
\|\hat{\boldsymbol c}_L - \overline{\boldsymbol c}^0\|_2
\big]
\;\ge\;
\begin{cases}
\displaystyle
\Omega(1),
\quad \epsilon \geq \frac{1}{2}- \frac{2}{D\sqrt{L}}, \\[10pt]
\Omega\!\left(
\frac{1}{(1-2\epsilon)\sqrt{D L}}
\right),
\quad \text{else}.
\end{cases}
\]
\end{theorem}

% \paragraph{Interpretation.}
The proof is provided in Appendix \ref{sec:pf_thm_lw_bd_corrupt}.
The lower bound shows that the estimation error
grows as the corruption level $\epsilon$ increases
and diverges as $\epsilon \to 1/2$.
This reflects a majority condition:
when fewer than half of the samples are clean,
the true preference vector becomes statistically indistinguishable
from permuted alternatives.

This phenomenon aligns with the impossibility result
of \citet{datar2022byzantine} under the BTL model,
which requires a majority of clean comparisons.
Our bound extends this insight to structured top-$m$
feedback under the PL model and explicitly characterizes
the dependence on $\epsilon$, $D$, and $L$.

\subsection{Robust Learning with Group-Wise Lasso}
\label{sec:robust_learning}

To address corrupted proactive queries, we propose a
\emph{group-wise Lasso maximum likelihood estimator}
that jointly learns the clean preference vector
and identifies corrupted ranking samples.
Unlike standard Lasso regularization used in scalar reward modeling~\citep{bukharin2024robust},
our method operates on structured top-$m$ rankings under the PL model
and enforces group sparsity across samples.

\paragraph{Corruption Parameterization.}
Under $\epsilon$-corrupted feedback,
each observed ranking of user $n$ at round $t$ is generated from
$\tilde{\boldsymbol c}_{n,t}
=
\overline{\boldsymbol c}_n
+
\boldsymbol{\delta}_{n,t}^*$,
where $\overline{\boldsymbol c}$ is the true preference vector and
$\boldsymbol{\delta}_{n,t}^*$ captures sample-specific corruption. 
For clean samples, $\boldsymbol{\delta}_{n,t}^* = \mathbf{0}$.
Since the corruption is induced by a permutation, we naturally impose the following restriction:
\[
\boldsymbol{\delta}_{n,t}^*
\in
\Theta_\delta^0
=
\{\boldsymbol{\delta} \in \mathbb R^D :
\boldsymbol{\delta} \perp \mathbf{1},\;
\|\boldsymbol{\delta}\|_\infty \le B\}.
\]

\paragraph{Group-Wise Lasso Estimator.}
Given corrupted observations $\tilde{\mathcal S}_t^{n}$,
to robustly recover the clean preference parameter
$\overline{\boldsymbol c}_n$,
we jointly estimate
$(\boldsymbol c, \boldsymbol\delta_{1:t})$
by solving the penalized likelihood problem
\[
% \textstyle
\tilde{\mathcal L}_{\tilde{\mathcal S}_t^{n}}(\boldsymbol c,\boldsymbol\delta_{1:t})
=
\sum_{j=1}^{t}
\ell_{\mathrm{PL}}\big(
\tilde S_j^n ; \boldsymbol c + \boldsymbol\delta_j
\big)
+
\eta \sum_{j=1}^{t}
\|\boldsymbol\delta_j\|_2,
\]
where $\ell_{\mathrm{PL}}(\cdot)$ denotes the PL negative log-likelihood
(see Eq.~\eqref{eq:PL_log_likehood})
and $\eta>0$ controls the regularization strength.
The robust estimator is defined as
\begin{equation}
\label{eq:loss_corrupt}
% \textstyle
(\hat{\boldsymbol c}^{\widetilde{\text{(QE)}}}_{n,t}, \hat{\boldsymbol\delta}_{n,1:t})
=
\argminA_{(\boldsymbol c,\boldsymbol\delta_{1:t}) \in
\Theta_C^0 \times [\Theta_\delta^0]^t}
\tilde{\mathcal L}_{\tilde{\mathcal S}_t^{n}}(\boldsymbol c,\boldsymbol\delta_{1:t}).
\end{equation}

\paragraph{Why Group Sparsity?}
Each $\boldsymbol\delta_j$ represents a structured corruption affecting
all $D$ objectives within ranking $j$.
The $\ell_2$ group penalty enforces sparsity across samples,
encouraging $\boldsymbol\delta_j = \mathbf{0}$ for clean feedback
while permitting non-zero corrections for corrupted rankings.
This matches the $\epsilon$-contamination model,
where only a fraction of samples are adversarial.
% The resulting objective remains convex, and the orthogonality constraint $\boldsymbol{\delta}_j \perp \mathbf{1}$ preserves identifiability under the shift-invariance of the PL model.

While Lasso-regularized MLE has been studied in RLHF under the BTL model~\citep{bukharin2024robust},
our setting is fundamentally different.
Each feedback instance is a structured top-$m$ ranking under the more general PL model, and corruption perturbs multiple objectives jointly rather than a single scalar reward.
The proposed group-wise regularization explicitly captures this structure and extends robustness guarantees to multi-objective preference learning.

\begin{lemma}
\label{lemma:qe_ucb_corrupt}
For any user $n$, if $\eta \ge \frac{D^3}{3}$,
then with probability at least
$1-\frac{(1+D+2e^2)\rho^2\pi^2}{6}$, for all
$t \ge C_0 \log(T/\rho)$, there exists a $C_3 >0$,
the group-wise Lasso estimator satisfies
\[
\|\hat{\boldsymbol c}_{n,t}^{\widetilde{\text{(QE)}}} - \overline{\boldsymbol c}^0_n\|_2^2
+
\|\mathrm{Vec}(\hat{\boldsymbol\delta}_{n,1:t}
-
\boldsymbol\delta_{1:t}^*)\|_2^2
\]
\[
\le
C_3
\left(
\frac{D^2 \eta^2}{m_{n,t}^{\downarrow}}
\Big(
\epsilon
+
\sqrt{\frac{\epsilon \log(t/\rho)}{t}}
\Big)
+
\frac{ D^3 \log(t/\rho)}{m_{n,t}^{\downarrow} t}
\right),
\]
where $\text{Vec}(\cdot): \mathbb{R}^{D \times t} \mapsto \mathbb{R}^{Dt}$ denotes the vectorized flatten operation in column-major order, $m_{n,t}^{\downarrow} = \min_{j \in [t]} m_{n,j}$.
\end{lemma}

The proof is provided in Appendix \ref{sec:pf_lemma_qe_ucb_corrupt}.
The estimation error scales linearly with the corruption rate $\epsilon$
(up to logarithmic factors).
When $\epsilon t = O(\log t)$,
the leading term reduces to
$O \big(\frac{\log t}{m_{n,t}^{\downarrow} t}\big)$,
matching the clean-sample convergence rate.
Thus the estimator remains statistically consistent
below the majority corruption threshold.

\noindent \textbf{Implementation.}
For the final bandit algorithm under corrupted queries, we adopt the same MO-PQUCB framework but replace line-\ref{alg:MO_PQUCB_PE} in Algorithm~\ref{alg:MO_PQUCB} with Eq.~\eqref{eq:loss_corrupt} for preference anchor estimation, and introduce $\eta$ as an additional hyperparameter. All other steps remain unchanged.

% Although the group-wise Lasso introduces additional corruption variables, it remains computationally tractable. In practice, we adopt a lightweight sliding-window approximation that updates only recent samples and applies an adaptive shrinkage step to scale down small corruption estimates while preserving significant deviations. This design effectively suppresses noisy feedback without heavy optimization and keeps the overall complexity linear in the number of objectives, ensuring scalability to real-world conversational recommendation systems.

% The above result shows that the joint estimation error for $\overline{\boldsymbol{c}}^0$ and $\delta_i^*$ under the group-wise Lasso MLE scales with the corruption rate $\epsilon$. Notably, when the corruption is sufficiently sparse, i.e., $\epsilon \cdot t \leq \log(t / \rho)$, the estimator still achieves a near-optimal convergence rate, matching the clean case.

% For the final bandit algorithm for corruption case, our solution is quite simple: we replace the line-5 in Algorithm \ref{alg:MO_PQUCB} with Eq. \ref{eq:loss_corrupt} for pseudo preference estimation, with an additional parameter of $\eta$ as the algorithm input, and all other lines remains the same.

\subsection{Theoretical Guarantees}
We state the main compacted results under corrupted queries with dependence on key parameters; the full versions and proofs are deferred to Appendix \ref{sec:pf_lemma_ucb_c_BF_corrupt} and \ref{sec:pf_thm_user_spe_reg_corrupt} respectively.

\begin{lemma}
\label{lemma:ucb_c_BF_corrupt}
Under the $\epsilon$-corruption model, for any user $n \in [N]$, for all $t \geq  C_0 \log (T)$, there exists a constant $C_4>0$ such that the estimation error of QE-anchored preference estimator with group-wise Lasso is bounded with high probability by:

\[
\hspace*{-5.5cm}
\|\hat{\boldsymbol c}^{(\mathrm{HB})}_{n,t}
-
\overline{\boldsymbol c}_n\|_{\boldsymbol V_{n,t}}
\]
\[
% \textstyle
\le
C_4
\Bigg(
\underbrace{
\alpha \sqrt{D^3 \log(\alpha D t)}
}_{\text{bandit noise term}}
+
\underbrace{
D \sqrt{
\frac{1-\alpha}{m_{n,t}^{\downarrow}}
\big(
\epsilon
+
\frac{D \log(t)}{t}
\big)}
}_{\text{Corrupted QE noise term}}
\Bigg).
\]
\end{lemma}

\begin{theorem}[Regret under Corrupted Queries]
\label{thm:user_spe_reg_corrupt}
Consider Algorithm~\ref{alg:MO_PQUCB}
with the group-wise Lasso estimator.
By setting
$\alpha=\lambda
=
\Theta\!\left(
\epsilon+\log^{-1}(T)
\right)$,
with high probability,
the user-specific regret satisfies
\[
% \textstyle
R^{n}(T)
=
O
\left(
\sqrt{\frac{D^3 K}{m_{n,T}^{\downarrow}}\,T\log T}
+
\sqrt{D^3 K T \epsilon}\,\log T
\right).
\]
Aggregating all users and absorbing the problem-dependent constants of $D,K$ yield the compact final regret
\[
R(T)
=
O \Big(
N
\big(
\sqrt{T\log T}
+
\sqrt{\epsilon T}\,\log T
\big)
\Big).
\]
\end{theorem}

\paragraph{Discussion.}
The regret decomposes into a clean-learning term
$O(\sqrt{T\log T})$
and a corruption-induced penalty
$O(\sqrt{\epsilon T}\log T)$.
When $\epsilon = O(1/\log T)$,
the corruption term is negligible
and the regret matches the uncorrupted setting.
Notably, as $\epsilon$ increases, $\alpha$ increases and the algorithm naturally shifts reliance from corrupted query signals to reliable bandit feedback.

% These results demonstrate that proactive QE remains beneficial even in noisy conversational environments:
% the algorithm adaptively shifts reliance from corrupted query signals to reliable bandit feedback, maintaining robustness without sacrificing efficiency.

\begin{figure*}[t]
    \subfigure[Synthetic data]{
    \label{fig:exp_syn_regret}
        \includegraphics[width=0.31\linewidth]{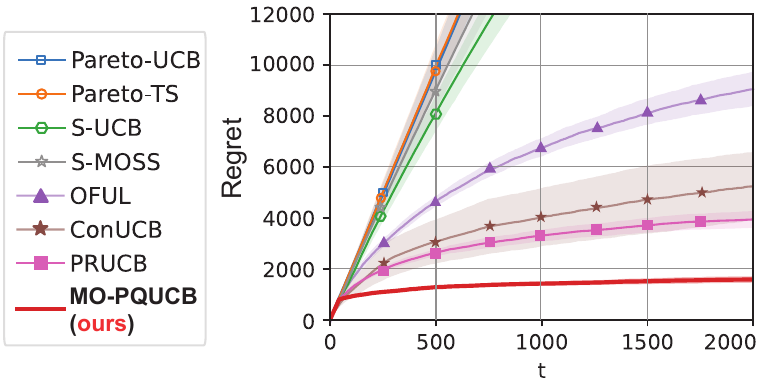}
    }
    \subfigure[TripAdvisor]{
    \label{fig:exp_trip_regret}
        \includegraphics[width=0.21\linewidth]{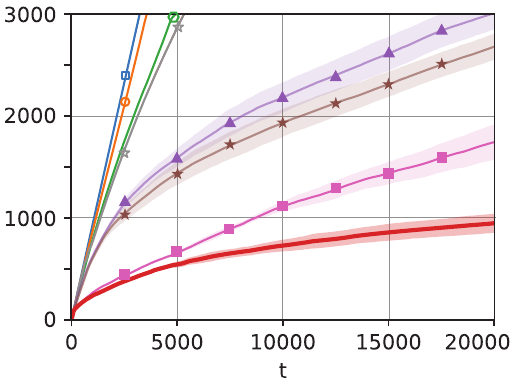}
    }
    \subfigure[BeerAdvocate]{
    \label{fig:exp_beer_regret}
        \includegraphics[width=0.21\linewidth]{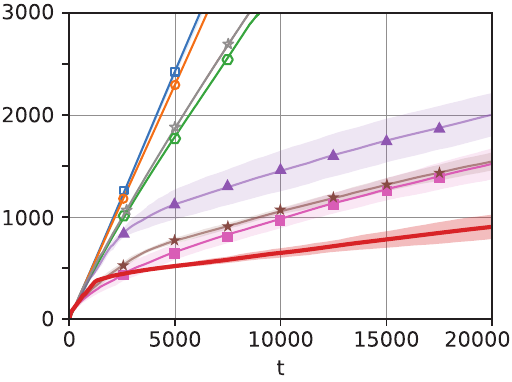}
    }
    \subfigure[Synthetic ($\epsilon$-Corruption)]{
    \label{fig:exp_syn_corrupt_regret}
        \includegraphics[width=0.21\linewidth]{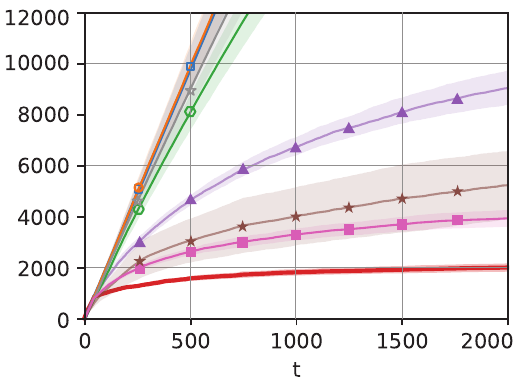}
    }
    \caption{Regret comparison on different datasets and query environments. ($t$ denotes the number of interaction rounds)}
\label{fig: exp_main_regret}
\end{figure*}

\begin{figure*}[t]
    \centering    
    \includegraphics[width=0.68\linewidth]{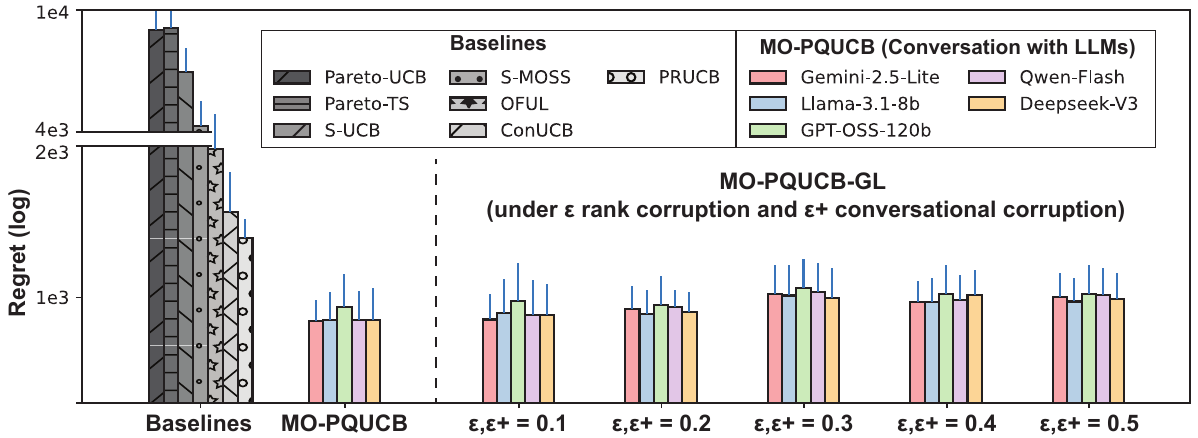}
    \caption{
    Cumulative regrets of hotel recommendation with LLMs as conversational agents on TripAdvisor.
    }
    \label{fig:exp_llm}
\end{figure*}

\section{Experiments}

We evaluate MO-PQUCB on real-world and synthetic data to validate:
(i) regret reduction via proactive preference elicitation,
(ii) robustness under corrupted queries, and
(iii) practical integration with LLM-based conversational agents\footnote{The code repository is provided at \url{https://github.com/caolinfeng/MO-PQUCB}}.

\subsection{Experimental Setup}

\textbf{Datasets.}
We consider:
(1) a synthetic MO-MAB environment with $K=40$, $D=20$, $N=20$ and $T=2000$.
(2) \textbf{TripAdvisor}: A hotel review dataset (878K reviews, 4.3K hotels). Each review contains six ($D\!=\!6$) aspect-ratings and an overall rating.
After preprocessing and filtering, we set $N \!=\!10$, $K\!=\!62$ and $T=20000$.
(3) \textbf{BeerAdvocate}: A beer review dataset (1.5M
reviews, 66K+ beers). Each review contains four ($D=4$) aspects and an overall rating. After preprocessing, we set $N \!=\!14$, $K\!=\!167$ and $T=20000$.
Detailed descriptions and setups are in Appendix~\ref{app:dataset}.

\textbf{Baselines.}
We compare our algorithm with the following multi-objective bandit algorithms:
(1) Scalarization-based methods: \textbf{S-UCB}~\citep{drugan2013designing}, \textbf{S-MOSS};
(2) Pareto-optimality methods: \textbf{Pareto-UCB}~\citep{drugan2013designing}, \textbf{Pareto-TS}~\citep{yahyaa2015thompson};
(3) Multi-objective variants of contextual and conversational bandit: \textbf{MO-OFUL}~\citep{abbasi2011improved}, \textbf{ConUCB}~\citep{zhang2020conversational};
(4) Preference-aware method: \textbf{PRUCB}~\citep{cao2025provably}.
Each experiment is repeated over 10 independent runs. By default, the query elicitation provides full-ranking feedback ($m = D$) unless stated otherwise.
We provide implementation details in Appendix~\ref{appendix:baselines}.

\subsection{Overall Performance}

The average cumulative regrets with standard error for each algorithm on synthetic and real-world datasets are shown in Fig.~\ref{fig:exp_syn_regret}, ~\ref{fig:exp_trip_regret} and ~\ref{fig:exp_beer_regret}.
MO-PQUCB consistently achieves the lowest cumulative regret and variance, outperforming all baselines. Notably, MO-PQUCB surpasses PRUCB~\citep{cao2025provably}, demonstrating that proactive query-based interaction effectively reduces cumulative regret through more informative preference elicitation. Although standard contextual and conversational bandit methods~\citep{abbasi2011improved, zhang2020conversational} exhibit sublinear regret, they perform significantly worse than preference-aware MO-MAB approaches, highlighting the advantage of dual exploration in explicitly preference-driven environments. Furthermore, traditional scalarization and Pareto-optimality methods~\citep{drugan2013designing, yahyaa2015thompson} fail to converge to personalized optimal arms, which is expected since they do not learn user preferences.

\subsection{Robustness under Corrupted Queries}

For the corrupted QE experiments, we introduce an additional $\epsilon$-corruption model defined in  Section~\ref{sec:corruption_model} with $\epsilon = 0.2$. 
The preference anchor is estimated from corrupted QE data using Eq.~\ref{eq:loss_corrupt}, which incorporates a group-wise Lasso penalty to enable robust preference learning. To distinguish this variant, we denote it as MO-PQUCB-GL.

In Fig. \ref{fig:exp_syn_corrupt_regret}, we report the regrets of different methods. Even under corruption, our MO-PQUCB-GL consistently outperforms all baselines, demonstrating its robustness and stable convergence despite noise in preference feedback. This improvement is attributed to the algorithm’s group-wise Lasso regularization and hybrid estimation scheme, which effectively suppress corrupted preference signals while still leveraging useful structural information from QE.

\subsection{LLMs as Conversational Agents}
\label{sec:llm_exp}

To evaluate practical deployment, we integrate MO-PQUCB into a conversational recommendation pipeline where large language models (LLMs) serve as preference interpreters between users and the bandit learner.

\noindent\textbf{Protocol.}
We set two LLM agents as simulators.
The Simulator-1 generates natural-language queries reflecting a user’s latent preference vector $\overline c^n$. The Simulator-2 then maps each query to a ranked subset of objectives, which is provided to MO-PQUCB as proactive feedback. 
To model realistic conversational uncertainty, we introduce two noise sources in Simulator-1: (i) ranking corruption with probability $\epsilon$, and (ii) ambiguous query generation with probability $\epsilon^{+}$. 
Detailed protocol design is provided in Appendix~\ref{app:llm_protocol}.

\noindent\textbf{Setup.}
Experiments are conducted on TripAdvisor dataset with $T=2000$. 
We evaluate representative LLMs including Gemini-2.5 \citep{comanici2025gemini}, Llama-3.1-8B \citep{dubey2024llama}, GPT-OSS-120B \citep{agarwal2025gpt}, Qwen-Flash \citep{yang2025qwen3} and Deepseek-V3 \citep{liu2024deepseek}. The implementation details are deferred to Appendix~\ref{app:llm_details}.

\textbf{Regret.}
As shown in Fig. \ref{fig:exp_llm}, all LLM-based MO-PQUCB variants achieve significantly lower regret than traditional scalarization, Pareto-optimality, and contextual bandit baselines, demonstrating the advantage of proactive query-based preference elicitation.
When introducing the $\epsilon$ and $\epsilon^{+}$ corruption model, the MO-PQUCB-GL variant shows stable performance even under increasing conversational and ranking noise. The gradual and limited increase in regret indicates that the group-wise Lasso regularization effectively compensates for corrupted preference signals.
\section{Conclusion}
In summary, we introduce a new proactive query-based preference elicitation framework for multi-objective bandit problems, enabling users to express high-level preferences directly through natural language. By modeling query-derived partial rankings with a PL subset choice model and integrating this into a UCB-style bandit algorithm, MO-PQUCB effectively balances structured preference signals and reward feedback. Our algorithm not only achieves provable regret guarantees but also demonstrates strong empirical performance. This work highlights the potential of proactive interaction in bridging language-driven user input with principled sequential decision-making.

% \begin{acknowledgements}
% This work has been supported in part by  the U.S. National Science Foundation under the grants: NSF AI Institute (AI-EDGE) 2112471, CNS2312836, CNS-2225561, and CNS-2239677, by the Office of Naval Research under grant N00014-24-1-2729, and by the Army Research Laboratory  under Cooperative Agreement Number W911NF-23-2-0225, by. The views and conclusions contained in this document are those of the authors and should not be interpreted as representing the official policies, either expressed or implied, of the Army Research Laboratory or the U.S. Government. The U.S. Government is authorized to reproduce and distribute reprints for Government purposes notwithstanding any copyright notation herein.
% \end{acknowledgements}

% References
% \setlength{\bibsep}{0pt}
\bibliography{reference}

\newpage

\onecolumn

% \title{Provably Efficient Personalized Multi-Objective Bandits with Proactive Conversational Queries}
% \maketitle

\appendix
\appendix

\tcbset{
  promptbox/.style={
    colback=gray!4,
    colframe=gray!50!black,
    boxrule=0.4pt,
    arc=2pt,
    left=5pt,right=5pt,top=4pt,bottom=4pt,
    fonttitle=\bfseries,
    coltitle=white
  },
  responsebox/.style={
    colback=blue!3,
    colframe=orange!80!black,
    boxrule=0.4pt,
    arc=2pt,
    left=5pt,right=5pt,top=4pt,bottom=4pt,
    fonttitle=\bfseries,
    coltitle=white
  }
}

\section{Detailed Related Work}
\label{app:related_work}
\subsection{Multi-Objective Multi-Armed Bandits}
The multi-objective multi-armed bandit (MO-MAB) framework extends the classical MAB setting by considering reward vectors over multiple objectives. Early works in this area focused on identifying the Pareto-optimal set of arms that cannot be strictly improved across all objectives \citep{drugan2013designing, turgay2018multi, lu2019multi, drugan2018covariance, balef2023piecewise}. These methods aim to approximate the Pareto front and often evaluate performance in terms of Pareto regret. While effective in general-purpose scenarios, such approaches overlook the individual trade-offs that users may have across objectives, thereby failing to deliver personalized solutions.

\subsection{Preference-Aware Bandits}
Several studies have extended MO-MABs with user preferences using lexicographic orders \citep{ehrgott2005multicriteria, huyuk2021multi, cheng2024hierarchize}, which assume strict objective prioritization. However, such models fail to capture real-world trade-offs.
Moreover, these methods assume the user preference to be known in advance, and the preference learning is not required.
Preference-aware approaches \citep{cao2025provably} instead model utility as the inner product between a latent preference vector and objective rewards, enabling more personalized decisions and generalizing the lexicographic ordering as a special case. Our work extends this framework by incorporating proactively elicited partial preferences.

\subsection{Language-Based Preference Elicitation for Bandits}
Recent advances in instruction-tuned language models (e.g., GPT, BERT) have enabled users to express preferences via high-level natural language queries, allowing structured preference inference through models trained for instruction following \citep{ouyang2022training, thoppilan2022lamda, zhou2023lima}. Such capabilities open new opportunities for integrating natural language elicitation into sequential decision-making. In contrast to prior conversational bandits \citep{christakopoulou2016towards, zhang2020conversational, xie2021comparison, zhao2022knowledge}, which rely on passive, system-initiated interactions with binary feedback, our approach introduces a proactive querying protocol. Related directions in language-based RL and reward learning \citep{christiano2017deep, wang2024rlhf} also highlight the value of natural language supervision. Our design enhances expressiveness and exploration efficiency, particularly in multi-objective settings where preference structures are more nuanced.

Table~\ref{tab:related_work} summarizes the MO-MAB methods most closely related to our work. As shown, our approach uniquely integrates multi-objective optimization, preference awareness, and language-driven interaction within a unified formulation, achieving both theoretical optimality and practical adaptability in realistic conversational settings.

\begin{table}[h]
\centering
\caption{
% Comparison of representative algorithms in regret, computational complexity, and functionality.
Comparison of representative algorithms in regret, computational complexity, and functionality. 
Columns denote: \textbf{Pref.} (preference awareness), \textbf{Conv.} (conversational interaction), 
\textbf{Proact.} (proactive query), and \textbf{Robust.} (robustness to corrupted preference signals). 
}
\label{tab:related_work}
\renewcommand{\arraystretch}{1.05} 
\setlength{\tabcolsep}{2pt}
\resizebox{0.8\linewidth}{!}{%
\begin{tabular}{l c c c c c c}
\toprule
\textbf{Algorithm} & \textbf{Regret} & \textbf{Pref.} & \textbf{Conv.} & \textbf{Proact.} & \textbf{Robust.} \\
\midrule
Pareto-UCB~\cite{drugan2013designing} 
& $O(T)$ & \ding{56} & \ding{56} & \ding{56} & \ding{56} \\

Pareto-TS~\cite{yahyaa2015thompson} 
& $O(T)$ & \ding{56} & \ding{56} & \ding{56} & \ding{56} \\

S-UCB~\cite{drugan2013designing}, S-MOSS 
& $O(T)$  & \ding{56} & \ding{56} & \ding{56} & \ding{56} \\

OFUL~\cite{abbasi2011improved} 
& $O(\sqrt{T}\log T)$ & \ding{52} & \ding{56} & \ding{56} & \ding{56} \\

ConUCB~\cite{zhang2020conversational} 
& $O(\sqrt{T}\log T)$ & \ding{52} & \ding{52} & \ding{56} & \ding{56} \\

PRUCB~\cite{cao2025provably} 
& $O(\sqrt{T}\log T)$ & \ding{52} & \ding{56} & \ding{56} & \ding{56} \\

\rowcolor{lightblue!35}
MO-PQUCB 
& $O(\sqrt{T\log T})$ & \ding{52} & \ding{52} & \ding{52} & \ding{56} \\

\rowcolor{lightblue!35}
MO-PQUCB-GL 
& $O(\sqrt{T\log T} + \sqrt{\epsilon T}\log T)$ & \ding{52} & \ding{52} & \ding{52} & \ding{52} \\
\bottomrule
\end{tabular}
}
\end{table}

\section{Experiments}

In this section, we describe experimental results on both synthetic data and real-world data under different environments to validate our proposed algorithm.

\subsection{Detailed Dataset Setups}
\label{app:dataset}
\noindent \textbf{Synthetic Data.}
We construct synthetic environments following \citep{cao2025provably}. Specifically, we consider a MO-MAB setting with $K = 40$ arms, each associated with a $D=20$ dimensional reward vector. For each arm $i \in [K]$, the reward on objective $d$ is drawn from a Gaussian distribution with randomized mean $\mu_i(d) \in [0,5]$ and variance $0.5$.
% To investigate the impact of query size, we fix the top-$m$ feedback size $m_t=m$ at each step and vary $m \in {1,5,10,20}$. 
We simulate $N=20$ users, where each user’s mean preference vector is randomized within $[0,5]$, and instantaneous preferences are sampled from a Gaussian distribution with variance $0.5$.
We set the time horizon to $T=2000$. At each round $t \in [T]$, a user $n_t \in [N]$ is randomly selected and presented with a subset of arms $\mathcal{A}_t^n \subset [K]$ of size $|\mathcal{A}_t^n| = \lfloor K/4 \rfloor$.

\noindent \textbf{TripAdvisor (Four-City Dataset\footnote{https://notebook.community/melqkiades/yelp/notebooks/TripAdvisor-Datasets}).}
The dataset contains 878,561 reviews (1.3 GB) from 4,333 hotels on TripAdvisor, with each review including an \emph{overall rating} and six 5-scale aspect ratings: \emph{cleanliness}, \emph{location}, \emph{rooms}, \emph{service}, \emph{sleep quality}, and \emph{value}. We treat these six aspects as the multi-objective reward dimensions (i.e., $D = 6$) and the overall rating as the utility signal.
After preprocessing by removing empty users, reviews with missing ratings, and users or hotels with fewer than 10 reviews, the dataset contains 786 reviews, 257 users, and 62 hotels. Each hotel is modeled as an arm with mean rewards given by its average aspect ratings. We select the 10 most active users ($N=10$) as the user set, with the user preferences estimated via ridge regression between aspect ratings and overall ratings. For simulation, we generate instantaneous rewards and preferences with Gaussian noise ($\sigma^2=0.5$), and set the horizon to $T=20{,}000$, where one user is randomly chosen at each step for single-city hotel recommendation.

\noindent \textbf{BeerAdvocate\footnote{https://snap.stanford.edu/data/web-BeerAdvocate.html}.}
The dataset contains over 1.5 million reviews of 66,000+ beers from various breweries. Each review provides 5-scale ratings on four aspects (\emph{appearance}, \emph{aroma}, \emph{palate}, \emph{taste}) and an \emph{overall impression}, which we treat as the utility signal, with the four aspects as the multi-objective reward dimensions ($D=4$).
After filtering users with at least 2,000 reviews and beers with at least 50 reviews, we retain 14 users and 167 beers. Each beer is modeled as an arm with mean rewards given by average aspect ratings, and user preferences are estimated via ridge regression between aspect ratings and overall impressions. In simulation, instantaneous rewards and preferences are drawn from Gaussian distributions with variance $0.5$, and the horizon is set to $T=20{,}000$, with one user randomly selected per step for recommendation.

\subsection{Baseline Details}
\label{appendix:baselines}
We select the following multi-objective algorithms as baselines to compare with ours:
\begin{itemize}[left=0pt]
    \item 
    \textbf{S-UCB}~\citep{drugan2013designing}:
    A scalarized UCB algorithm, which scalarizes the multi-dimensional reward by assigning weights to each objective and then employs the single objective UCB algorithm \citet{auer2002finite}. Throughout the experiments, we assign each objective with equal weight. 
    \item 
    \textbf{S-MOSS}:
    A scalarized UCB algorithm, which follows the similar way with S-UCB by scalarizing the multi-dimensional reward into a single one, but uses MOSS \citep{audibert2009minimax} policy for arm selection.
    \item 
    \textbf{Pareto-UCB}~\citep{drugan2013designing}:
    A Pareto-based algorithm, which compares different arms by the upper confidence bounds of their expected multi-dimensional reward by Pareto order and pulls an arm uniformly from the approximate Pareto front.
    \item 
    \textbf{Pareto-TS}~\citep{yahyaa2015thompson}:
    A Pareto-based algorithm, which makes use of the Thompson sampling technique to estimate the expected reward for every arm and selects an arm uniformly at random from the estimated Pareto front.
    \item 
    \textbf{MO-OFUL}~\citep{abbasi2011improved}: 
    A multi-objective variant of the classic OFUL algorithm for linear bandits. For MO-MAB environment adaptation, it estimate user preferences via ridge regression using both reward and overall utility feedback, and we replaces the input feature with empirical reward estimates, and applies an exploration rate of $\epsilon = 0.05$ to handle the uncertainty of reward estimates.
    \item 
    \textbf{ConUCB}~\citep{zhang2020conversational}: 
    A multi-objective variant of the traditional conversational contextual bandit algorithm based on the passive query interaction. Similar to MO-OFUL, it replaces the input feature with empirical reward estimates and uses an exploration rate of $\epsilon = 0.05$ for adaptation in preference-aware MO-MAB environment. 
    The algorithm is developed based on a passive keyword-level cue to accelerate convergence. For the keyword implementation, we set the same number of keywords $N_{kw}$ as the number of objectives (i.e., $N_{kw} = D$). For the $i$-th keyword, we relate it with 3 arms that has the largest reward value on $i$-th objective. 
    When the conversation session is triggered, the system provide a keyword and receive a binary feedback (e.g., like / dislike) from user to infer the user's preference and identify optimal arm.
    The conversational frequency is set as $5\log (T)$, which follows the optimal setting in original paper.
    \item
    \textbf{PRUCB}~\citep{cao2025provably}: 
    A recently proposed preference-aware MO-MAB algorithm based on the UCB strategy to handle the hidden user preference case.
\end{itemize}

\noindent \textbf{Devices.}
All experiments were run on a Linux system with 2 CPUs (Intel(R) Xeon(R) Gold 6238R CPU @ 2.20GHz).

\subsection{Experimental Setups of LLMs as Conversational Agents}

\subsubsection{Experimental Protocol.}
\label{app:llm_protocol}
The experimental protocol is shown in Fig. \ref{fig:exp_llm_protocol}.
Simulator-1 generates natural language queries that reflect a user’s latent preference vector $\overline{c}^{n}$ (e.g., “Please recommend affordable hotels with a clean environment and good facilities”).
To evaluate robustness, we introduce a corruption model that simulates realistic conversational noise. The corruption acts in two ways: (1) with probability $\epsilon$, the true ranking $S_t$ is permuted before being passed to LLM Simulator-1; and (2) with probability $\epsilon^{+}$, the LLM Simulator-1 generates vague or ambiguous natural-language queries. 
This setup reflects real-world interactions where users may change their minds, express conflicting priorities, or describe preferences vaguely using imprecise language, enabling systematic evaluation of MO-PQUCB’s robustness to imperfect conversational inputs.
Simulator-2 then receives the generated query and infers the corresponding objective ranking $S_t'$, which is passed to MO-PQUCB as proactive preference feedback.

\begin{figure}[t]
    \centering    
    \includegraphics[width=0.75\linewidth]{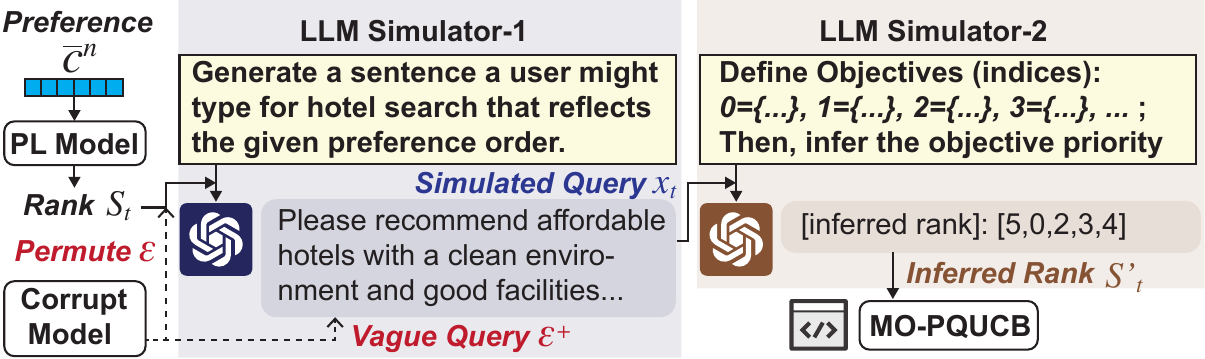}
    \caption{
    Experimental protocol of hotel recommendation with LLMs as conversational agents for simulation.
    }
    \label{fig:exp_llm_protocol}
\end{figure}

\subsubsection{Implementation}
\label{app:llm_details}
We employ several state-of-the-art LLMs as conversational agents, including Gemini-2.5-Lite \cite{comanici2025gemini}, Llama-3.1-8B \cite{dubey2024llama}, GPT-OSS-120B \cite{agarwal2025gpt}, Qwen-Flash \cite{yang2025qwen3} and Deepseek-V3 \cite{liu2024deepseek}. In each experiment, both LLM Simulator-1 and LLM Simulator-2 invoke the same model API to ensure consistent behavior between query generation and preference interpretation, with the temperature set to 0.5 to control generation randomness.
We frame our experiments as a personalized hotel recommendation task using the \emph{TripAdvisor Four-City} dataset. 
Specifically, we perform five independent runs, each run randomly selecting one active user and setting the interaction horizon $T = 2{,}000$ to regulate the query-calling rate.

\subsubsection{Prompts and Examples in Section \ref{sec:llm_exp}}

This section introduces wording of prompts used and examples in the hotel recommendation simulation with LLMs as conversational agents in Section \ref{sec:llm_exp}. Specifically, the prompt instructions for LLM simulator-1 to generate natural and human-like hotel recommendation queries are listed in Table \ref{tab:prompt_design_llm_1}, and the prompt instructions for LLM simulator-2 to infer user's preference rank are listed in Table \ref{tab:prompt_design_llm_2}.

\subsection{Additional Experimental Results}

\subsubsection{Results With Clean Queries}

\noindent \textbf{Semi-Normed Preference Estimation Error.}
We further analyze the averaged semi-normed preference estimation error, $\Vert \hat{c}_{t}^{(\text{HB})} - \overline{c} \Vert{V_t}$ on synthetic data (Fig.~\ref{fig:exp} b). As shown, the proposed MO-PQUCB consistently achieves lower and more stable estimation error compared to other methods. This result confirms that incorporating proactive preference elicitation not only provides better initial guidance for learning but also enhances convergence with bandit feedback.

\noindent \textbf{Effect of Query Elicitation Size.}
Fig.~\ref{fig:exp}(c) and (d) show the cumulative regret and averaged semi-normed preference estimation error of MO-PQUCB under different query sizes $m$, denoting the number of top-ranked objectives observed per query.
As $m$ increases, both regret and estimation error decrease, reflecting faster convergence and more accurate preference learning.
Even partial rankings ($m=5$ or $m=10$) markedly outperform minimal feedback ($m=1$), while full rankings ($m=20$) yield the best performance.
These results highlight the statistical value of proactive preference elicitation—MO-PQUCB benefits consistently from more informative queries.

\begin{figure*}[t]
    \centering    
    \includegraphics[width=\linewidth]{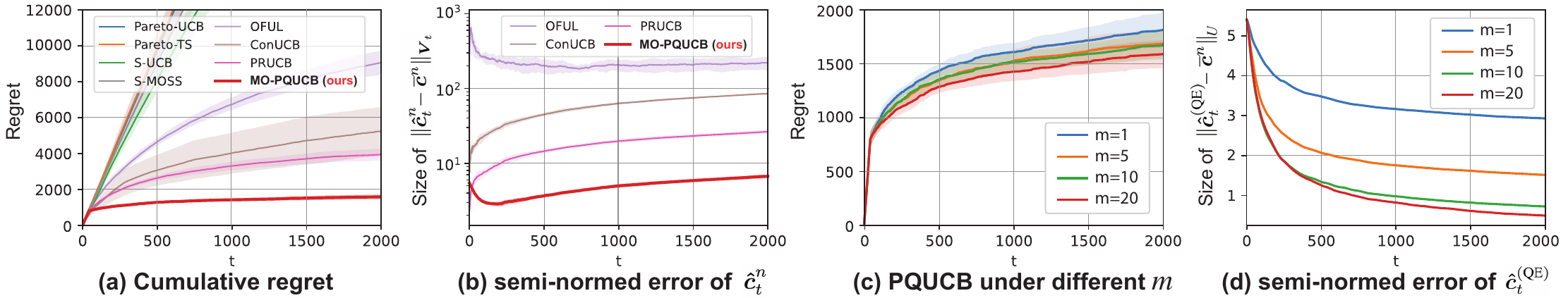}
    \caption{
    Experimental results on synthetic dataset.
    }
    \label{fig:exp}
\end{figure*}

\begin{figure*}[t]
    \centering    
    \includegraphics[width=\linewidth]{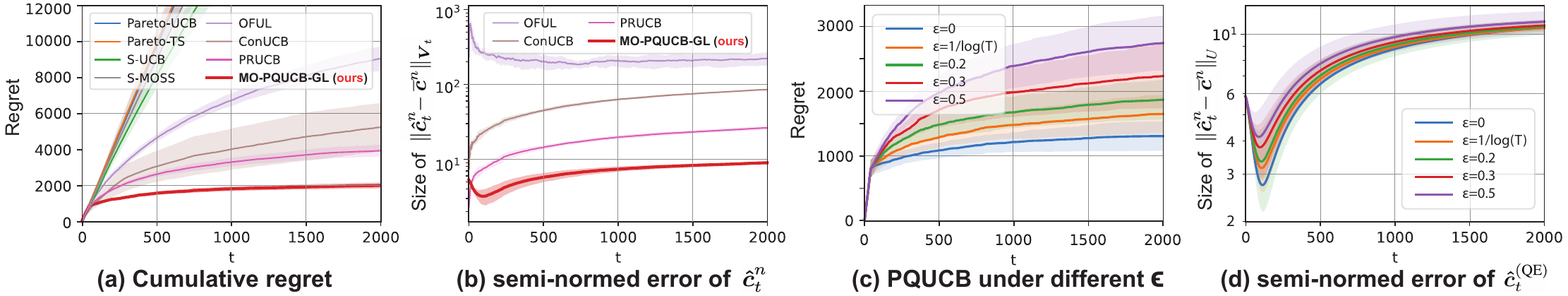}
    \caption{
    Experimental results on synthetic dataset under query corruption.
    }
    \label{fig:exp_currpt}
\end{figure*}

\subsubsection{Results under $\epsilon$-Corrupted Queries}
\noindent \textbf{Analysis on Corruption Rate.}
Fig. \ref{fig:exp_currpt}(c,d) presents the performance of MO-PQUCB under different corruption levels $\epsilon$ in the PL-based preference elicitation. As the corruption level increases, both regret and estimation error rise, reflecting the increased difficulty of accurately learning user preferences from noisy rankings. However, even at higher corruption levels ($\epsilon = O(1)$), the algorithm maintains bounded regret growth and controlled estimation error, showcasing its robustness.

\newpage

\begin{table}[H]
\centering
\caption{Wording of prompts used in the hotel recommendation simulation with LLMs as conversational agents. The pink-highlighted line indicates the additional instruction that is appended only when an $\epsilon^{+}$-vague query corruption event is triggered.}
\label{tab:prompt_design_llm_1}
\renewcommand{\arraystretch}{1.15}
\setlength{\tabcolsep}{2pt}

\begin{tabularx}{\linewidth}{@{}X@{}} 
\toprule
\textbf{LLM Simulator-1: Prompt Construction Overview} \\ 
\midrule

\textbf{(a) Static Header.}  
This prompt guides the language model to generate natural hotel-search queries that align with user-specific objective preferences. 
It explicitly constrains the output format to ensure structured data synthesis and prohibits extraneous text generation.

\begin{tcolorbox}[promptbox, title={Prompt Header}]
\small
You are helping with hotel recommendation data synthesis.  

Objectives (indices): \texttt{0=[...], 1=[...], 2=[...], 3=[...], 4=[...], 5=[...]}.

For each preference order list: Generate one natural, descriptive sentence that a user might type when searching for a hotel, reflecting the given preference order. 
\sethlcolor{pink}\hl{The sentence should be vague and ambiguous.}

\textbf{STRICT OUTPUT:} Return only JSON---a list of objects formatted as  
\texttt{\{"id": <int>, "query": <string>\}}.  
No extra text.

\medskip
\textbf{INPUT:} Given preference order list [...].
\end{tcolorbox}

\vspace{4pt}
\textbf{(b) Dynamic Input Block.}  
Specifies the preference ranking order data generated by the PL model under the user’s true preference vector to be inserted into the header.

\begin{tcolorbox}[promptbox, title={Example Input}]
\small
\texttt{
\{"objectives": ["cleanliness","location","rooms","service","sleep quality","value"]\}
}\\
\texttt{
\{"id": 12, "preferences": [0, 4, 3, 2, 5, 1]\}
}
\end{tcolorbox}

\vspace{4pt}
\textbf{(c) Final Assembled Prompt.}  
The complete prompt sent to the LLM, obtained by concatenating the header and the input block.

\begin{tcolorbox}[promptbox, title={Final Prompt Example}]
\small
You are helping with hotel recommendation data synthesis.  

Objectives (indices): 
\texttt{0=cleanliness, 1=location, 2=rooms, 3=service, 4=sleep quality, 5=value}.

For each preference order lis: Generate one natural, descriptive sentence that a user might type when searching for a hotel, reflecting the given preference order.
\sethlcolor{pink}\hl{The sentence should be vague and ambiguous.}

\textbf{STRICT OUTPUT:} Return only JSON---a list of objects formatted as  
\texttt{\{"id": <int>, "query": <string>\}}.  
No extra text.

\textbf{INPUT:}  
Given preference order list:
\texttt{
"id": 12,  
"preferences": [0, 4, 3, 2, 5, 1]
}
\end{tcolorbox}

\vspace{4pt}
\textbf{(d) Response Output Example.}  
Example responses of both clean and $\epsilon^{+}$-corrupted cases returned by the LLM following the instruction.

\vspace{3pt}
\noindent
\begin{tabular}{@{}p{0.49\linewidth} p{0.49\linewidth}@{}}  % ✅ 两列并排
\begin{tcolorbox}[responsebox, title = Response Example (Clean)]
\small
\texttt{
\{
"id": 12,  
"query": "Please recommend a hotel with excellent cleanliness, quiet environment, and great service."
\}
}
\end{tcolorbox}
&
\begin{tcolorbox}[
  colback=pink!10, colframe=pink!70!black,
  title={Response Example ($\epsilon^{+}$-Corrupted)}, boxrule=0.5pt,
  fonttitle=\bfseries, coltitle=white]
\small
\texttt{
\{
"id": 12,  
"query": "Please recommend a hotel for a business trip."
\}
}
\end{tcolorbox}
\\
\end{tabular}
\\
\bottomrule
\end{tabularx}
\end{table}

\newpage

\begin{table*}[t]
\centering
\caption{Wording of prompts used in hotel recommendation simulation with LLMs as conversational agents.}
\label{tab:prompt_design_llm_2}
\renewcommand{\arraystretch}{1.15}
\setlength{\tabcolsep}{2pt}

\begin{tabularx}{\linewidth}{@{}X@{}} % ✅ 仅一列，自动换行
\toprule
\textbf{LLM Simulator-2: Prompt Construction Overview} \\ 
\midrule

\textbf{(a) Static Header.}  
This prompt guides the language model to infer the user's latent objective preference order ranking based on the given query. 
It explicitly constrains the output format to ensure structured data synthesis and prohibits extraneous text generation.

\begin{tcolorbox}[promptbox, title={Prompt Header}]
\small
Infer the user preference priority from the given hotel search query as a permutation of [0,1,2,3,4,5]. 

Objectives (indices): \texttt{0=[...], 1=[...], 2=[...], 3=[...], 4=[...], 5=[...]}.

\textbf{STRICT OUTPUT:} Return only JSON---a list of objects formatted as  
\texttt{\{"id": <int>, "rank": [i0,i1,i2,i3,i4,i5]\}}.  
No extra text.

\medskip
\textbf{INPUT:} Hotel search query: [...].
\end{tcolorbox}

\vspace{4pt}
\textbf{(b) Dynamic Input Block.}  
Input the hotel query just generated by LLM Simulator-1.

\begin{tcolorbox}[promptbox, title={Example Input}]
\small
\texttt{
\{"objectives": ["cleanliness","location","rooms","service","sleep quality","value"]\}
}\\
\texttt{
\{
"id": 12,  
"query": "Please recommend a hotel with excellent cleanliness, quiet environment, and great service."
\}
}
\end{tcolorbox}

\vspace{4pt}
\textbf{(c) Final Assembled Prompt.} 
The complete prompt sent to the LLM, obtained by concatenating the header and the input block.

\begin{tcolorbox}[promptbox, title={Final Prompt Example}]
\small
Infer the user preference priority from the given hotel search query as a permutation of [0,1,2,3,4,5]

Objectives (indices): 
\texttt{0=cleanliness, 1=location, 2=rooms, 3=service, 4=sleep quality, 5=value}.

\textbf{STRICT OUTPUT:} Return only JSON---a list of objects formatted as {"id": <int>, "rank": [i0,i1,i2,i3,i4,i5]}. No extra text.

\textbf{INPUT:}  
Hotel search query:
\texttt{
"id": 12,  
"query": Please recommend a hotel with excellent cleanliness, quiet environment, and great service.
}
\end{tcolorbox}

\vspace{4pt}
\textbf{(d) Response Output Example.}  
An example response returned by the LLM Simulator-2 following the instruction.

\begin{tcolorbox}[responsebox, title={Response Example}]
\small
\texttt{
\{
"id": 12,  
"rank": [0,4,3,2,1,5]
\}
}
\end{tcolorbox}
\\
\bottomrule
\end{tabularx}
\end{table*}

\newpage

\section{Limitations and Future Work}
While our proposed MO-PQUCB framework effectively integrates proactive preference elicitation and hybrid feedback to model user-specific preferences, it assumes a static preference vector for each user throughout the decision horizon. In many real-world applications, however, user preferences may evolve over time due to contextual changes, shifting objectives, or accumulated experiences. Our current model does not account for such dynamics, potentially limiting its adaptability in non-stationary environments. Extending the framework to accommodate temporally varying preferences remains an important direction for future work.

Besides, in this paper, we focus on the theoretical foundation of proactive preference elicitation and regret-optimal learning. While our framework assumes that the top-$m$ preferred objectives can be extracted from user queries, the actual parsing of natural language using a language model is treated as an external module and is not the focus of our algorithmic design or analysis.
In practice, parsing can be implemented via lightweight LLMs or intent classifiers, with constant overhead per round that does not scale with the number of arms $K$ or horizon $T$, and can often be preprocessed or cached.
Developing and integrating such a practical natural language interface with MO-PQUCB is an important and promising direction for future work.

\section{Proof of Lemma \ref{lemma:c_QE}}
\label{sec:pf_lemma_c_QE}

The main idea of the proof is inspired from the proof of \citep{hajek2014minimax} Theorem 3. Our main results depend on the structure of a weighted undirected graph $\mathcal{G}_t$ defined as follows. For notational simplicity, we omit the user index $n$ in this section, and present the formulation for a fixed user. All definitions naturally extend to the multi-user setting.

\begin{definition}[Comparison Graph $\mathcal{G}_t$]
For any $t \in [T]$, let each objective $j \in [D]$ corresponds to a vertex. For any pair of vertices $j,j' \in [D]$, there is a weighted edge between them with the weight equals to $\sum_{i \in [t]} \frac{m_i}{D^2}$. Let $L_{\mathcal{G}_t}$ denotes the Laplacian matrix of this weighted complete graph.
\end{definition}

Observe that for any $t$, $L_{\mathcal{G}_t}$ is positive semi-definite and the smallest eigenvalue of $L_{\mathcal{G}_t}$ is zero with the corresponding eigenvector given by the normalized all-one vector. Let $0 = \xi_1(\mathcal{G}_t) \leq \xi_2({\mathcal{G}_t}) \leq ... \leq \xi_D({\mathcal{G}_t})$ denote the eigenvalues of $L_{\mathcal{G}_t}$ in ascending order.

\begin{proof}[Proof of Lemma \ref{lemma:c_QE}]

By the loss function for MLE optimization, for any number of QE sample $t,$ we have 
\[
\begin{aligned}
\mathcal{L}_{\mathcal{S}_t}(\boldsymbol{c}) 
& =
- \sum_{i=1}^{t} \sum_{j=1}^{m_i} \log \big( \frac{ \exp (\boldsymbol{c}_{S_i(j)})}{\sum_{j^{\prime}=j}^{m_i}\exp (\boldsymbol{{c}}_{S_i(j^{\prime})}) + \sum_{j^{\prime} \in [D] \setminus S_i } \exp (\boldsymbol{{c}}_{j^{\prime}})} \big) \\
& =
- \sum_{i=1}^{t} \sum_{j=1}^{m_i} \log \big( \frac{ \exp (\boldsymbol{c}_{\sigma_i(j)})}{\sum_{j^{\prime}=j}^{D}\exp (\boldsymbol{{c}}_{\sigma_i(j^{\prime})})} \big)    
\end{aligned}
\]

Let $\nabla \mathcal{L}_t(c)$ as the gradient of $\mathcal{L}_{\mathcal{S}_t}(\boldsymbol{c})$ w.r.t $\boldsymbol{c}$, $H_t(c)$ as the Hessian matrix of $\mathcal{L}_{\mathcal{S}_t}(\boldsymbol{c})$ w.r.t $\boldsymbol{c}$. 
Define $\Delta_t = \hat{c}^{\text{(QE)}}_t - \overline{c}^0$. It follows from the definition that $\Delta_t$ is orthogonal to the vector $\boldsymbol{1}$. By the definition of MLE estimator, we have $\mathcal{L}_{\mathcal{S}_t}(\hat{c}^{\text{(QE)}}_t) \leq \mathcal{L}_{\mathcal{S}_t}(\overline{c}^{0})$, and thus
\[
\mathcal{L}_{\mathcal{S}_t}(\hat{c}^{\text{(QE)}}_t)
-
\mathcal{L}_{\mathcal{S}_t}(\overline{c}^{0})
-
\nabla \mathcal{L}_t(\overline{c}^{0})^{\top} \Delta_t
\leq 
- \nabla \mathcal{L}_t(\overline{c}^{0})^{\top} \Delta_t
\leq
\Vert \nabla \mathcal{L}_t(\overline{c}^{0}) \Vert_2 \Vert \Delta_t \Vert_2,
\]
where the last inequality holds due to the Cauchy-Schwartz inequality. By the Taylor expansion,
there exists a $c = a \overline{c}^{0} +(1-a) \hat{c}^{\text{(QE)}}$ for some $a \in [0,1]$ such that 
\[
\mathcal{L}_{\mathcal{S}_t}(\hat{c}^{\text{(QE)}}_t)
-
\mathcal{L}_{\mathcal{S}_t}(\overline{c}^{0})
-
\nabla \mathcal{L}_t(\overline{c}^{0})^{\top} \Delta_t
= 
\frac{1}{2} \Delta^{\top} H_t(c) \Delta \geq \frac{1}{2} \xi_2(H_t(c)) \Vert \Delta_t \Vert_2^2,
\]
where the last inequality holds because the Hessian matrix $H_t(c)$ is positive semi-definite and $H_t(c)^{\top}\boldsymbol{1}=0, \Delta_t^{\top}\boldsymbol{1}=0$. Combining above results yields
\begin{equation}
\label{eq:lemma_1_delta}
\Vert \Delta_t \Vert_2 \leq \frac{2 \Vert \nabla \mathcal{L}_t(\overline{c}^{0}) \Vert_2}{\xi_2(H_t(c))}.
\end{equation}

We next present two key auxiliary results used in the proof. 

\begin{lemma}
\label{lemma:xi_hessian}
For any sufficiently large number of samples $t \geq 8(2e^{2B}+1) \sqrt{2DT \log\frac{T}{\rho}}$, with probability at least $1-\frac{D \rho^2}{t^2}$, we have
\[
\xi_2(H_t(c)) 
\geq 
\frac{t \overline{m}_t}{D(4e^{4B} + 2e^{2B})}.
\]
\end{lemma}

\begin{lemma}
\label{lemma:gradient}
With probability at least $1-\frac{e^2 \pi^2 \rho^2}{3}$, for any $t \in [T]$, we have
\[
\Vert \nabla \mathcal{L}_t(\overline{c}^{0}) \Vert_2 \leq 4\sqrt{ \overline{m}_t t \log \frac{t}{\rho}}.
\]
\end{lemma}

Combining result in Eq. \ref{eq:lemma_1_delta} with above lemmas, we have with probability at least $1-\frac{(D+2e^2) \pi^2 \rho^2}{6}$, for all $t \geq 8(2e^{2B}+1) \sqrt{2DT \log\frac{T}{\rho}}$, we have 
\[
\Vert \Delta_t \Vert_2 \leq D(32e^{4B} + 16e^{2B})\sqrt{\frac{\log(t/\rho)}{\overline{m}_t t}}.
\]

\end{proof}

\subsection{Proof of Lemma \ref{lemma:xi_hessian}}
\begin{proof}
    
Define $S^{(i,j)}$ as the set of objectives contending for the $j$-th position in the ranking of $i$-th ranking sample after higher ranking objectives have been selected. Here $j \in [1,m]_i$. Let $\mathds{1}^{(i,j)}$ denote the indicator vector for the set $S^{(i,j)}$, and let $p^{(i,j)}$ denote the the corresponding probability column vector for the selection, then the Hessian can be written as:
\[
H_t(c) = \sum_{i=1}^t \sum_{j=1}^{m_i} H^{i,j},
\]
where $H^{i,j} = \frac{1}{2} \sum_{k,k' \in S^{(i,j)}} (e_k - e_{k'})(e_k - e_{k'})^{\top} p_k^{(i,j)} p_{k'}^{(i,j)} = \text{diag}(p^{(i,j)})-p_k^{(i,j)} {p_k^{(i,j)}}^{\top}$, with $p_k^{(i,j)} = \frac{\mathds{1}_{k}^{(i,j)} \exp( c_k ) }{\sum_{k' \in S^{(i,j)}} \exp( c_{k'} )}$, $e_x$ is the $x$-th standard basis vector.

By Claim 1 in \citep{hajek2014minimax}, we have
\[
e^{2B} H^{i,j} \succeq \frac{1}{2(D-j+1)^2} \sum_{k,k' \in S^{i,j}} (e_k - e_{k'})(e_k - e_{k'})^{\top}.
\]

Summing over $i$ and $j$, and note that $D-j+1 \leq D$, we have \begin{equation}
\label{eq:tilde_L_t}
e^{2B} H_t(c) \succeq \frac{1}{2} \sum_{k,k' \in S^{i,j}} (e_k - e_{k'})(e_k - e_{k'})^{\top} \frac{1}{D^2} \sum_{j=1}^{m_i} \mathds{1}_{\{\sigma_i^{-1}(k), \sigma_i^{-1}(k') \geq j\}} := \tilde{L}_t.
\end{equation}

Additionally, Observe that
\[
\begin{aligned}
\sum_{j=1}^{m_i} \mathbb{P}_c[\sigma_i^{-1}(k), \sigma_i^{-1}(k') \geq j]
&= 
1 + \sum_{k'' \in [D]}\mathds{1}_{\{k'' \neq k, k'\}} \frac{\exp(c_{i''})}{\exp(c_{i}) + \exp(c_{i'}) + \exp(c_{i''})} \\
&\geq 
1 + \frac{m_i-1}{2e^{2B}+1} \geq \frac{m_i}{2e^{2B}+1}.
\end{aligned}
\]

Recall that $L_{\mathcal{G}_t}$ is the Laplacian of $\mathcal{G}_t$ and $L_{\mathcal{G}_t} = \sum_{i=1}^t L_i$, we have
\begin{equation}
\begin{aligned}
\label{eq: tilde_L_t}
\mathbb{E}[\tilde{L}_t] 
&= 
\frac{1}{2}\sum_{i=1}^{t} \sum_{k,k' \in [D]} (e_k - e_{k'})(e_k - e_{k'})^{\top} \frac{1}{D^2} \sum_{j=1}^{m_i} \mathbb{P}_c[\sigma_i^{-1}(k), \sigma_i^{-1}(k') \geq j]\\
& \geq
\frac{1}{2}\sum_{i=1}^{t} \sum_{k,k' \in [D]} (e_k - e_{k'})(e_k - e_{k'})^{\top} \frac{m_i}{D^2(2e^{2B}+1)} = \frac{L_{\mathcal{G}_t}}{2e^{2B}+1},
\end{aligned}
\end{equation}

Define $
a_{kk'} = \frac{1}{D^2} \sum_{j=1}^{m_i} \left( \mathds{1}_{\left\{ \sigma_i^{-1}(k), \sigma_i^{-1}(k') \ge j \right\}} - \mathbb{P}_c \left[ \sigma_i^{-1}(k), \sigma_i^{-1}(k') \ge j \right] \right)
$, 
then
\[
\tilde{L} - \mathbb{E}_c[\tilde{L}] = \frac{1}{2} \sum_{i=1}^t \left( \sum_{k, k' \in [D]} a_{kk'} (e_k - e_{k'})(e_k - e_{k'})^\top \right) := \sum_{i=1}^t Y_i.
\]

Observe that \( |a_{kk'}| \le \frac{m_i}{D^2} \) and therefore $- L_i \le Y_i \le L_i$.
Furthermore, note the spectral norm \( \|L_i\| = \frac{m_i}{D} \), and thus \( \|Y_j\| \le \frac{m_i}{D} \).
Moreover,
\[
Y_j^2 = \sum_{k, k', k'' \in [D]} a_{kk'} a_{kk''} (e_k - e_{k'})(e_k - e_{k''})^\top.
\]
It follows that for any vector \( x \in \mathbb{R}^D \),
\begin{align*}
x^\top Y_j^2 x 
&= \sum_{k, k', k'' \in [D]} a_{kk'} a_{kk''} (x_k - x_{k'})(x_k - x_{k''}) \\
&\le \frac{m_i^2}{D^4} \sum_{k, k', k'' \in [D]} |x_k - x_{k'}||x_k - x_{k''}| \\
&= \frac{m_i^2}{D^4} \sum_{k \in [D]} \left( \sum_{k' \in [D]} |x_k - x_{k'}| \right)^2 \\
&\le
\frac{m_i^2}{D^4} \sum_{k \in [D]} \left( D \sum_{k' \in [D]} (x_k - x_{k'})^2 \right) \qquad \text{(Cauchy–Schwarz inequality)} \\
& = \frac{m_i^2}{D^3} \sum_{k, k' \in [D]} (x_k - x_{k'})^2 = \frac{2 m_i}{D} x^\top L_i x.
\end{align*}

Therefore, we have $Y^2_i \leq \frac{2 m_i}{D} L_i \leq 2 L_i$, which implies $\Vert \sum_{i=1}^t \mathbb{E}_c[Y^2_i] \Vert \leq \frac{2 m_i}{D} \xi_{D}(\mathcal{G}_t) \leq 2\xi_{D}(\mathcal{G}_t)$.

By the matrix Bernstein inequality \citep{tropp2012user}, we have with probability at least $1 - \frac{D \rho^2}{t^2}$,
\[
\Vert \tilde{L}_t - \mathbb{E}_c[\tilde{L}_t] \Vert \leq 2\sqrt{2\xi_{D}(\mathcal{G}_t) \log\frac{t}{\rho}} + \frac{4}{3} \log\frac{t}{\rho}.
\]

By the assumption that $\xi_{D}(\mathcal{G}_t) \geq C \log\frac{T}{\rho}$ for some sufficiently large $C$, we have
\[
\Vert \tilde{L}_t - \mathbb{E}_c[\tilde{L}_t] \Vert \leq 4\sqrt{2\xi_{D}(\mathcal{G}_t) \log\frac{t}{\rho}}.
\]

Combining Eq. \ref{eq:tilde_L_t} Eq. \ref{eq: tilde_L_t} with above result yields
\[
\xi_{2}(H_t(c)) 
\geq 
\frac{1}{e^{2B}} \xi_{2}(\tilde{L}_t) 
\geq
\frac{1}{e^{2B}} \Big(\frac{1}{2e^{2B}+1} \xi_{2}({\mathcal{G}_t}) - 4\sqrt{2\xi_{D}({\mathcal{G}_t}) \log\frac{t}{\rho}} \Big).
\]

Recall $L_{\mathcal{G}_t} = \sum_{i=1}^t L_i$, and $\overline{m}_t = \frac{1}{t} \sum_{i=1}^t m_i$, we have the complete comparison graph $\mathcal{G}_t$ with edge weight of $\frac{\overline{m}_t t}{D^2}$, which implies 
\[
\xi_{1}({\mathcal{G}_t}) = 1, \qquad \xi_{2}({\mathcal{G}_t}) = ... = \xi_{D}({\mathcal{G}_t}) = \frac{\overline{m}_t t}{D}.
\]

Putting everything together, we have for any $t \geq 8(2e^{2B}+1) \sqrt{2DT \log\frac{T}{\rho}}$, with probability at least $1 - \frac{D \rho^2}{t^2}$,
\[
\begin{aligned}
\xi_{2}(H_t(c)) & \geq \frac{1}{e^{2B}} \Big(\frac{1}{2e^{2B}+1} \xi_{2}({\mathcal{G}_t}) - 4\sqrt{2\xi_{D}({\mathcal{G}_t}) \log\frac{t}{\rho}} \Big) \\
& = 
\frac{1}{e^{2B}} \Big(\frac{1}{2e^{2B}+1} \frac{\overline{m}_t t}{D} - 4\sqrt{2\xi_{D}({\mathcal{G}_t}) \log\frac{t}{\rho}} \Big) \\
& \geq
\frac{\overline{m}_t t}{D (4e^{4B} + 2e^{2B})},
\end{aligned}
\]
where the last inequality holds due to the assumption of $t \geq 8(2e^{2B}+1) \sqrt{2DT \log\frac{T}{\rho}}$.
\end{proof}

\subsection{Proof of Lemma \ref{lemma:gradient}}
\begin{proof}
    
By the definition of MLE loss function, we have
\[
L_{\mathcal{S}_t}(c) = -\frac{1}{t} \sum_{i=1}^t \sum_{j=1}^{m_i} \log\left( \frac{\exp(c_{\sigma_i(j)})}{\sum_{k=j}^{D} \exp(c_{\sigma_i(k)})} \right),
\]

which gives the gradient of the loss function with respect to $c$ as
\[
\begin{aligned}
\nabla L_{\mathcal{S}_t}(c) &= -\frac{1}{t} \sum_{i=1}^t \sum_{j=1}^{m_i} \frac{\partial \log\big( \frac{\exp(c_{\sigma_i(j)})}{\sum_{k=j}^{D} \exp(c_{\sigma_i(k)})} \big)}{\partial c}  \\
& = 
-\frac{1}{t} \sum_{i=1}^t \sum_{j=1}^{m_i} \sum_{j' = j}^{D} \left( \frac{\exp(c_{\sigma_i(j')})}{\sum_{k=j}^{D} \exp(c_{\sigma_i(k)})} \right) (e_{\sigma_i(j')} - e_{\sigma_i(j)})
\end{aligned}
\]

\[
\]

Define
\[
V_i^{(jj')} =
\begin{cases}
\frac{ \exp(\overline{c}^0_{\sigma_i(j')}) }{ \sum\limits_{k=j}^{D} \exp(\overline{c}^0_{\sigma_i(k)}) }, & \text{if } \sigma_i(j) < \sigma_i(j') \\
- \frac{ \exp(\overline{c}^0_{\sigma_i(j)}) }{ \sum\limits_{k=j}^{D} \exp(\overline{c}^0_{\sigma_i(k)}) }, & \text{otherwise}
\end{cases}
\]

By \citep{zhu2023principled}, we have:
\[
\mathbb{E}[V_i^{(j j')}] = 0 \quad \text{and} \quad \nabla L_{\mathcal{S}_t}(\overline{c}^0) = \frac{1}{t} \sum_{i=1}^t \sum_{j=1}^{m_i} \sum_{j'=j+1}^{D} V_i^{(j j')} (e_j - e_{j'})
\]

Note:
\[
\left\| \sum_{j'=j+1}^{D} V_i^{(j j')} (e_j - e_{j'}) \right\| \le \left\| \sum_{j'=j+1}^{D} V_i^{(j j')} e_j \right\| + \left\| \sum_{j'=j+1}^{D} V_i^{(j j')} e_{j'} \right\| = 2
\]

And considering $t$ interactions and $m_j$ rounds of repeatedly selecting the most preferred objective from remaining objectives within each step $i \in [T]$,
we see that $\nabla L_{\mathcal{S}_t}(\overline{c}^0)$ is the value of a discrete-time martingale at time $\sum_{i=1}^{t} m_i$, such that the martingale has initial value $0$ and increments with norm bounded by $2$. By Lemma \ref{lemma:martingale_central_bd} \citep{hayeslarge}, we have:
\begin{equation}
\label{eq: c_gradient_bound}  
\mathbb{P} \left[ \left\| \nabla L_{\mathcal{S}_t}(\overline{c}^0) \right\|_2 \ge \sqrt{16 t \bar{m}_t \log \tfrac{t}{\rho}} \right]
\le 2e^2 \exp\left( - \frac{16 t \bar{m}_t \log \tfrac{t}{\rho}}{8 \sum_{i=1}^{t} m_i} \right)
= 2e^2 \left( \frac{\rho}{t} \right)^2.
\end{equation}

Summing over $t$ from $1$ to $T$, we have:
\[
\sum_{t=1}^{T} 2e^2 \left( \frac{\rho}{t} \right)^2
\le 2e^2 \rho^2 \sum_{t=1}^{T} \frac{1}{t^2}
\le \frac{e^2 \pi^2 \rho^2}{3}.
\]

Thus, with probability at least $1 - \frac{e^2 \pi^2 \rho^2}{3}$, for any $t \in [T]$, we have:
\[
\left\| \nabla L_{\mathcal{H}_t}(\overline{c}^0) \right\|_2 \le \sqrt{16 t \bar{m} \log \tfrac{t}{\rho}} \quad \text{holds}.
\]

\end{proof}

\section{Proof of Theorem \ref{theorem: lower_bd}}
\label{sec:pf_theorem_lower_bd}

\begin{proof}
We prove the Theorem \ref{theorem: lower_bd} using a simple toy instance with 2 arms and 2 objectives.

Let us consider an instance with two arms, where arm-1 has fixed reward $\boldsymbol{r}_{1} = [0,1]^{\top}$, arm-2 has fixed reward $\boldsymbol{r}_{2} = [0.7, 0.7]^{\top}$. There two preference vectors $\boldsymbol{c}_1 = [0.1,0.3]^{\top}$ and $\boldsymbol{c}_2 = [0.4,0.6]^{\top}$.
Apparently, under preference $\boldsymbol{c}_1$, arm-1 is the optimal arm, while for preference $\boldsymbol{c}_2$, arm-2 is the optimal.

Assume there exists a QE-only preference-aware algorithm $\mathcal{A}$ achieving sub-linear regret under preference $\boldsymbol{c}_1$.
By Lemma \ref{lemma: R_N_relation}, we have the expected number of pulls of arm-2 (suboptimal) given preference $\boldsymbol{c}_t$ is
\[
\mathbb{E}_{\boldsymbol{c}_1}[N_{2,T}] = \sum_{t \in [T]} \mathbb{P}_{\pi^{\mathcal{A}}_t \mid \boldsymbol{c}_1} (a_t = 2 \mid \boldsymbol{c}_1) = o(T).
\]

Due to the shift-invariant property of the Plackett–Luce (PL) model, any additive shift to the preference vector does not affect the induced ranking distribution. Specifically, since $\boldsymbol{c}_2 = \boldsymbol{c}_1 + 0.3 \cdot \boldsymbol{1}$, both $\boldsymbol{c}_1$ and $\boldsymbol{c}_2$ induce the same distribution over rankings. Consequently, a QE-based estimator that relies solely on ranking data cannot distinguish between $\boldsymbol{c}_1$ and $\boldsymbol{c}_2$, yielding the same mean-centered estimate in both cases.
As a result, for any algorithm $\mathcal{A}$ interacting with a user whose true preference is $\boldsymbol{c}_2$, the probability of selecting any given arm remains identical to that under $\boldsymbol{c}_1$:
\[
\mathbb{E}_{\boldsymbol{c}_2}[N_{2,T}] = \sum_{t=1}^T \mathbb{P}_{\pi_t^{\mathcal{A}} \mid \boldsymbol{c}_2}(a_t = 2) 
= \sum_{t=1}^T \mathbb{P}_{\pi_t^{\mathcal{A}} \mid \boldsymbol{c}_1}(a_t = 2) 
= \mathbb{E}_{\boldsymbol{c}_1}[N_{2,T}] = o(T).
\]

However, recall that under preference $\boldsymbol{c}_2$, arm-2 is the optimal arm, which implies that the regret of $\mathcal{A}$ under preference $\boldsymbol{c}_2$ would be at least $\Omega(T)$, i.e.,
\[
R_{\boldsymbol{c}_2}(T) 
= 
1 \cdot \mathbb{E}_{\boldsymbol{c}_2}[N_{1,T}]
=
1 \cdot (T - \mathbb{E}_{\boldsymbol{c}_2}[N_{2,T}])
>
T - o(T)
=
\Omega(T).
\]
\end{proof}

\section{Proof of Lemma \ref{lemma:g_estimator_upper_conf_bd}}
\label{sec:pf_lemma_g_estimator_upper_conf_bd}

\begin{proof}

For any $i \in [K], d \in [D], t>0$, by Hoeffding’s Inequality (Lemma~\ref{lemma: Hoeffding}), we have:

\begin{equation}
\begin{aligned}
\mathbb{P} \left( | \hat{\boldsymbol{r}}_{i,t}(d) - \boldsymbol{\mu}_{i}(d) | > \sqrt{\frac{ \log(t/\rho)}{N_{i,t}}} \right)
\leq
2\exp \left( \frac{-2 N_{i,t}^2 \log(t/\rho) }{ N_{i,t} \sum_{\ell=1}^{N_{i,t}}(1-0)^2 } \right)
=
2\exp \left( -2 \log(t/\rho) \right) 
=
2 \left( \frac{\rho}{t} \right)^2,
\end{aligned}
\end{equation}

Thus for any $i \in [K] $ and $t > 0$, with at least probability $1 - 2D \left( \frac{\rho}{t} \right)^2$, we have 

\[
\boldsymbol{\hat{r}}_{i,t} - \sqrt{\frac{ \log(t/\rho)}{N_{i,t}}} \boldsymbol{1} \preceq \boldsymbol{\mu}_{i} \preceq \boldsymbol{\hat{r}}_{i,t} + \sqrt{\frac{ \log(t/\rho)}{N_{i,t}}} \boldsymbol{1}.
\]

And thus for any $n \in [N]$, we can derive 
\begin{equation}
\begin{aligned}  
\label{eq: c_hidden_confidence_upper_bd2}
\langle \boldsymbol{\hat{c}}_{n,t}^{(\text{HB})}, \boldsymbol{\hat{r}}_{i,t} \rangle
&+ B_{i,t}^{n(r)} + B_{i,t}^{n(c)} 
-
\langle \boldsymbol{\overline{c}}_n, \boldsymbol{\mu}_{i} \rangle \\
& \geq
\langle \boldsymbol{\hat{c}}_{n,t}^{(\text{HB})}, \boldsymbol{\hat{r}}_{i,t} \rangle
+ B_{i,t}^{n(r)} + B_{i,t}^{n(c)} - \left \langle \boldsymbol{\overline{c}}_{n}, \boldsymbol{\hat{r}}_{i,t} + \sqrt{\frac{ \log(t/\rho)}{N_{i,t}}} \boldsymbol{1} \right\rangle  \\
& =
\left \langle \boldsymbol{\hat{c}}_{n,t}^{(\text{HB})},  \boldsymbol{\hat{r}}_{i,t} + \sqrt{\frac{ \log(t/\rho)}{N_{i,t}}} \boldsymbol{1} \right\rangle 
+ B_{i,t}^{n(c)} 
- 
\left\langle \boldsymbol{\overline{c}}_{n}, \boldsymbol{\hat{r}}_{i,t} + \sqrt{\frac{ \log(t/\rho)}{N_{i,t}}} \boldsymbol{1} \right\rangle  \\
& = 
\left\langle \boldsymbol{\hat{c}}_{n,t}^{\text{(HB)}} - \boldsymbol{\overline{c}}_{n}, \boldsymbol{\hat{r}}_{i,t} + \sqrt{\frac{ \log(t/\rho)}{N_{i,t}}} \boldsymbol{1} \right\rangle + B_{i,t}^{n(c)} \\
& \underset{(a)}{\geq}
- \Vert \boldsymbol{\hat{c}}_{n,t}^{\text{(HB)}} - \boldsymbol{\overline{c}}_n \Vert_{{\boldsymbol{V}_{n,t}}} \cdot \Big\Vert \boldsymbol{\hat{r}}_{i,t} + \sqrt{\frac{ \log(t/\rho)}{N_{i,t}}} \boldsymbol{1} \Big\Vert_{({\boldsymbol{V}_{n,t}})^{-1}}
+
B_{i,t}^{n(c)} \\
& \underset{(b)}{\geq}
- \beta_t^n \cdot \left\Vert \boldsymbol{\hat{r}}_{i,t} + \sqrt{\frac{ \log(t/\rho)}{N_{i,t}}} \boldsymbol{1} \right\Vert_{({\boldsymbol{V}_{n,t}})^{-1}}
+
B_{i,t}^{n(c)} = 0, \\
& \qquad \qquad (\text{with probability at least } 1 - 2D \left( \rho/t \right)^2)
\end{aligned}
\end{equation}
where (a) holds by Cauchy-Schwarz inequality, (b) holds due to $\beta_t^n$ as the radius of confidence region for preference estimation $\boldsymbol{\hat{c}}_{n,t}^{(\text{HB})}$. By the convergence of sum of reciprocals of squares that  $\sum_{t=1}^{\infty} t^{-2} = \frac{\pi^2}{6}$, we have with probability larger than $1 - D\rho^2\pi^2/6$, above result holds for all $t$, completing the proof.
\end{proof}

\section{Proof of Lemma \ref{lemma:ucb_C_BF}}
\label{sec:pf_lemma_ucb_C_BF}

\newcounter{templemma}
\setcounter{templemma}{\value{lemma}} % 保存当前值

\setcounter{lemma}{1} % 想让下一个 lemma 显示为3
\begin{lemma}[Detailed Version]
For any user $n \in [N]$, any $\rho \in (0,1)$ and any sufficiently large samples $t \geq 8(2e^{2B}+1)\sqrt{2DT \log(\frac{T}{\rho})}$, with probability at least $1-\frac{(1+D+2e^2) \rho^2 \pi^2}{6}$, the estimated preference of Eq.\ref{eq:c_BF} in Algorithm \ref{alg:MO_PQUCB} has the following upper-confidence bound 
\begin{equation}
\label{eq:c_ucb_size}
\begin{aligned}
\Vert \boldsymbol{\hat{c}}^{(\text{HB})}_{n,t} - \overline{\boldsymbol{c}}_{n} \Vert_{ \boldsymbol{V}_{n,t}}
& \leq
D (32 e^{4B} + 16 e^{2B}) \sqrt{ \frac{ (1-\alpha) 
\log(t/\rho)}{\overline{m}_t t} } + B\sqrt{2 (1-\alpha) \lambda} \\
& \quad +
\alpha \frac{BD}{2} \sqrt{D \log\big(\frac{1+\alpha D(t-1) \ (\lambda+1)}{\rho} \big) + \log(\frac{1+\lambda}{\lambda \rho})}
\end{aligned}
\end{equation}
\end{lemma}

\setcounter{lemma}{\value{templemma}} % 恢复

\begin{proof}

For notational simplicity, we fix a user $n$ and omit the user index $n$ as superscript or subscript. 
By Eq. \ref{eq:c_BF}, we have closed-form solution for preference estimation for bandit policy as:
\[
\hat{c}_t^{(\text{HB})} = V_t^{-1} \left( \alpha \sum_{\tau = 1}^{t-1} g_{a_\tau} r_{a_\tau, \tau} + (1 - \alpha) U_\lambda \hat{c}_t^{\text{(QE)}} \right)
\]

Rearranging the equation yields
\begin{align*}
\hat{c}_t^{(\text{HB})} 
&= V_t^{-1} \left( \alpha \sum_{\tau = 1}^{t-1} g_{a_\tau} r_{a_\tau, \tau} + (1 - \alpha) U_\lambda \hat{c}_t^{\text{(QE)}} \right) \\
&= V_t^{-1} \left( \alpha \sum_{\tau = 1}^{t-1} r_{a_\tau, \tau} r_{a_\tau, \tau}^\top (\overline{c} + \xi_\tau)  + (1 - \alpha) U_\lambda \hat{c}_t^{\text{(QE)}} \right) \\
&= V_t^{-1} \left( \Big(\alpha \sum_{\tau = 1}^{t-1} r_{a_\tau, \tau} r_{a_\tau, \tau}^\top + (1-\alpha) U_{\lambda} \Big) \overline{c} - (1-\alpha) U_{\lambda} \overline{c} + \alpha \sum_{\tau = 1}^{t-1} r_{a_\tau, \tau} r_{a_\tau, \tau}^{\top} \xi_\tau + (1 - \alpha) U_\lambda \hat{c}_t^{\text{(QE)}} \right) \\
&= \overline{c} - (1 - \alpha) V_t^{-1} U_\lambda \overline{c} + \alpha V_t^{-1} \sum_{\tau = 1}^{t-1} r_{a_\tau, \tau} r_{a_\tau, \tau}^{\top} \xi_\tau + (1 - \alpha) V_t^{-1} U_\lambda \hat{c}_t^{\text{(QE)}},
\end{align*}

\noindent
which implies that
\[
\|\hat{c}_t^{(\text{HB})} - \overline{c} \|_{V_t} 
= \Big\| (1 - \alpha) V_t^{-1} U_\lambda (\hat{c}_t^{\text{(QE)}} - \overline{c}) + \alpha V_t^{-1} \sum_{\tau=1}^{t-1} r_{a_\tau, \tau} r_{a_\tau, \tau}^{\top} \xi_\tau \Big\|_{V_t}
\]
Since $V_t$ is PSD, we have:
\begin{equation}
\label{eq: proof_bf_error}
\| \hat{c}_t^{(\text{HB})} - \overline{c} \|_{V_t} 
\le 
(1 - \alpha) \underbrace{\Big\| V_t^{-1} U_\lambda (\hat{c}_t^{\text{(QE)}} - \overline{c}) \Big\|_{V_t} }_{A}
+ \alpha  \underbrace{\Big\| V_t^{-1} \sum_{\tau=1}^{t-1} r_{a_\tau, \tau} r_{a_\tau, \tau}^{\top} \xi_\tau \Big\|_{V_t}}_{B}
\end{equation}

For \textbf{Term A}, we have:
\begin{equation}
\label{eq:ucb_c_term_A}
\begin{aligned}
A &= \| V_t^{-1} U_\lambda (\hat{c}_t^{\text{(QE)}} - \overline{c}) \|_{V_t} 
 = \| U_\lambda (\hat{c}_t^{\text{(QE)}} - \overline{c}) \|_{V_t}^{-1} \\
& \underset{(a)}{\leq}
\frac{1}{\sqrt{1-\alpha}} \| U_\lambda (\hat{c}_t^{\text{(QE)}} - \overline{c}) \|_{U_{\lambda}^{-1}} 
= 
\frac{1}{\sqrt{1-\alpha}} \| (\hat{c}_t^{\text{(QE)}} - \overline{c}) \|_{U_{\lambda}},
\end{aligned}
\end{equation}

where (a) holds since $V_t^{-1} \preceq V_{t-1}^{-1} \preceq V_1^{-1} = ((1-\alpha)U_{\lambda})^{-1}$. For $\| (\hat{c}_t^{\text{(QE)}} - \overline{c}) \|_{U_{\lambda}}$ we have
\begin{align*}
\left\| \hat{c}_t^{\text{(QE)}} - \overline{c} \right\|_{U_\lambda} 
&= \sqrt{ (\hat{c}_t^{\text{(QE)}} - \overline{c})^\top U_\lambda (\hat{c}_t^{\text{(QE)}} - \overline{c}) } \\
&= \sqrt{ (\hat{c}_t^{\text{(QE)}} - \overline{c})^\top (U + \lambda I)(\hat{c}_t^{\text{(QE)}} - \overline{c}) } \\
&= \sqrt{ 
\underbrace{ (\hat{c}_t^{\text{(QE)}} - \overline{c})^\top U
(\hat{c}_t^{\text{(QE)}} - \overline{c})}_{C}
+ \lambda \underbrace{(\hat{c}_t^{\text{(QE)}} - \overline{c})^\top (\hat{c}_t^{\text{(QE)}} - \overline{c})}_{D}
}
\end{align*}

For \textbf{Term C}, we expand:
\begin{align*}
(\hat{c}_t^{\text{(QE)}} - \overline{c})^\top U (\hat{c}_t^{\text{(QE)}} - \overline{c}) 
&= \left( \hat{c}_t^{\text{(QE)}} - \left( \overline{c}^0 + \tfrac{\| \overline{c} \|_1}{D} \mathbf{1} \right) \right)^\top U 
\left( \hat{c}_t^{\text{(QE)}} - \left( \overline{c}^0 + \tfrac{\| \overline{c} \|_1}{D} \mathbf{1} \right) \right) \\
&= (\hat{c}_t^{\text{(QE)}} - \overline{c}^0)^\top U (\hat{c}_t^{\text{(QE)}} - \overline{c}^0)
+ \left( \tfrac{\| \overline{c} \|_1}{D} \right)^2 \mathbf{1}^\top U \mathbf{1}
- \tfrac{2 \| \overline{c} \|_1}{D} \mathbf{1}^\top U (\hat{c}_t^{\text{(QE)}} - \overline{c}^0)
\end{align*}

Note $U = I - \frac{1}{D} \mathbf{1} \mathbf{1}^\top$, so:
\[
\mathbf{1}^\top U = \mathbf{1}^\top - \tfrac{1}{D} \mathbf{1}^\top \mathbf{1}^\top = 0, \quad U \mathbf{1} = 0
\Rightarrow \text{(all-ones projection vanishes)}
\]

This gives that,
\[
(\hat{c}_t^{\text{(QE)}} - \overline{c})^\top U (\hat{c}_t^{\text{(QE)}} - \overline{c}) 
= (\hat{c}_t^{\text{(QE)}} - \overline{c}^0)^\top U (\hat{c}_t^{\text{(QE)}} - \overline{c}^0)
\quad \text{(orthogonal to} \boldsymbol{1}).
\]

For \textbf{Term D}, we have:

\begin{align*}
\lambda (\hat{c}_t^{\text{(QE)}} - \overline{c})^\top (\hat{c}_t^{\text{(QE)}} - \overline{c}) 
&= \lambda \left( \hat{c}_t^{\text{(QE)}} - \overline{c}^0 - \tfrac{\| \overline{c} \|_1}{D} \mathbf{1} \right)^\top 
\left( \hat{c}_t^{\text{(QE)}} - \overline{c}^0 - \tfrac{\| \overline{c} \|_1}{D} \mathbf{1} \right) \\
&\le \lambda \left( \tfrac{\| \overline{c} \|_1}{D} \mathbf{1} + \tfrac{\| \overline{c} \|_1}{D} \mathbf{1} \right)^\top 
\left( \tfrac{\| \overline{c} \|_1}{D} \mathbf{1} + \tfrac{\| \overline{c} \|_1}{D} \mathbf{1} \right) \\
&= 2 \lambda \cdot \tfrac{\| \overline{c} \|_1^2}{D^2} \cdot \mathbf{1}^\top \mathbf{1} \\
&= 2 \lambda \cdot \tfrac{\| \overline{c} \|_1^2}{D^2} \cdot D 
= 2 \lambda \cdot \tfrac{\| \overline{c} \|_1^2}{D}
\end{align*}

Combining together we have 
\begin{align*}
\left\| \hat{c}_t^{\text{(QE)}} - \overline{c} \right\|_{U_\lambda}
&= 
\sqrt{ (\hat{c}_t^{\text{(QE)}} - \overline{c}^0)^\top U (\hat{c}_t^{\text{(QE)}} - \overline{c}^0) + \lambda \frac{\| \overline{c} \|_1^2}{D} } \\
&\le 
\sqrt{ (\hat{c}_t^{\text{(QE)}} - \overline{c}^0)^\top U (\hat{c}_t^{\text{(QE)}} - \overline{c}^0) } + B \sqrt{2 \lambda D } \\
& \le 
D(32e^{4b} + 16e^{2b}) \sqrt{ \frac{ \log \tfrac{t}{\rho} }{\overline{m}_t t} } + B \sqrt{2\lambda D}. \\
& \qquad \qquad \text{(with probability $\geq 1 - \frac{(D + 2e^2) \pi^2 \rho^2}{6}$ for all $t$ by Lemma \ref{lemma:c_QE}.)}
\end{align*}

Plugging back Eq \ref{eq:ucb_c_term_A} gives:
\begin{equation}
\label{eq:proof_QE_BF_error_A}
A \le \frac{1}{\sqrt{1 - \alpha}} \left\| \hat{c}_t^{\text{(QE)}} - \overline{c} \right\|_{U_\lambda}
\le \frac{1}{\sqrt{1 - \alpha}} \left( D(32e^{4b} + 16e^{2b}) \sqrt{ \frac{ \log \tfrac{t}{\rho} }{\overline{m}_t t} } + B \sqrt{2\lambda D} \right)
\end{equation}

\medskip

For \textbf{term B}, we first show that that $r_{a_\tau, \tau}^{\top} \xi_\tau$ is conditionally sub-Gaussian.

Let $\boldsymbol{a}_{1:t} = \{\boldsymbol{a}_{\iota}\}_{\iota = 1}^{t}$ be the sequence of historical pulled actions within $t$ steps, $\boldsymbol{g}_{1:t}=\{\boldsymbol{g}_{a_{\iota}, \iota}\}_{\iota = 1}^{t}$ and $\boldsymbol{r}_{1:t}=\{\boldsymbol{r}_{a_{\iota}, \iota}\}_{\iota = 1}^{t}$ be the sequences of historical overall scores and reward vectors within $t$ steps respectively, 
and define the $\sigma$ algebra $\mathcal{F}_{t-1} = \sigma(\boldsymbol{a}_{1:t-1}, \boldsymbol{g}_{1:t-1}, \boldsymbol{r}_{1:t})$.
Note that for any $t \geq 1$,
\[
\begin{aligned}
\mathbb{E} [ r_{a_\tau, \tau}^{\top} \xi_\tau \mid \mathcal{F}_{t-1} ] 
& =
\mathbb{E} [ {c}_t^{\top} {r}_{a_t,t} \mid \mathcal{F}_{t-1} ] 
- 
\mathbb{E} [ {\overline{c}}^{\top} {r}_{a_t,t} \mid \mathcal{F}_{t-1} ] \\
& \underset{(a)}{=}
\boldsymbol{\overline{c}}^{t-1} \boldsymbol{r}_{a_t,t} 
-
\boldsymbol{\overline{c}}^{t-1} \boldsymbol{r}_{a_t,t} =0,
\end{aligned}
\]
where (a) holds since $\boldsymbol{c}_t$ is independent of $\mathcal{F}_{t-1}$ and the conditional expectation fact that $\mathbb{E}(X \mid \mathcal{F}) = X$ if $X \in \mathcal{F}$.
Furthermore, by assumption of $DB$-bounded overall-reward, $-DB \leq  r_{a_\tau, \tau}^{\top} \xi_\tau  \leq DB$ holds almost surely, and hence we can conclude that $r_{a_\tau, \tau}^{\top} \xi_\tau$ is conditionally $DB$-sub-Gaussian. Also, since $\boldsymbol{r}_{a_t,t}$ is $\mathcal{F}_{t-1}$-measurable, 
define $r_{a_\tau, \tau}' = \sqrt{\alpha} r_{a_\tau, \tau}$,
by Lemma \ref{lemma:self_norm_martingale}, we have with probability at least $1 - \rho$:
\begin{equation}
\label{eq:norm_error_martingale_intermedia}
\left\| \sum_{\tau=1}^{t-1} r_{a_\tau, \tau} r_{a_\tau, \tau}^{\top} \xi_\tau \right\|_{V_t^{-1}}
=
\frac{1}{\sqrt{\alpha}} \left\| \sum_{\tau=1}^{t-1} r_{a_\tau, \tau}' r_{a_\tau, \tau}'^{\top} \xi_\tau \right\|_{V_t^{-1}}
\le 
\frac{DB}{\sqrt{\alpha}} \sqrt{2 \log \left( \frac{ \det(V_t)^{1/2} \det(V_1)^{-1/2} }{\rho} \right)}.
\end{equation}

Note
$
V_1 = (1 - \alpha) U_\lambda = (1 - \alpha) \left( I - \tfrac{1}{D} \mathbf{1} \mathbf{1}^\top + \lambda I \right),
$
it can be easily verified that:
\[
\det(U_{\lambda}) = \lambda (1 + \lambda)^{D - 1} \Rightarrow \text{(eigenvalues: } 
\lambda \text{ with multiplicity } 1,\quad 1 + \lambda \text{ with multiplicity } D - 1 \text{)}.
\]

Since 
\[
V_t = (1 - \alpha) U_\lambda + \sum_{\tau=1}^{t-1} r_{a_\tau, \tau}' r_{a_\tau, \tau}'^\top 
\le (1 - \alpha)(1 + \lambda) I + \sum_{\tau=1}^{t-1} r_{a_\tau, \tau}' r_{a_\tau, \tau}'^\top ,
\quad \text{with } \|r_{a_\tau, \tau}'\|_2 \le \sqrt{\alpha D},
\]
we have
\[
\det(V_t) \le \left( (1 - \alpha)(1 + \lambda) \right)^D \cdot \left( 1 + \tfrac{\alpha D(t-1)}{1 + \lambda} \right)^D.
\]

And note $\det(V_1) = (1 - \alpha)^D \det(U_{\lambda}) = (1 - \alpha)^D \lambda (1 + \lambda)^{D - 1}$, we have
\[
\frac{\det(V_t)}{\det(V_1)} 
= \frac{1 + \lambda}{\lambda} \cdot \left( 1 + \tfrac{\alpha D(t - 1)}{1 + \lambda} \right)^D.
\]

Plugging back to Eq. \ref{eq:norm_error_martingale_intermedia} gives (with probability at least $1 - \rho$):
\[
\left\| \sum_{\tau=1}^{t-1} r_{a_\tau, \tau} r_{a_\tau, \tau}^{\top} \xi_\tau \right\|_{V_t^{-1}}
\le 
BD \sqrt{D \log\left( \frac{1 + \lambda}{\lambda \rho} \cdot \frac{1 + \alpha D(t - 1)/(1 + \lambda)}{\rho} \right)}.
\]

Combine A and B together, we have
\[
\begin{aligned}
\left\| \hat{c}_t^{(\text{HB})} - \overline{c} \right\|_{V_t} 
& \le 
(1 - \alpha) \left\| V_t^{-1} U_\lambda (\hat{c}_t^{\text{(QE)}} - \overline{c}) \right\|_{V_t}
+ \alpha \Big\| V_t^{-1} \sum_{\tau=1}^{t-1} r_{a_\tau, \tau} r_{a_\tau, \tau}^{\top} \xi_\tau \Big\|_{V_t} \\
& \le (1 - \alpha) \left( \frac{ D (32e^{4b} + 16e^{2b}) }{ \sqrt{1 - \alpha} } \cdot 
\sqrt{ \frac{ \log \tfrac{t}{\rho} }{\overline{m}_t t} } + B \sqrt{ \frac{ 2 \lambda }{1 - \alpha } } \right) \\
& \quad + \alpha BD \sqrt{ D \log \left( \frac{1 + \alpha D(t - 1)/(1 + \lambda)}{\rho} \right) + \log\left( \tfrac{1 + \lambda}{\lambda \rho} \right) },
\end{aligned}
\]
thus completing the proof.
\end{proof}

\section{Proof of Theorem \ref{thm:main_regret}}
\label{sec:pf_thm_user_spe_reg}

Before the detailed proofs, we first restate Theorem 2 with a detained version and show how the optimal choices of balance-weight $\alpha$ and regularization $\lambda$ are derived.

\setcounter{templemma}{\value{theorem}} % 保存当前值

\setcounter{theorem}{1} % 想让下一个 lemma 显示为3
\begin{theorem}[Detailed Version]
For any user $n \in [N]$, any $\rho \in (0,1)$, by setting $\beta_t$ as the UCB of $\boldsymbol{\hat{c}}^{n(\text{HB})}_t$ as Eq. \eqref{eq:c_ucb_size}, with probability at least $1-\frac{(2+D+2e^2) \rho^2 \pi^2}{6}$, we have the following regret bound 
\[
% \textstyle
R^{n}(T) \leq 
\Psi^{n}_{\hat{c}}(T)
+
\Psi^{n}_{\hat{r}}(T)
+
M, \quad \text{where}
\]
\[
% \textstyle
M = \max
\{
\frac{4 D^2 K \alpha^2}{\nu^4_{r \downarrow}}\log (\frac{T}{\rho})
,
\frac{1}{\alpha}\log (\frac{T}{\rho})
,
2D(16e^{2B}+8)^2 \log(\frac{T}{\rho})
\},
\]
\[
% \textstyle
\Psi^{n}_{\hat{c}}(T) = 
O \Big( 
\big( \sqrt{\frac{D^{3}}{\overline{m}_T}}
+ B\sqrt{\frac{\lambda}{\alpha} }
+ 
BD^{\frac{3}{2}} \sqrt{\alpha \log(\frac{T}{\rho}) + \alpha \log(\frac{1+\lambda}{\lambda \rho})} \big) \sqrt{T \log(\frac{\alpha T}{\lambda})} \Big),
\]
\[
\textstyle
\Psi^{n}_{\hat{r}}(T) = O \Big( 
\sqrt{\frac{K D^3}{\overline{m}_T \lambda}} \log^{2} T
+
\big( BD+
\frac{\alpha BD^{\frac{3}{2}}}{\sqrt{\lambda}} \sqrt{\log(\frac{T}{\rho}) + \log(\frac{1+\lambda}{\lambda \rho})} \big) \sqrt{ KT \log (\frac{T}{\rho})}
\Big).
\]
\end{theorem}

\setcounter{theorem}{\value{templemma}} % 恢复

\begin{remark}
\label{remark:user_spe_reg}
With above result, we can optimize the choices of balance-weight $\alpha$ and regularization $\lambda$ to minimize the user-specific regret. Specifically, the bound reveals the main multipliers of $\sqrt{T \log T}$. By setting $\alpha = \lambda = O(\log^{-1}(T/\rho))$, we can guarantee all such multipliers become $O(1)$, thus achieving a near optimal regret
\[
R^{n}(T) = O(\sqrt{T \log T}).
\]
Compared to the previous PRUCB \citep{cao2025provably} in preference-aware MO-MAB, our proactive QE-based framework and algorithm achieve achieves much better performance by reducing a multiplicative $\sqrt{\log T}$ term.
\end{remark} 

Before the proof, we present some essential observations that will be used in the proof. 

\begin{claim}
\label{claim_1}
By Hoeffding’s Inequality (Lemma~\ref{lemma: Hoeffding}), with probability at least $1 - \frac{KD \rho^2 \pi^2}{3}$,
we have 
\[
|\boldsymbol{\mu}_{i}(d) - \boldsymbol{\hat{r}}_{i,t}(d)| \leq \sqrt{\frac{\log ( \frac{t}{\rho} )}{N_{i,t}}}
\]
holds for all $i \in [K], d \in [D], t \in [T]$.
\end{claim}

\begin{claim}

\label{claim_2}
Since 
$(\alpha \boldsymbol{r}_{i,\ell} \boldsymbol{r}_{i,\ell}^{\top})(m,n) \in [0,\alpha], \forall (m,n) \in [D]$ and by Hoeffding’s Inequality (Lemma~\ref{lemma: Hoeffding}), with probability at least $1 - \frac{KD(D-1) \rho^2 \pi^2}{12}$,
we have 
\[
 \mathbb{E} \left[ \sum_{\ell \in \mathcal{T}_{i,t-1}} \alpha \boldsymbol{r}_{i,\ell} \boldsymbol{r}_{i,\ell}^{\top} \right](m,n)
-
\sum_{\ell \in \mathcal{T}_{i,t-1}} \left( \alpha \boldsymbol{r}_{i,\ell} \boldsymbol{r}_{i,\ell}^{\top} \right)(m,n)
\leq 
\alpha \sqrt{ N_{i,t} \log \left( \frac{t}{\rho} \right)}
\]
holds $\forall i \in [K], \forall m \in [D], \forall n \in [m,D]$.
\end{claim}

\begin{proof}[Proof of Theorem \ref{thm:main_regret}]

Here we focus on the user-specific regret $R^{n}(T)$. For notational simplicity, we fix a user $n$ and omit the user index $n$ as superscript or subscript.
% The proof is inspired by \citep{cao2025provably}.
Let $M$ be an arbitrary positive integer, we can express the user-specific regret $R^{n}(T)$ in a truncated form with respect to $M$ as follows:
\begin{equation}
\begin{aligned}
\label{eq: trunc_regret_0}
R(T)^n = \sum_{t=1}^{T} \text{regret}_{t} \leq  O(M) + \sum_{t=M+1}^{T} \text{regret}_{t},
\end{aligned}
\end{equation}

where $\text{regret}_{t}$ denotes the instantaneous regret of algorithm at step $t \in \left[ T \right]$, and the last inequality holds since the fact that the instantaneous regret is bounded.

Next, we analyze the instantaneous regret over the truncated time horizon $[M+1, T]$. 
By Claim \ref{claim_1}, we have 
\begin{equation}
\label{eq: a^*_confidence_hidden}
\boldsymbol{\mu}_{a^*_t} \preceq \boldsymbol{\hat{r}}_{a^*_t,t} + \sqrt{\frac{ \log(t/\rho)}{N_{i,t}}} \boldsymbol{1}, 
\text{ and }
\boldsymbol{\mu}_{a_t} \succeq  \boldsymbol{\hat{r}}_{a_t, t} - \sqrt{\frac{ \log(t/\rho)}{N_{i,t}}} \boldsymbol{1}.
\end{equation}

By the definition of regret and fact above, we can derive the upper-bound of expected instantaneous regret as follows:

\[
\begin{aligned}
\text{regret}_{t} 
& = \boldsymbol{\overline{c}}^{\top}  \boldsymbol{\mu}_{a^*_t} - \boldsymbol{\overline{c}}^{\top}  \boldsymbol{\mu}_{a_t} \\
& \underset{(a)}{\leq}
\boldsymbol{\hat{c}}^{\top} \boldsymbol{\hat{r}}_{a^*_t,t} + B_{a^*_t,t}^{r} + B_{a^*_t,t}^{c} - \boldsymbol{\overline{c}}^{\top}  \boldsymbol{\mu}_{a_t} \\
& \underset{(b)}{\leq}
\boldsymbol{\hat{c}}^{\top} \boldsymbol{\hat{r}}_{a^*_t,t} + B_{a^*_t,t}^{r} + B_{a^*_t,t}^{c} - \boldsymbol{\overline{c}}^{\top} \left( \boldsymbol{\hat{r}}_{a_t,t} + \sqrt{\frac{ \log(t/\rho)}{N_{a,t}}} \boldsymbol{1} \right) + 2\Vert\boldsymbol{\overline{c}}\Vert_1 \sqrt{\frac{ \log(t/\rho)}{N_{a,t}}} \\
& \underset{(c)}{\leq}
\boldsymbol{\hat{c}}^{\top} \boldsymbol{\hat{r}}_{a_t,t} + B_{a_t,t}^{r} + B_{a_t,t}^{c} - \boldsymbol{\overline{c}}^{\top} \left( \boldsymbol{\hat{r}}_{a_t,t} + \sqrt{\frac{ \log(t/\rho)}{N_{a,t}}} \boldsymbol{1} \right) + 2\Vert\boldsymbol{\overline{c}}\Vert_1 \sqrt{\frac{ \log(t/\rho)}{N_{a,t}}}  \\
& = 
(\boldsymbol{\hat{c}} - \boldsymbol{\overline{c}})^{\top} \left( \boldsymbol{\hat{r}}_{a_t,t} + \sqrt{\frac{ \log(t/\rho)}{N_{a,t}}} \boldsymbol{1} \right) + B_{a_t,t}^{c} + 2\Vert\boldsymbol{\overline{c}}\Vert_1 \sqrt{\frac{ \log(t/\rho)}{N_{a,t}}} \\
& \underset{(d)}{\leq}
\Vert \boldsymbol{\hat{c}}_t - \boldsymbol{\overline{c}}_t \Vert_{\boldsymbol{V}_{t-1}} \cdot \left \Vert \boldsymbol{\hat{r}}_{i,t} + \sqrt{\frac{ \log(t/\rho)}{N_{a,t}}} \boldsymbol{1} \right \Vert_{\boldsymbol{V}_{t-1}^{-1}} + B_{a_t,t}^{c} + 2\Vert\boldsymbol{\overline{c}}\Vert_1 \sqrt{\frac{ \log(t/\rho)}{N_{a,t}}} \\
& \underset{(e)}{\leq}
\min \left( 2 B_{a_t,t}^{c}, BD \right) 
+ 
2 \Vert\boldsymbol{\overline{c}}\Vert_1 \sqrt{\frac{ \log(t/\rho)}{N_{a,t}}} ,
\end{aligned}
\]

where (a) followed by Lemma \ref{lemma:g_estimator_upper_conf_bd}, (b) followed by Eq.~\ref{eq: a^*_confidence_hidden}, (c) holds by the definition of optimization policy for arm selection, (d) holds by Cauchy-Schwarz inequality, (e) followed by the definition of $B_{a_t,t}^{c}$, and the fact that the instantaneous regret is at most $BD$. Above result implies
\begin{equation}
\begin{aligned}
\label{eq: trunc_regret}
R(T) 
& \leq
O(M)
+
\sum_{t=M+1}^{T} \text{regret}_t \\
& \leq
O(M)
+
\sum_{t=M+1}^{T} \Big( \text{regret}_{t}^{\tilde{\boldsymbol{c}}} 
+
\text{regret}_{t}^{\tilde{\boldsymbol{r}}}
\Big)
\\ 
& \leq
O(M)
+
\underbrace{
\sum_{t=M+1}^{T} 
\min \left( 2 B_{a_t,t}^{c}, BD \right)
}_{R_{M+1:T}^{\tilde{\boldsymbol{c}}}}+
\underbrace{
\sum_{t=M+1}^{T} 2 \Vert\boldsymbol{\overline{c}}\Vert_1 \sqrt{\frac{ \log(t/\alpha)}{N_{a,t}}} 
}_{R_{M+1:T}^{\tilde{\boldsymbol{r}}}},
\end{aligned}
\end{equation}

Next we analyze two components of $R_{M+1:T}^{\tilde{\boldsymbol{c}}}$ and $R_{M+1:T}^{\tilde{\boldsymbol{r}}}$ separately.

\textbf{Step-1. (Upper-Bound over $R_{M+1:T}^{\tilde{\boldsymbol{c}}}$)}

We first present two useful lemmas:

\begin{lemma}
\label{lemma: iter_E_upsilon}
For any action sequence of ${a_1},...,{a_T}$ and any $M \in [0,T]$, we have
\[
\emph{det} \left( \mathbb{E}[\boldsymbol{V}_{T}] \right) 
\geq
\emph{det} \left( \mathbb{E}[\boldsymbol{V}_{M}] \right) 
\prod_{t=M}^{T-1} 
\left(
1 + 
% \frac{\emph{det} \left( \Sigma_{r,a_t} \right)}{\emph{det} \left( \mathbb{E}[\boldsymbol{V}_{t-1}] \right) }
% + 
\alpha \boldsymbol{\mu}_{a_t}^{\top} \mathbb{E}[\boldsymbol{V}_{t}]^{-1} \boldsymbol{\mu}_{a_t}
\right).
\]
\end{lemma}
Please see Appendix~\ref{sec: proof_lemma_iter_E_upsilon} for the detailed proof of Lemma~\ref{lemma: iter_E_upsilon}.

\begin{lemma}
\label{lemma: log_E_upsilon}
For any action sequence of ${a_1},...,{a_T}$, any weight $\alpha > 0$, and any $M \in [1,T]$, we have
\[
\log \left( \frac{ \emph{det} \left(\mathbb{E}[\boldsymbol{V}_{T}] \right) }{ \emph{det} \left(\mathbb{E}[\boldsymbol{V}_{M}] \right)} \right)
\leq 
D \log \left( 1 + \frac{\alpha}{\lambda (1-\alpha)}(T-M) \right).
\]
\end{lemma}
Please see Appendix~\ref{sec: proof_lemma_log_E_upsilon} for the detailed proof of Lemma~\ref{lemma: log_E_upsilon}.

Under Claim \ref{claim_1}, for any $t>1$ and $i \in [K]$, we have  
\[
\left \Vert \boldsymbol{\hat{r}}_{i,t} + \sqrt{\frac{ \log(t/\rho)}{N_{i,t}}} \boldsymbol{1} - \boldsymbol{\mu}_{i} \right \Vert_2
\leq
\left \Vert 2 \sqrt{\frac{ \log(t/\rho)}{N_{i,t}}} \boldsymbol{1} \right \Vert_2
=
2 \sqrt{D \frac{ \log(t/\rho)}{N_{i,t}}} 
=
2 \sqrt{D} \gamma_{i,t}.
\]

Consequently,
\[
\begin{aligned}
\left \Vert \boldsymbol{\hat{r}}_{i,t} + \sqrt{\log(t/\rho)/N_{i,t}} \boldsymbol{1} \right \Vert_{\boldsymbol{V}_{t}^{-1}}
&-
\Vert \boldsymbol{\mu}_{i} \Vert_{\boldsymbol{V}_{t}^{-1}}
\leq
\left \Vert \boldsymbol{\hat{r}}_{i,t} + \sqrt{\log(t/\rho)/N_{i,t}} \boldsymbol{1} 
-
\boldsymbol{\mu}_{i} \right \Vert_{\boldsymbol{V}_{t}^{-1}} \\
& \leq
\frac{1}{\sqrt{\lambda}} 
\left \Vert \boldsymbol{\hat{r}}_{i,t} + \sqrt{\log(t/\rho)/N_{i,t}} \boldsymbol{1} 
-
\boldsymbol{\mu}_{i} \right \Vert_{2} 
\leq
2 \sqrt{\frac{D}{\lambda}} \gamma_{i,t}, \\
\end{aligned}
\]
\begin{equation}
\label{eq:expect_r_gap_hidden}
\implies
\left \Vert \boldsymbol{\hat{r}}_{i,t} + \sqrt{\log(t/\rho)/N_{i,t}} \boldsymbol{1} \right \Vert_{\boldsymbol{V}_{t}^{-1}}
\leq 
\Vert \boldsymbol{\mu}_{i} \Vert_{\boldsymbol{V}_{t}^{-1}}
+
2 \sqrt{\frac{D}{\lambda}} \gamma_{i,t}.
\end{equation}

Define $M = \left\lfloor \min \big \{ t^{\prime} \mid t  \nu^2_{r \downarrow} + \lambda \geq 2D \alpha \sqrt{Kt\log \frac{t}{\rho} }, \forall t \geq t^{\prime} \big \} \right \rfloor$. 
Please note that for $\nu^2_{r \downarrow} >0$, we have $\lim_{t \rightarrow \infty}\frac{2D \alpha \sqrt{Kt \log \frac{t}{\rho} }}{\nu^2_{r \downarrow}t} = \lim_{t \rightarrow \infty} C_1 \sqrt{ \frac{ \log (t) - C_2 }{t}} = 0$ since as $t$ increase, $\sqrt{\log t}$ grows much slowly compared to $\sqrt{t}$. Hence for sufficiently large $t^{\prime}$, the inequality $t \nu^2_{r \downarrow} + \lambda \geq 2D \alpha  \sqrt{Kt\log \frac{t}{\rho} }, \forall t \geq t^{\prime}$ holds, which implies that such an $M$ does indeed exist.
By Lemma~\ref{lemma: hidden_sum_reg_c_expectation_bd}, for any $t \in \left[ M+1, T \right]$, we have 

\begin{equation}
\begin{aligned}
\label{eq: hidden_regret_c}
R_{M+1:T}^{\tilde{\boldsymbol{c}}}
& =
\sum_{t=M+1}^{T} \min \left(
2 B_{a_t,t}^{c}, BD \right)  \\
& =
\sum_{t=M+1}^{T} \min \left( 
2 \beta_t 
\left \Vert \boldsymbol{\hat{r}}_{a_t,t} + \sqrt{ \log(t/\rho)/N_{a_t,t}} \boldsymbol{1} \right \Vert_{\boldsymbol{V}_{t}^{-1}},
BD \right) \\
& \underset{(a)}{\leq}
\sum_{t=M+1}^{T}
\min \left( 2 \beta_t 
\left(
\Vert \boldsymbol{\mu}_{a_t} \Vert_{\boldsymbol{V}_{t}^{-1}}
+
\frac{1}{\sqrt{\lambda}} \gamma_{a_t,t}
\right), 
BD \right)\\
& \underset{(b)}{\leq}
\sum_{t=M+1}^{T} \min \left( 
2 \beta_t 
\left(
\sqrt{2} \Vert \boldsymbol{\mu}_{a_t} \Vert_{\mathbb{E}\left[\boldsymbol{V}_{t}\right]^{-1}}
+
\frac{1}{\sqrt{\lambda}} \gamma_{a_t,t}
\right), 
BD \right)\\
& \underset{(c)}{\leq}
2\sqrt{2} 
\underbrace{
\sum_{t=M+1}^{T}
\min \left( \beta_t
\Vert \boldsymbol{\mu}_{a_t} \Vert_{\mathbb{E}\left[\boldsymbol{V}_{t}\right]^{-1}},
\frac{BD}{2\sqrt{2}}  \right)
}_{\text{Reg}^{(1)}}
+
4 \underbrace{
\sum_{t=M+1}^{T}
\beta_t \sqrt{\frac{D}{\lambda}} \gamma_{a_t,t}
}_{\text{Reg}^{(2)}},
\end{aligned}
\end{equation}

\textbf{Step-1.1. (Upper-Bound over $\text{Reg}_{t}^{(1)}$)}

Plugging the $\beta_t = D (32 e^{4B} + 16 e^{2B}) \sqrt{ \frac{ (1-\alpha) 
\log(t/\rho)}{\overline{m}_t t} } + B\sqrt{2 (1-\alpha) \lambda} + \alpha \frac{BD}{2} \sqrt{D \log\big(\frac{1+\alpha D(t-1) / (\lambda+1)}{\rho} \big) + \log(\frac{1+\lambda}{\lambda \rho})}$ (defined in Lemma \ref{lemma:ucb_C_BF}), we have

\[
\begin{aligned}
& \text{Reg}^{(1)}
\leq 
D (32 e^{4B} + 16 e^{2B}) \sqrt{ \frac{ (1-\alpha) }{\overline{m}_t} } \sum_{t=M+1}^{T} \min \left(\sqrt{\frac{\log(t/\rho)}{t}} \Vert \boldsymbol{\mu}_{a_t} \Vert_{\mathbb{E}\left[\boldsymbol{V}_{t}\right]^{-1}}, \frac{BD}{2\sqrt{2}} \right) \\
& +
\left(B\sqrt{2\lambda(1-\alpha)} + \alpha \frac{BD}{2} \sqrt{D \log\big(\frac{1+\alpha DT / (\lambda+1)}{\rho} \big) + \log(\frac{1+\lambda}{\lambda \rho})} \right) \sum_{t=M+1}^{T} \min \left( \Vert \boldsymbol{\mu}_{a_t} \Vert_{\mathbb{E}\left[\boldsymbol{V}_{t}\right]^{-1}}, \frac{BD}{2\sqrt{2}} \right)
\end{aligned}
\]

By Cauchy–Schwarz inequality, we have
\begin{equation}
\label{eq: Reg_1}
\begin{aligned}
& \text{Reg}^{(1)}
\leq 
D (32 e^{4B} + 16 e^{2B}) \sqrt{ \frac{ (1-\alpha) }{\overline{m}_t} } \sqrt{\sum_{t=M+1}^{T} 1 \cdot \underbrace{\sum_{t=M+1}^{T} \min \left(\frac{\log(t/\rho)}{t} \Vert \boldsymbol{\mu}_{a_t} \Vert^2_{\mathbb{E}\left[\boldsymbol{V}_{t} \right]^{-1}}, \frac{B^2 D^2}{8} \right)}_{I_1}} \\
& +
\left(B\sqrt{2\lambda(1-\alpha)} + \alpha \frac{BD}{2} \sqrt{D \log\big(\frac{1+\alpha DT / (\lambda+1)}{\rho} \big) + \log(\frac{1+\lambda}{\lambda \rho})} \right) \sqrt{\sum_{t=M+1}^{T} 1 \cdot \underbrace{\sum_{t=M+1}^{T} \min \left( \Vert \boldsymbol{\mu}_{a_t} \Vert_{\mathbb{E}\left[\boldsymbol{V}_{t}\right]^{-1}}^2, \frac{B^2 D^2}{8} \right)}_{I_2}}.
\end{aligned}
\end{equation}

For $\boldsymbol{I_1}$, note for sufficiently large enough iterations $t \ge \tfrac{\log (T/\rho)}{\alpha}$, we have $\tfrac{\log (T/\rho)}{t} \le \alpha$, thus
\begin{equation}
\begin{aligned}
\label{eq: regret_I_1}
\sum_{t=M+1}^{T} \min \left(\frac{\log(t/\rho)}{t} \Vert \boldsymbol{\mu}_{a_t} \Vert^2_{\mathbb{E}\left[\boldsymbol{V}_{t} \right]^{-1}}, \frac{B^2 D^2}{8} \right)
& \le 
\sum_{t=M+1}^{T} \min \left(\alpha \Vert \boldsymbol{\mu}_{a_t} \Vert^2_{\mathbb{E}\left[\boldsymbol{V}_{t} \right]^{-1}}, \frac{B^2 D^2}{8} \right) \\
& \underset{(a)}{\le} 
c_0 \sum_{t=1}^T \log\left( 1 + \alpha \Vert \boldsymbol{\mu}_{a_t} \Vert^2_{\mathbb{E}\left[\boldsymbol{V}_{t} \right]^{-1}} \right) \\
& \underset{(b)}{\le} 
c_0 D \log \left( 1 + \frac{\alpha}{\lambda (1-\alpha)}(T-M) \right),
\end{aligned}
\end{equation}
where (a) holds by using the inequality $c_0 \log(1+x) \geq \frac{B^2D^2}{8 \log(1+\tfrac{B^2D^2}{8})} \log(1 + x) \ge x$, with some constant $c_0$ for $x \in [0, \tfrac{B^2 D^2}{8}]$,
last inequality (b) holds by chaining the results of Lemma \ref{lemma: iter_E_upsilon} and Lemma \ref{lemma: log_E_upsilon}, and plugging back to the result above.

For $\boldsymbol{I_2}$, similarly, we have:
\begin{equation}
\begin{aligned}
\label{eq: regret_I_2}
\sum_{t=M+1}^{T} \min \left( \Vert \boldsymbol{\mu}_{a_t} \Vert_{\mathbb{E}\left[\boldsymbol{V}_{t}\right]^{-1}}^2, \frac{B^2 D^2}{8} \right)
&\le 
\frac{1}{\alpha} \sum_{t=M+1}^{T} \min \left(\alpha \Vert \boldsymbol{\mu}_{a_t} \Vert^2_{\mathbb{E}\left[\boldsymbol{V}_{t} \right]^{-1}}, \frac{B^2 D^2}{8} \right) \\
& \le
\frac{c_0}{\alpha}  \sum_{t=1}^T \log\left( 1 + \alpha \Vert \boldsymbol{\mu}_{a_t} \Vert^2_{\mathbb{E}\left[\boldsymbol{V}_{t} \right]^{-1}} \right) \\
& \le
\frac{c_0 D}{\alpha} \log \left( 1 + \frac{\alpha}{\lambda (1-\alpha)}(T-M) \right).
\end{aligned}
\end{equation}

Combining Eq. \ref{eq: regret_I_1} and Eq. \ref{eq: regret_I_2} yields for $t \ge \tfrac{\log (T/\rho)}{\alpha}$,
\begin{equation}
\label{eq: Reg_1_result}
\begin{aligned}
& \text{Reg}^{(1)}
\leq 
D (32 e^{4B} + 16 e^{2B}) \sqrt{ \frac{ c_0 D (1-\alpha) }{\overline{m}_t} } \sqrt{ (T-M) \log \left( 1 + \frac{\alpha}{\lambda (1-\alpha)}(T-M) \right) } \\
& +
\left(B\sqrt{\frac{2\lambda(1-\alpha) c_0 D}{\alpha}} +  \frac{BD}{2} \sqrt{\alpha c_0 D} \sqrt{D \log\big(\frac{1+\alpha DT / (\lambda+1)}{\rho} \big) + \log(\frac{1+\lambda}{\lambda \rho})} \right) \sqrt{ (T-M) \log \left( 1 + \frac{\alpha}{\lambda (1-\alpha)}(T-M) \right) }.
\end{aligned}
\end{equation}

\textbf{Step-1.2. (Upper-Bound over $\text{Reg}_{t}^{(2)}$)}

Plugging the $\beta_t = D (32 e^{4B} + 16 e^{2B}) \sqrt{ \frac{ (1-\alpha) 
\log(t/\rho)}{\overline{m}_t t} } + B\sqrt{2 (1-\alpha) \lambda} + \alpha \frac{BD}{2} \sqrt{D \log\big(\frac{1+\alpha D(t-1) / (\lambda+1)}{\rho} \big) + \log(\frac{1+\lambda}{\lambda \rho})}$ (defined in Lemma \ref{lemma:ucb_C_BF}), we have

\[
\begin{aligned}
& \text{Reg}^{(2)}
\leq 
D^2 (32 e^{4B} + 16 e^{2B}) \sqrt{ \frac{(1-\alpha) }{m_T^{\downarrow} \lambda } } \sum_{t=M+1}^{T} \sqrt{\frac{ \log^2(t/\rho)}{t N_{a_t,t}}}  \\
& +
\left(B\sqrt{2\lambda(1-\alpha)} + \alpha \frac{BD}{2} \sqrt{D \log\big(\frac{1+\alpha DT / (\lambda+1)}{\rho} \big) + \log(\frac{1+\lambda}{\lambda \rho})} \right) 
\sqrt{\frac{D}{\lambda}}
\sum_{t=M+1}^{T} \sqrt{\frac{ \log(t/\rho)}{N_{a_t,t}}}
\end{aligned}.
\]

By Lemma \ref{lemma: hidden_sum_reg_c_2} and Lemma \ref{lemma: hidden_sum_reg_c_3}, we have
\begin{equation}
\label{eq: Reg_2_result}
\begin{aligned}
& \text{Reg}^{(2)}
\leq 
O \left(D (32 e^{4B} + 16 e^{2B}) \sqrt{ \frac{ D(1-\alpha) }{\lambda } } \sqrt{K  \log^2(\frac{T}{\rho})} \right)  \\
& +
O \left(\Big(B\sqrt{2D(1-\alpha)} + \alpha \frac{BD}{2} \sqrt{\frac{D^2}{\lambda} \log\big(\frac{1+\alpha DT / (\lambda+1)}{\rho} \big) + \frac{D}{\lambda}\log(\frac{1+\lambda}{\lambda \rho})} \Big) 
\sqrt{ K T \log (\frac{T}{\rho} )} \right)
\end{aligned}.
\end{equation}

\textbf{Step-2. (Upper-Bound over $R_{M+1:T}^{\tilde{\boldsymbol{r}}}$)}

\begin{equation}
\begin{aligned}
\label{eq: Reg_r_result}
\sum_{t=M+1}^{T} 2 \Vert\boldsymbol{\overline{c}}\Vert_1 \sqrt{\frac{ \log(t/\alpha)}{N_{a,t}}} 
& \leq  
2 BD \sum_{t=M+1}^{T} \sqrt{\frac{ \log(t/\alpha)}{N_{a,t}}} 
\leq
4 BD \sqrt{ K (T-M) \log ( \frac{T}{\rho})}
\end{aligned}
\end{equation}
where the final inequality holds by Lemma \ref{lemma: hidden_sum_reg_c_2}.

\textbf{Step-3. (Combine the Bounds)}
Summing up the results of Eq. \ref{eq: Reg_1_result}, Eq. \ref{eq: Reg_2_result} and Eq. \ref{eq: Reg_r_result} together complete the proof.
\end{proof}

\subsection{Proof of Lemma \ref{lemma: iter_E_upsilon}}
\label{sec: proof_lemma_iter_E_upsilon}

\begin{proof}[Proof of Lemma~\ref{lemma: iter_E_upsilon}]

Let $\boldsymbol{r}_{a_t, t}' = \sqrt{\alpha} \boldsymbol{r}_{a_t, t}$, $\boldsymbol{\mu}_{a_t}'= \sqrt{\alpha} \boldsymbol{\mu}_{a_t}$. For $\boldsymbol{V}_t$ and $\mathbb{E}[\boldsymbol{V}_t]$, by definition,

\[
\boldsymbol{V}_{t+1}
=
\boldsymbol{V}_{t} + \alpha \boldsymbol{r}_{a_t, t} \boldsymbol{r}_{a_t, t}^{\top}
\quad \text{and} \quad
\boldsymbol{V}_{1} = \lambda \boldsymbol{I} + \boldsymbol{U}
\]
\[
\mathbb{E}[\boldsymbol{V}_{t+1}]
=
\mathbb{E}[\boldsymbol{V}_{t} + \alpha \boldsymbol{r}_{a_t, t} \boldsymbol{r}_{a_t, t}^{\top}] 
=
\mathbb{E}[\boldsymbol{V}_{t} + \boldsymbol{r}_{a_t, t}' \boldsymbol{r}_{a_t, t}'^{\top}] 
=
\mathbb{E}[\boldsymbol{V}_{t}] 
+ \boldsymbol{\mu}_{a_t}' \boldsymbol{\mu}_{a_t}'^{\top} + \alpha \Sigma_{r, a_t}.
\]

Since $\mathbb{E}[\boldsymbol{V}_{t}]$ is symmetric and positive definite, we have 

\begin{small}
\begin{equation}
\begin{aligned}
\label{eq: iter_E_upsilon1}
\text{det} \left( \mathbb{E}[\boldsymbol{V}_{t+1}] \right)
& = 
\text{det} \left( \mathbb{E}[\boldsymbol{V}_{t}] + \boldsymbol{\mu}_{a_t}' \boldsymbol{\mu}_{a_t}'^{\top} + \alpha \Sigma_{r, a_t} \right) \\
& = 
\text{det} \left( 
\mathbb{E}[\boldsymbol{V}_{t}]^{\frac{1}{2}} \left( 
\boldsymbol{I} + \mathbb{E}[\boldsymbol{V}_{t}]^{-\frac{1}{2}} \left (\boldsymbol{\mu}_{a_t}' \boldsymbol{\mu}_{a_t}'^{\top} + \alpha\Sigma_{r,a_t} \right) \mathbb{E}[\boldsymbol{V}_{t}]^{-\frac{1}{2}}
\right)  \mathbb{E}[\boldsymbol{V}_{t-1}]^{\frac{1}{2}}
\right) \\
& = 
\text{det} \left( 
\mathbb{E}[\boldsymbol{V}_{t-1}] \right) 
\text{det} \left( 
\boldsymbol{I} + \mathbb{E}[\boldsymbol{V}_{t}]^{-\frac{1}{2}} \left (\boldsymbol{\mu}_{a_t}' \boldsymbol{\mu}_{a_t}'^{\top} + \alpha\Sigma_{r,a_t} \right) \mathbb{E}[\boldsymbol{V}_{t}]^{-\frac{1}{2}}
\right) \\
& \underset{(a)}{\geq}
\text{det} \left( 
\mathbb{E}[\boldsymbol{V}_{t}] \right) 
\left(
\text{det} \left( 
\boldsymbol{I} + \mathbb{E}[\boldsymbol{V}_{t}]^{-\frac{1}{2}} \boldsymbol{\mu}_{a_t}' \boldsymbol{\mu}_{a_t}'^{\top} \mathbb{E}[\boldsymbol{V}_{t}]^{-\frac{1}{2}}
\right)
+ \text{det} \left( 
\mathbb{E}[\boldsymbol{V}_{t}]^{-\frac{1}{2}}  \alpha\Sigma_{r,a_t} \mathbb{E}[\boldsymbol{V}_{t}]^{-\frac{1}{2}}
\right)
\right) \\
\end{aligned}
\end{equation}
\end{small}

where (a) holds since $\left( 
\boldsymbol{I} + \mathbb{E}[\boldsymbol{V}_{t}]^{-\frac{1}{2}} \boldsymbol{\mu}_{a_t}' \boldsymbol{\mu}_{a_t}'^{\top} \mathbb{E}[\boldsymbol{V}_{t}]^{-\frac{1}{2}}
\right)$ and $\left( 
\mathbb{E}[\boldsymbol{V}_{t}]^{-\frac{1}{2}}  \alpha\Sigma_{r,a_t} \mathbb{E}[\boldsymbol{V}_{t}]^{-\frac{1}{2}}
\right)$ are positive definite and applying Lemma~\ref{lemma: symmetric_det} yields the result.

Let $\mathbb{E}[\boldsymbol{V}_{t-1}]^{-\frac{1}{2}} \boldsymbol{\mu}_{a_t}' = \boldsymbol{v}_t$, and we observe that 
\[
\left( \boldsymbol{I} + \boldsymbol{v}_t \boldsymbol{v}_t^{\top} \right) \boldsymbol{v}_t
=
\boldsymbol{v}_t + \boldsymbol{v}_t \left( \boldsymbol{v}_t^{\top} \boldsymbol{v}_t \right)
=
\left(1 + \boldsymbol{v}_t^{\top} \boldsymbol{v} \right) \boldsymbol{v}_t.
\]

Hence, $1 + \boldsymbol{v}_t^{\top} \boldsymbol{v}$ is an eigenvalue of $\boldsymbol{I} + \boldsymbol{v}_t \boldsymbol{v}_t^{\top}$. And since $\boldsymbol{v}_t \boldsymbol{v}_t^{\top}$ is a rank-1 matrix, all other eigenvalue of $\boldsymbol{I} + \boldsymbol{v}_t \boldsymbol{v}_t^{\top}$ equal to 1, implying

\begin{equation}
\begin{aligned}
\label{eq: iter_E_upsilon2}
\text{det} \left( 
\boldsymbol{I} + \mathbb{E}[\boldsymbol{V}_{t}]^{-\frac{1}{2}} \boldsymbol{\mu}_{a_t}' \boldsymbol{\mu}_{a_t}'^{\top} \mathbb{E}[\boldsymbol{V}_{t}]^{-\frac{1}{2}}
\right) 
& =
\text{det} \left( \boldsymbol{I} + \boldsymbol{v}_t \boldsymbol{v}_t^{\top} \right)\\
& =
1 + \boldsymbol{v}_t \boldsymbol{v}_t^{\top} \\
& =
1 + \left( \mathbb{E}[\boldsymbol{V}_{t}]^{-\frac{1}{2}} \boldsymbol{\mu}_{a_t}' \right)^{\top} \left( \mathbb{E}[\boldsymbol{V}_{t}]^{-\frac{1}{2}} \boldsymbol{\mu}_{a_t}' \right) \\
& = 
1 + \boldsymbol{\mu}_{a_t}'^{\top} \mathbb{E}[\boldsymbol{V}_{t}]^{-1} \boldsymbol{\mu}_{a_t}'.
\end{aligned}
\end{equation}

Combining Eq.~\ref{eq: iter_E_upsilon1} and Eq.~\ref{eq: iter_E_upsilon2}, we have

\[
\begin{aligned}
\text{det} \left( \mathbb{E}[\boldsymbol{V}_{t+1}] \right)
& \geq
\text{det} \left( 
\mathbb{E}[\boldsymbol{V}_{t}] \right) 
\left(
1 + \boldsymbol{\mu}_{a_t}'^{\top} \mathbb{E}[\boldsymbol{V}_{t}]^{-1} \boldsymbol{\mu}_{a_t}'
+ \text{det} \left( 
\mathbb{E}[\boldsymbol{V}_{t}]^{-\frac{1}{2}}  \Sigma_{r,a_t} \mathbb{E}[\boldsymbol{V}_{t}]^{-\frac{1}{2}}
\right)
\right) \\
& \geq 
\text{det} \left( 
\mathbb{E}[\boldsymbol{V}_{t}] \right) 
\left(
1 + \boldsymbol{\mu}_{a_t}'^{\top} \mathbb{E}[\boldsymbol{V}_{t}]^{-1} \boldsymbol{\mu}_{a_t}'
\right) 
\end{aligned}
\]

The solution of Lemma~\ref{lemma: iter_E_upsilon} follows from induction.
\end{proof}

\subsection{Proof of Lemma \ref{lemma: log_E_upsilon}}
\label{sec: proof_lemma_log_E_upsilon}

\begin{proof}

By the definition of $\boldsymbol{V}_{t}$, we have 

\begin{equation}
\begin{aligned}
\label{eq: log_E_upsilon1}
\log \left( \frac{ \text{det} \left(\mathbb{E}[\boldsymbol{V}_{T}] \right) }{ \text{det} \left(\mathbb{E}[\boldsymbol{V}_{M}] \right)} \right)
& =
\log \left( \text{det} \left( \frac{ \mathbb{E}[\boldsymbol{V}_{M}] + \sum_{t=M}^{T-1} \mathbb{E}[ \alpha \boldsymbol{r}_{a_t,t} \boldsymbol{r}_{a_t,t}^{\top}] }{ \mathbb{E}[\boldsymbol{V}_{M}] } \right) \right) \\
& \underset{(a)}{\leq}
\log \left( \text{det} \left( \Big( (1-\alpha) \boldsymbol{U}_{\lambda} + \sum_{t=M}^{T-1} \mathbb{E}[ \alpha \boldsymbol{r}_{a_t,t} \boldsymbol{r}_{a_t,t}^{\top}] \Big) \frac{1}{1-\alpha} \boldsymbol{U}_{\lambda}^{-1} \right) \right) \\
& = 
\log \left( \text{det} \Big( \boldsymbol{I} +  \frac{\alpha}{1-\alpha} \sum_{t=M}^{T-1} \mathbb{E}[ \boldsymbol{r}_{a_t,t} \boldsymbol{r}_{a_t,t}^{\top}] \boldsymbol{U}_{\lambda}^{-1}  \Big) \right), \\
\end{aligned}
\end{equation}

where (a) holds since $\text{det}(\mathbb{E}[\boldsymbol{V}_{M}]) \geq \text{det}(\mathbb{E}[\boldsymbol{V}_{1}]) =  \boldsymbol{U}_{\lambda}$. 
By Sherman–Morrison formula, we have:
\begin{equation}
\label{eq:U_lambda_sherman}
\begin{aligned}
U_\lambda^{-1} 
&= \left( \lambda I + I - \tfrac{1}{D} \mathbf{1} \mathbf{1}^\top \right)^{-1}
= \left( (\lambda + 1) I - \tfrac{1}{D} \mathbf{1} \mathbf{1}^\top \right)^{-1} \\
&= \frac{1}{\lambda + 1} I 
+ \frac{ \frac{1}{D(\lambda + 1)^2} \mathbf{1} \mathbf{1}^\top }{ 1 - \frac{1}{D} \cdot \frac{1}{\lambda + 1} \cdot \mathbf{1}^\top \mathbf{1} } 
= \frac{1}{\lambda + 1} I 
+ \frac{ \frac{1}{D(\lambda + 1)^2} \mathbf{1} \mathbf{1}^\top }{ 1 - \frac{1}{\lambda + 1} }
\end{aligned}
\end{equation}

Let $\xi_1, ... , \xi_D$ denote the eigenvalues of $\sum_{t=M}^{T-1} \mathbb{E}[ \boldsymbol{r}_{a_t,t} \boldsymbol{r}_{a_t,t}^{\top}] \boldsymbol{U}_{\lambda}^{-1}$, and note:

\begin{equation}
\begin{aligned}
\label{eq: log_E_upsilon2}
\sum_{d=1}^{D} \xi_d 
& =
\text{Trace} \left( \sum_{t=M}^{T-1} \mathbb{E}[ \boldsymbol{r}_{a_t,t} \boldsymbol{r}_{a_t,t}^{\top}] \boldsymbol{U}_{\lambda}^{-1} \right) \\
& \underset{(a)}{=}
\frac{\sum_{t=M}^{T-1} \mathbb{E} \left[\text{Trace} \left( \boldsymbol{r}_{a_t,t} \boldsymbol{r}_{a_t,t}^{\top} \right) \right]}{\lambda+1} 
+
\frac{\sum_{t=M}^{T-1} \text{Trace} \left(  \mathbb{E}[ \boldsymbol{r}_{a_t,t} \boldsymbol{r}_{a_t,t}^{\top}] \mathbf{1} \mathbf{1}^\top \right)}{D(\lambda+1)\lambda} 
\\
& \underset{(b)}{=}
\frac{\sum_{t=M}^{T-1} \mathbb{E} \left[\Vert \boldsymbol{r}_{a_t,t} \Vert_2^2 \right]}{\lambda+1} 
+
\frac{\sum_{t=M}^{T-1} \mathbb{E}[ \Vert \boldsymbol{r}_{a_t,t} \Vert_1^2 ]}{D(\lambda+1)\lambda} 
\\
& \underset{(c)}{\leq}
\frac{D(T-M)}{\lambda+1} + \frac{D^2(T-M)}{D(\lambda+1)\lambda} \\
& =
\frac{D}{\lambda}(T-M).
\end{aligned}
\end{equation}

where (a) holds by Eq. \ref{eq:U_lambda_sherman}, (c) holds since we assume $\boldsymbol{r}_t \leq [1]^D$, (b) holds since 
\[
\begin{aligned}
\text{Trace} \left(  \mathbb{E}[ \boldsymbol{r}_{a_t,t} \boldsymbol{r}_{a_t,t}^{\top}] \mathbf{1} \mathbf{1}^\top \right)
& =
\mathbb{E}[ \text{Trace} \left(   \boldsymbol{r}_{a_t,t} \boldsymbol{r}_{a_t,t}^{\top} \mathbf{1} \mathbf{1}^\top  \right) ]
=
\mathbb{E}[ (\boldsymbol{r}_{a_t,t}^{\top} \mathbf{1}) \text{Trace} \left(   \boldsymbol{r}_{a_t,t}  \mathbf{1}^\top  \right) ] \\
& =
\mathbb{E}[ (\boldsymbol{r}_{a_t,t}^{\top} \mathbf{1})^2 ]
= 
\mathbb{E}[ \Vert \boldsymbol{r}_{a_t,t} \Vert_1^2 ].
\end{aligned}
\]

Combining Eq.~\ref{eq: log_E_upsilon1} and Eq.~\ref{eq: log_E_upsilon2} implies 

\[
\begin{aligned}
\log \left( \frac{ \text{det} \left(\mathbb{E}[\boldsymbol{V}_{T}] \right) }{ \text{det} \left(\mathbb{E}[\boldsymbol{V}_{M}] \right)} \right)
& \leq
\log \left( \text{det} \Big( \boldsymbol{I} +  \frac{\alpha}{1-\alpha} \sum_{t=M}^{T-1} \mathbb{E}[ \boldsymbol{r}_{a_t,t} \boldsymbol{r}_{a_t,t}^{\top}] \boldsymbol{U}_{\lambda}^{-1}  \Big) \right) \\
& =
\log \left( \prod_{i=1}^{D} \left( 1 + \frac{\alpha}{1-\alpha}\xi_i \right) \right) \\
& =
D \log \left( \prod_{i=1}^{D} \left(  1 + \frac{\alpha}{1-\alpha}\xi_i \right) \right)^{\frac{1}{D}} \\
& \underset{(a)}{\leq}
D \log \left( \frac{1}{D} \sum_{i=1}^{D} \left( 1 + \frac{\alpha}{1-\alpha}\xi_i \right) \right) \\
& \underset{(b)}{\leq}
D \log \left( 1 +  \frac{\alpha D (T-M)}{D \lambda (1-\alpha)} \right), \\
\end{aligned}
\]
where (a) follows from the inequality of arithmetic and geometric means, and (b) follows from Eq.~\ref{eq: log_E_upsilon2}.
\end{proof}

\section{Proof of Theorem \ref{theorem:lw_bd_corrupt}}
\label{sec:pf_thm_lw_bd_corrupt}

\begin{proof}

We prove the result with a constructed instance.

\textbf{Construct instances with a pair of two different zero-mean preference vectors:}
\[
c =
\begin{cases}
\frac{\Delta}{2}, & i \le \lfloor \frac{D}{2} \rfloor \qquad\text{    (upper list)} \\
-\frac{\Delta}{2}, & \lfloor \frac{D}{2} \rfloor < i \le D \text{ (lower list)}
\end{cases},
\quad
c' =
\begin{cases}
-\frac{\Delta}{2}, & i \le \lfloor \frac{D}{2} \rfloor \qquad\text{    (upper list)} \\
\frac{\Delta}{2}, & \lfloor \frac{D}{2} \rfloor < i \le D \text{ (lower list)}
\end{cases}
\]

Let $P = \frac{e^{\Delta/2}}{e^{\Delta/2} + e^{-\Delta/2}} = \sigma(\Delta)$ denote the probability of objectives with preference value $\frac{\delta}{2}$ preferred over objectives with preference of $-\frac{\delta}{2}$.
Under $\epsilon$-corruption model with arbitrary flipping, we can easily verify that the corrupted pairwise comparison probability with largest distribution shift is caused by entries swaps between upper and lower lists, resulting the pairwise comparison probability between objectives in upper list and lower list as:
\[
\widetilde{P}^{(c)}_{ij} = (1-\epsilon) P + \epsilon (1 - P) = (1 - 2\epsilon) P + \epsilon
\]
\[
\widetilde{P}^{(c')}_{ij} = (1-\epsilon)(1 - P) + \epsilon P = (2\epsilon - 1) P - \epsilon + 1
\]

We leverage the ranking-breaking strategy that reduces the full ranking $\sigma$ into independent pairwise comparisons, where each compared objective being preferred follows a Bernoulli distribution shown above.
Let $\sigma = (i_1 \succ i_2 \succ \dots \succ i_D)$ be a full ranking over $D$ items. We model the likelihood as:
\[
P(\sigma \mid {c}) = \prod_{(i \succ j) \in \sigma} \mathrm{Ber}(1; \widetilde{P}^{(c)}_{ij})
\]
where each pairwise outcome $Y_{ij}$ is drawn independently as:
\[
Y_{ij} \sim \mathrm{Ber}(\widetilde{P}^{(c)}_{ij})
\]
Thus, the full-data distribution becomes a product of independent Bernoulli comparisons:
\[
P(\{Y_{ij}\} \mid {c}) = \prod_{(i, j) \in S} \mathrm{Ber}(Y_{ij}; \widetilde{P}^{(c)}_{ij})
\]

Now consider set $S$ of all $(i,j)$ pairs where $i \in [\lfloor \frac{D}{2} \rfloor]$, $j \in [\lfloor \frac{D}{2} \rfloor+1, D]$ with $|S| = \frac{D^2}{4}$.
Then the KL divergence between two corrupted models with underlying $c$ and $c'$ is given as
\[
\mathrm{KL}(P^{(c)} \| P^{(c')}) = \sum_{(i, j) \in S} \mathrm{KL}(\mathrm{Ber}(\widetilde{P}^{(c)}_{ij}) \| \mathrm{Ber}(\widetilde{P}^{(c')}_{ij}))
\]

Define $\delta = \frac{1}{2} \left| \widetilde{P}^{(c)}_{ij} - \widetilde{P}^{(c')}_{ij} \right| = (1 - 2\epsilon)(P - \frac{1}{2}) = (1 - 2\epsilon)\tanh(\Delta/2)$ for $\epsilon < \frac{1}{2}$.

Apply standard bound for Bernoulli KL divergence (since $\widetilde{P}^{(c)}_{ij} + \widetilde{P}^{(c')}_{ij} = 1$):
\[
\mathrm{KL}\left( \mathrm{Ber}(\widetilde{P}^{(c)}_{ij}) \middle\| \mathrm{Ber}(\widetilde{P}^{(c')}_{ij}) \right)
\le 4 \delta^2 = 4(1 - 2\epsilon)^2 \tanh^2(\Delta/2)
\]

Thus we have,
\[
\mathrm{KL}(P^{(c)} \| P^{(c')}) \le |S| \cdot 4(1 - 2\epsilon)^2 \tanh^2(\Delta/2) = D^2(1 - 2\epsilon)^2 \tanh^2(\Delta/2).
\]

Using Pinsker's inequality, for $L$ ranking samples with $\epsilon < \frac{1}{2}$, we have:
\[
\text{TV}((P^{(c)})^L \| (P^{(c')})^L) = \sqrt{\frac{L}{2} \mathrm{KL}(P^{(c)} \| P^{(c')})}
\leq
\sqrt{\frac{L}{2}} D(1 - 2\epsilon) \tanh(\frac{\Delta}{2})
\]

Solving for $\text{TV} = \frac{1}{2}$ gives,
\[
\tanh(\frac{\Delta}{2}) = \frac{1}{\sqrt{2L} D(1 - 2\epsilon) }
\implies
\Delta = c_0 \cdot 2\sqrt{2} \tanh^{-1}(\frac{2}{(1 - 2\epsilon) D \sqrt{L}}),
\]
 with some constant $c_0>0$.
Finally, let $\hat{c}_L$ be any estimator with $L$ samples, by Le Cam's lemma we have:
\[
\begin{aligned}
\inf_{\boldsymbol{\hat{c}}_L} \sup_{\overline{\boldsymbol{c}}^0 \in \{c, c'\} }
\mathbb{E}\big[\Vert \boldsymbol{\hat{c}}_L - \overline{\boldsymbol{c}}^0 \Vert_{ \boldsymbol{2}}\big ]
& \geq
\frac{\Vert c - c' \Vert_2 }{2} (1-\text{TV}((P^{(c)})^L \| (P^{(c')})^L)) \\
& = \frac{ \sqrt{D} \Delta }{4} = c_0 \sqrt{\frac{D}{2} }\tanh^{-1} \left( \frac{2}{(1 - 2\epsilon) D \sqrt{L}} \right),
\end{aligned}
\]
with the valid region of $\frac{2}{(1 - 2\epsilon) D \sqrt{L}} < 1$, yields $\epsilon \leq \frac{1}{2} - O(\frac{1}{D\sqrt{L}})$. Additionally, we further have $\Vert \boldsymbol{\hat{c}}_L - \overline{\boldsymbol{c}}^0 \Vert_2$ be bounded by $2B\sqrt{D}$ due to the bounded region of $c$. 
For $\epsilon > \frac{1}{2}$, it can be easily verify that $P^{(c)} $ and $ P^{(c')}$ are distinguishable, leading to a constant estimator error gap.
\end{proof}

\section{Proof of Lemma \ref{lemma:qe_ucb_corrupt}}
\label{sec:pf_lemma_qe_ucb_corrupt}
\begin{proof}

For notational simplicity, we fix a user $n$ and omit the user index $n$ as superscript or subscript.
Note for any number of sample $t$, we define the error components as $[\delta_i]_t = [ \delta_1, \dots, \delta_t] \in \mathbb{R}^{D \times t}$, where $\delta_i \in \mathbb{R}^D$.
Recall we have the solution for the estimator
\[
(\hat{c}_t^{\text{(QE)}}, [\hat{\delta}_i]_t) = \arg\min_{c \in \Theta_c^0, [\hat{\delta}_i]_t} \; \hat{\mathcal{L}}_{\tilde{S}_t}(c, \delta) + \lambda \sum_{j=1}^{t} \| \delta_j \|_2
\]
where
\[
\hat{\mathcal{L}}_{\tilde{S}_t}(c, [\hat{\delta}_i]_t) = -\frac{1}{t} \sum_{i=1}^t \sum_{j=1}^{m_i} \log \left( \frac{ \exp(c_{\sigma_i(j)} + \delta_{i,\sigma_i(j)}) }{ \sum_{k=d}^{D} \exp(c_{\sigma_i(k)} + \delta_{i,\sigma_i(k)}) } \right)
\]

Define $\Delta_c = \hat{c}^{\text{(QE)}} - \overline{c}^0$, $\Delta_{\delta} = [\hat{\delta}_i]_t - [\delta^*_i]_t \in \mathbb{R}^{D \times T}$ and $\Delta_{\delta}^{\text{flat}} = \text{vec}(\Delta_{\delta}) \in \mathbb{R}^{D T}$, where $\text{vec}(\cdot)$ denotes column-major flattening.
By the definition of the MLE estimator, we have:
\[
\hat{\mathcal{L}}_{\tilde{S}_t}(c, [\hat{\delta}_i]_t) + \lambda \sum_{i=1}^{t} \| \hat{\delta}_i \|_2 
\le \hat{\mathcal{L}}_{\tilde{S}_t}(c, [\delta^*_i]_t) + \lambda \sum_{i=1}^{t} \| \delta_i^* \|_2
\]

Due to strong convexity of $\hat{\mathcal{L}}_{\tilde{S}_t}$ (by Lemma \ref{lemma:convexity_corrupt_loss}), we can rewrite as:
\[
\begin{aligned}    
\eta \sum_{i=1}^t \| \delta_i^* \|_2 - \eta \sum_{i=1}^t \| \hat{\delta}_i \|_2 
& \geq
\hat{\mathcal{L}}_{\tilde{S}_t}(\hat{c}_t^{\text(QE)}, [\hat{\delta}_i]_t) - \hat{\mathcal{L}}_{\tilde{S}_t}(\overline{c}^0, [\delta^*_i]_t) \\
&= \hat{\mathcal{L}}_{\tilde{S}_t}(\hat{c}_t^{\text(QE)}, [\delta^*_i]_t + \Delta_{\delta}) - \hat{\mathcal{L}}_{\tilde{S}_t}(\hat{c}_t^{\text(QE)}, [\delta^*_i]_t) \\
& \qquad +
\hat{\mathcal{L}}_{\tilde{S}_t}(\overline{c}^0 + \Delta_c, [\delta^*_i]_t) - \hat{\mathcal{L}}_{\tilde{S}_t}(\overline{c}^0, [\delta^*_i]_t) \\
& \ge
\nabla_c \hat{\mathcal{L}}_{\tilde{S}_t}(\overline{c}^0, [\delta^*_i]_t)^\top \Delta_c + \nabla_{\delta} \hat{\mathcal{L}}_{\tilde{S}_t}(\hat{c}_t^{\text(QE)}, [\delta^*_i]_t)^\top \Delta_{\delta}^{\text{flat}} \\
& \qquad + 
\frac{t \overline{m}_t}{D(4e^{4B} + 2e^B)} \| \Delta_c \|_2^2
+ \frac{ m_t^{\downarrow} }{ D(4e^{4B} + 2e^B) } \| \Delta_{\delta}^{\text{flat}} \|_2^2.
\end{aligned}
\]

Reorganizing above result gives
\[
\begin{aligned}    
& \frac{t \overline{m}_t}{D(4e^{4B} + 2e^B)} \| \Delta_c \|_2^2 + 
\frac{ m_t^{\downarrow} }{ D(4e^{4B} + 2e^B) } \| \Delta_{\mathcal{S}}^{\text{flat}} \|_2^2 \\
& \quad \le 
\eta \sum_{i=1}^t \| \hat{\delta}_i^* \|_2 - \eta \sum_{i=1}^t \| \hat{\delta}_i \|_2 
- \nabla_c \hat{\mathcal{L}}_{\tilde{S}_t}(\overline{c}^0, [\delta^*_i]_t)^\top \Delta_c 
- \nabla_{\mathcal{S}} \hat{\mathcal{L}}_{\tilde{S}_t}(\hat{c}^{\text{QE}}_t, [\delta^*_i]_t)^\top \Delta_{\mathcal{\delta}}^{\text{flat}}.
\end{aligned}
\]

Let $\mathcal{T}_t^c$ denote the set of time steps where ranking is corrupted within $t$ samples, and $|\mathcal{T}_t^c| = N_t^c$, let $\Delta_{\delta+}^{\text{flat}} = \Delta_{\delta}^{\text{flat}}[\mathcal{T}_t^c]$ denote the gap of estimated $\delta$ over corrupted samples. 
By Lemma \ref{lemma:eta_setting}, with $\eta = \frac{ 2D + 9p + 2p^2 D^2 + 3 + 6(D^2 + t) }{27}$,
we have:
\[
\begin{aligned}
& \frac{t \overline{m}_t}{D(4e^{4B} + 2e^B)} \| \Delta_c \|_2^2 + 
\frac{ m_t^{\downarrow} }{ D(4e^{4B} + 2e^B) } \| \Delta_{\delta}^{\text{flat}} \|_2^2 \\
& \qquad \quad \le 
2 \sqrt{N_t^c} \eta \| \Delta_{\delta}^{\text{flat}} \|_2 + 
\| \nabla_c \hat{\mathcal{L}}_{\tilde{S}_t}(\overline{c}^0, [\delta^*_i]_t) \|_2 \cdot \| \Delta_c \|_2 \\
& \qquad \quad \le 
\left( 2 \sqrt{N_t^c} \eta + \sqrt{\frac{1}{t}} \| \nabla_c \hat{\mathcal{L}}_{\tilde{S}_t}(\overline{c}^0, [\delta^*_i]_t) \|_2 \right) 
\left( \| \Delta_{\delta+}^{\text{flat}} \|_2 +\sqrt{t} \| \Delta_c \|_2 \right) \\
& \qquad \quad \le 
\left( 2 \sqrt{N_t^c} \eta + \sqrt{16 D \log \tfrac{t}{\rho}} \right) 
\sqrt{2( \| \Delta_{\delta+}^{\text{flat}} \|_2^2 + t \| \Delta_c \|_2^2)},
\end{aligned}
\]
where the last inequality holds by Eq. \ref{eq: c_gradient_bound} and Cauchy–Schwarz inequality. This implies

\[
\sqrt{2} D(4e^{4B} + 2e^B) \left( 2 \sqrt{N_t^c} \eta + \sqrt{16 D \log \tfrac{1}{\rho}} \right)
\ge 
\frac{t \overline{m}_t \| \Delta_c \|_2^2 + m_t^{\downarrow} \| \Delta_{\delta}^{\text{flat}} \|_2^2}
{ \sqrt{ t \| \Delta_c \|_2^2 + \| \Delta_{\delta+}^{\text{flat}} \|_2^2 } }
\]

\[
\ge 
\sqrt{ m_t^{\downarrow} } \cdot 
\sqrt{ t \| \Delta_c \|_2^2 + \| \Delta_{\delta+}^{\text{flat}} \|_2^2 }.
\]

Note that within $t$ steps, by Chernoff bound on Bernoulli distribution (Lemma \ref{lemma: chernoff_bound}) of corruption $B(\varepsilon)$,  
the number of corruption time satisfies:
\[
\mathbb{P} \left( N_t^c \ge \varepsilon t + 2 \sqrt{\varepsilon t \log\left( \tfrac{t}{\rho} \right)} \right) 
\le \left( \tfrac{\rho}{t} \right)^2
\]
\[
\Rightarrow \quad N_t^c \geq \varepsilon t + 2 \sqrt{ \varepsilon t \log\left( \tfrac{t}{\rho} \right)} 
\quad \text{with high probability } (1 - \tfrac{\rho^2}{t^2}).
\]

Finally, for any $t \ge 2D(16e^{2B} + 8)^2 \log (\tfrac{t}{\rho})$, by setting $\eta \ge \frac{2D + 9D + (2D^2+6) \sqrt{D^2 + 3}}{27}$, we have: 
\[
\begin{aligned}
\| \Delta_c \|_2^2 + \| \Delta_{\delta+}^{\text{flat}} \|_2^2 
& \le 
\frac{2D^2 (4e^{4B} + 2e^B)^2}{t m_t^{\downarrow}} 
\left( 
2 \eta \sqrt{ \varepsilon t + 2 \left(\varepsilon t \log( \tfrac{t}{\rho} ) \right)^{1/2} }
+ \sqrt{16 D \log \tfrac{t}{\rho}} 
\right)^2 \\
& \le 
\frac{4 D^2 (4e^{4B} + 2e^B)^2}{ m_t^{\downarrow} } 
\left( 
4 \eta^2 (\varepsilon + 2 \sqrt{ \frac{\varepsilon}{t} \log ( \tfrac{t}{\rho}) }) 
+ 16 D \log\left( \tfrac{t}{\rho} \right) 
\right)
\end{aligned}
\]
holds with probability at least $1 - (1 + D + 2e^2) \cdot \frac{\rho^2}{t^2}$, which complete the proof
\end{proof}

\begin{lemma}[Strong Convexity]
\label{lemma:convexity_corrupt_loss}
The loss function $\hat{\mathcal{L}}_{\tilde{S}_t}$ is strongly convex in the following sense:

\begin{itemize}
  \item[\textbf{(i)}] $\hat{\mathcal{L}}_{\tilde{S}_t}$ is strongly convex with respect to $c$ at $(\overline{c}^0, [\delta^*_i]_t)$:
  \[
  \hat{\mathcal{L}}_{\tilde{S}_t}(\overline{c}^0 + \Delta_c, [\delta^*_i]_t) - \hat{\mathcal{L}}_{\tilde{S}_t}(\overline{c}^0, [\delta^*_i]_t) - \nabla_c \hat{\mathcal{L}}_{\tilde{S}_t}(\overline{c}^0, [\delta^*_i]_t)^\top \Delta_c
  \ge \frac{t \overline{m}_t}{D(4e^{4B} + 2e^B)} \| \Delta_c \|_2^2
  \]

  \item[\textbf{(ii)}] $\hat{\mathcal{L}}_{\tilde{S}_t}$ is strongly convex with respect to $\delta$ at $(\hat{c}^{\text{(QE)}}, [\delta^*_i]_t)$:
  \[
  \hat{\mathcal{L}}_{\tilde{S}_t}(\hat{c}^{\text{(QE)}}, [\delta^*_i]_t + \Delta_\delta) - \hat{\mathcal{L}}_{\tilde{S}_t}(\hat{c}^{\text{(QE)}}, [\delta^*_i]_t) - \nabla_\delta \hat{\mathcal{L}}_{\tilde{S}_t}(\hat{c}^{\text{(QE)}}, [\delta^*_i]_t)^\top \Delta_\delta^{\text{flat}}
  \ge \frac{ m_t^{\downarrow} }{ D(4e^{4B} + 2e^B) } \| \Delta_\delta^{\text{flat}} \|_2^2
  \]
\end{itemize}
\end{lemma}
The proof is provided in Appendix \ref{sec:proof_convexity_corrupt_loss}.

\begin{lemma}
\label{lemma:eta_setting}
Let 
\[
\eta \ge \frac{2D + 9p + 2p^2(D^2 + 3) + 6(D^2 + 3)}{27},
\]
then the following inequality holds:
\begin{equation}
\begin{aligned}
\lambda \sum_{i=1}^{t} \| \delta_i^* \|_2 - \lambda \sum_{i=1}^{t} \| \hat{\delta}_i \|_2 
& - \nabla_c \hat{\mathcal{L}}_{\tilde{S}_t}(\overline{c}^0, [\delta^*_i]_t)^\top \Delta_c 
- \nabla_{\delta} \hat{\mathcal{L}}_{\tilde{S}_t}(\hat{c}^{\text{(QE)}_t}, [\delta^*_i]_t)^\top \Delta_{\delta}^{\text{flat}} \\
& \le 
2 \sqrt{N_t^c} \eta \| \Delta_{\delta+}^{\text{flat}} \|_2 
+ \| \nabla_c \hat{\mathcal{L}}_{\tilde{S}_t}(\overline{c}^0, [\delta^*_i]_t) \|_2 \cdot \| \Delta_c \|_2.
\end{aligned}
\end{equation}
\end{lemma}
The proof is provided in Appendix \ref{sec:proof_lemma_eta_setting}.

\subsection{Proof of Lemma \ref{lemma:convexity_corrupt_loss}}
\label{sec:proof_convexity_corrupt_loss}
\begin{proof}

Without loss of generality, we fixed the total number of sample as $t$ and omit the index of $t$ for the simplicity.

\textbf{(i)} Strongly Convex w.r.t $c$ at $(\overline{c}^0, [\delta^*_i]_t)$

We follow some results of QE preference estimator in uncorrupted case for proof. Specifically, let 
$$p_k^{(i,j)} = \frac{\mathds{1}_{k}^{(i,j)} \exp( c_k + \delta_{i,k} ) }{\sum_{k' \in S^{(i,j)}} \exp( c_{k'} + \delta_{i,k'} )}$$
denote the probability for selection in $S^{i, j}$. The Hessian with respect to $c$ can be written as:
\[
H_c(c, [\delta]_t) = \sum_{i=1}^{t} \sum_{j=1}^{m_i} H^{(i, j)},
\quad \text{where}
\]
\[
H^{(i, j)} = \frac{1}{2} \sum_{j, j' \in S^{i, j}} (e_k - e_{k'})(e_k - e_{k'})^\top p_k^{(i, j)} p_{k'}^{(i, j)}
= \operatorname{diag}(p^{(i, j)}) - p^{(i, j)} (p^{(i, j)})^\top.
\]

By Lemma \ref{lemma:xi_hessian}, we have 
\[
\xi_2(H_c(c, [\delta]_t)) 
\geq 
\frac{t \overline{m}_t}{D(4e^{4B} + 2e^{2B})}.
\]
Thus for any $u \in \Theta_c^0$, we have
\[
u^{\top} H_c(c, [\delta]_t) u 
\geq
\xi_2(H_c(c, [\delta]_t)) \Vert u \Vert_2^2 
\geq
\frac{t \overline{m}_t}{D(4e^{4B} + 2e^{2B})} \Vert u \Vert_2^2,
\]
which implies the strong convexity of $\hat{\mathcal{L}}_{\tilde{S}_t}$ w.r.t $c$ at $(\overline{c}^0, [\delta^*_i]_t)$.

\textbf{(ii)} Strongly Convex w.r.t $\delta$ at $(\hat{c}^{\text{(QE)}}, [\delta^*_i]_t)$

Since each ranking sample $\sigma_i$ is only dependent on the model of $PL(c+\delta_i)$, we have the gradient $\nabla_\delta \hat{\mathcal{L}}_{\tilde{S}_t}(c, [\delta^*_i]_t)$ splits \emph{sample-by-sample}. In this case, the Hessian $H_\delta(c, [\delta]_t)$ is \emph{block-diagonal}, i.e., $D \times D$ blocks for each ranking sample $\sigma_i$. And thus, we only need to check the convexity \emph{sample-by-sample}.

Specifically, for each sample $\sigma_i$, the Hessian w.r.t $\delta_i$ is 
\[
H_{\delta_i}^{(i)}(c, [\delta]_t) = \sum_{j=1}^{m_i} H^{(i,j)}, \quad \text{where} \quad
H^{(i,j)} = \text{diag}(p^{(i, j)}) - p^{(i, j)} (p^{(i, j)})^\top \in \mathbb{R}^{D \times D}.
\]

And it gives the final Hessian w.r.t $[\delta_i]_t$ as:
\[
H_{\delta}(c, [\delta]_t) = 
\begin{bmatrix}
H_{\delta_1}^{(1)} & & & \\
 & H_{\delta_1}^{(1)} & & \\
 & & \ddots & \\
 & & & H_{\delta_t}^{(t)} \\
\end{bmatrix}
\in \mathbb{R}^{Dt \times Dt}.
\]

Recall that each $\delta_i$ is induced by swaps, for any $i \in [t]$, we have $\delta_i \perp \boldsymbol{1}$, and by Lemma \ref{lemma:xi_hessian}, we have 
\begin{equation}
\label{eq: xi_hessian_corrupt}
\xi_2(H_{\delta_i}^{(i)}(c, [\delta]_t)) 
\geq 
\frac{m_i}{D(4e^{4B} + 2e^{2B})}.
\end{equation}

Thus for any $\delta_i \in \Theta_{\delta}$ and $U = \begin{bmatrix}
u_1  \\
\vdots \\
u_t
\end{bmatrix} \in \mathbb{R}^{D \cdot t}$ with $u_i \in \mathbb{R}^D$ and $u_i \perp \boldsymbol{1}, \forall i \in [t]$, we have
\[
u^{\top} H_{\delta_i}^{(i)} u 
\geq
\xi_2(H_{\delta_i}^{(i)}) \Vert u \Vert_2^2 
\geq
\frac{m_i}{D((4e^{4B} + 2e^{2B})} \Vert u \Vert_2^2,
\]
which implies the strong convexity of $H_{\delta_i}^{(i)}$, and for the final Hessian, we have 
\[
\begin{aligned}
U^{\top} H_{\delta}(c, [\delta]_t) U 
& = 
[u_1^{\top} H_{\delta_1}^{(1)}, u_2^{\top} H_{\delta_2}^{(2)}, ..., u_t^{\top} H_{\delta_t}^{(t)} ] U 
= 
\sum_{i=1}^{t} u_i H_{\delta_i}^{(i)} u_i \\
& \geq 
\sum_{i=1}^{t} \xi_2 \left(H_{\delta_i}^{(i)} \right) \Vert u_i \Vert_2^2
\geq 
\min_{i \in [t]} \xi_2 \left(H_{\delta_i}^{(i)} \right) \sum_{i=1}^{t} \Vert u_i \Vert_2^2 
= 
\frac{m_t^{\downarrow}}{D(4e^{4B} + 2e^{2B})} \Vert U \Vert_2^2,
\end{aligned}
\]
where the last equality holds by Eq. \ref{eq: xi_hessian_corrupt} and $\sum_{i=1}^{t} \Vert u_i \Vert_2^2 = \Vert U \Vert_2^2$. 
Above result implies the strong convexity of $\hat{\mathcal{L}}_{\tilde{S}_t}$ w.r.t $[\delta_i]_t$ at $(\overline{c}^0, [\delta^*_i]_t)$.
\end{proof}

\subsection{Proof of Lemma \ref{lemma:eta_setting}}
\label{sec:proof_lemma_eta_setting}

We first present a useful lemma that will be used for proof.
\begin{lemma}
\label{lemma:eta_gradient_corrupt_bound}
For any ranking sample $\delta_i$ under any corrupted Plackett-Luce model $\mathrm{PL}(\delta_i)$, we have:
\[
\left\| \nabla_{\delta_i}^{(i)} \hat{\mathcal{L}}_{\tilde{S}_t}(\overline{c}^0, [\delta^*]_t) \right\|_2 
\le 
\frac{2D^3 + 9p + 2p^2 \sqrt{D^2 + 3} + 6 \sqrt{D^2 + 3}}{27}
\]
\end{lemma}
The proof is provided in Appendix \ref{sec:proof_eta_gradient_corrupt_bound}

\begin{proof}[Proof of Lemma \ref{lemma:eta_setting}]

Let $\mathcal{T}_t^c$ denote the set of time steps where ranking is corrupted within $t$ samples, and $|\mathcal{T}_t^c| = N_t^c$, let $\Delta_{\delta+}^{\text{flat}} = \Delta_{\delta}^{\text{flat}}[\mathcal{T}_t^c]$ denote the gap of estimated $\delta$ over corrupted samples.
we have

\begin{equation}
\begin{aligned}
\sum_{i=1}^{t} \| \delta_i^* \|_2 - \sum_{i=1}^{t} \| \hat{\delta}_i \|_2 
&= \sum_{i \in \mathcal{T}_t^c} \left( \| \delta_i^* \|_2 - \| \hat{\delta}_i \|_2 \right) 
+ \sum_{i \in [t] \setminus \mathcal{T}_t^c} \left( \| \delta_i^* \|_2 - \| \hat{\delta}_i \|_2 \right) \\
&= \sum_{i \in \mathcal{T}_t^c} \left( \| \delta_i^* \|_2 - \| \hat{\delta}_i \|_2 \right) 
- \sum_{i \in [t] \setminus \mathcal{T}_t^c} \| \hat{\delta}_i \|_2 \\
&\le \sum_{i \in \mathcal{T}_t^c} \| \delta_i^* - \hat{\delta}_i \|_2 
- \sum_{i \in [t] \setminus \mathcal{T}_t^c} \| \hat{\delta}_i \|_2 \\
&\le \sqrt{N_t^c} 
\Big( \sum_{i \in \mathcal{T}_t^c} \| \delta_i^* - \hat{\delta}_i \|_2^2 \Big)^{1/2}
- \sum_{i \in [t] \setminus \mathcal{T}_t^c} \| \hat{\delta}_i \|_2 \\
&= \sqrt{N_t^c} \cdot \| \Delta_{\delta+}^{\text{flat}} \|_2 
- \sum_{i \in [t] \setminus \mathcal{T}_t^c} \| \hat{\delta}_i \|_2
\end{aligned}
\end{equation}

On the other hand, we have 
\begin{equation}
\begin{aligned}
\nabla_{\delta} & \hat{\mathcal{L}}_{\tilde{S}_t}(\overline{c}^0, [\delta^*_i]_t)^\top \Delta_{\delta} \\
&= \sum_{i \in \mathcal{T}_t^c} 
\left| \nabla_{\delta_i}^{(i)} \hat{\mathcal{L}}_{\tilde{S}_t}(\overline{c}^0, [\delta^*_i]_t)^\top (\delta_i^* - \hat{\delta}_i) \right|
+ \sum_{i \in [t] \setminus \mathcal{T}_t^c} 
\left| \nabla_{\delta_i}^{(i)} \hat{\mathcal{L}}_{\tilde{S}_t}(\overline{c}^0, [\delta^*_i]_t)^\top (\delta_i^* - \hat{\delta}_i) \right| \\[1ex]
&\ge 
- \sum_{i \in \mathcal{T}_t^c} 
\left\| \nabla_{\delta_i}^{(i)} \hat{\mathcal{L}}_{\tilde{S}_t}(\overline{c}^0, [\delta^*_i]_t) \right\|_2 
\cdot \left\| \delta_i^* - \hat{\delta}_i \right\|_2
- \sum_{i \in [t] \setminus \mathcal{T}_t^c} 
\left\| \nabla_{\delta_i}^{(i)} \hat{\mathcal{L}}_{\tilde{S}_t}(\overline{c}^0, [\delta^*_i]_t) \right\|_2 
\cdot \left\| \hat{\delta}_i \right\|_2 \\[1ex]
& \geq
- \max_{i \in \mathcal{T}_t^c} 
\left\| \nabla_{\delta_i}^{(i)} \hat{\mathcal{L}}_{\tilde{S}_t}(\overline{c}^0, [\delta^*_i]_t) \right\|_2 
\cdot \sum_{i \in \mathcal{T}_t^c} \left\| \delta_i^* - \hat{\delta}_i \right\|_2
- \max_{i \in [t] \setminus \mathcal{T}_t^c} 
\left\| \nabla_{\delta_i}^{(i)} \hat{\mathcal{L}}_{\tilde{S}_t}(\overline{c}^0, [\delta^*_i]_t) \right\|_2 
\cdot \sum_{i \in [t] \setminus \mathcal{T}_t^c} \left\| \hat{\delta}_i \right\|_2 \\[1ex]
&\ge 
- \max_{i \in \mathcal{T}_t^c} 
\left\| \nabla_{\delta_i}^{(i)} \hat{\mathcal{L}}_{\tilde{S}_t}(\overline{c}^0, [\delta^*_i]_t) \right\|_2 
\cdot \sqrt{N_t^c} \cdot \left\| \Delta_{\delta_t}^{\text{(flat)}} \right\|_2
- \max_{i \in [t] \setminus \mathcal{T}_t^c} 
\left\| \nabla_{\delta_i}^{(i)} \hat{\mathcal{L}}_{\tilde{S}_t}(\overline{c}^0, [\delta^*_i]_t) \right\|_2 
\cdot \sum_{i \in [t] \setminus \mathcal{T}_t^c} \left\| \hat{\delta}_i \right\|_2
\end{aligned}
\end{equation}

Combining above results together yields
\begin{equation}
\begin{aligned}
\eta \sum_{i=1}^{t} \| \delta_i^* \|_2 
& - \eta \sum_{i=1}^{t} \| \hat{\delta}_i \|_2
- \nabla_c \hat{\mathcal{L}}_{\tilde{S}_t}(\overline{c}^0, [\delta^*_i]_t)^\top \Delta_c
- \nabla_{\delta} \hat{\mathcal{L}}_{\tilde{S}_t}(\hat{c}, \delta^*)^\top \Delta_{\delta}^{\text{flat}} \\
& \le 
\sqrt{N_{t}^c} 
\left( 
\max_{i \in \mathcal{T}_t^c} 
\left\| \nabla_{\delta_i}^{(i)} \hat{\mathcal{L}}_{\tilde{S}_t}(\overline{c}^0, [\delta^*_i]_t) \right\|_2 + \eta 
\right) 
\left\| \Delta_{\delta+}^{\text{flat}} \right\|_2 \\
\qquad & + \left(
\max_{i \in [t] \setminus \mathcal{T}_t^c}
\left\| \nabla_{\delta_i}^{(i)} \hat{\mathcal{L}}_{\tilde{S}_t}(\overline{c}^0, [\delta^*_i]_t) \right\|_2 - \eta
\right) 
\sum_{i \in [t] \setminus \mathcal{T}_t^c} \| \hat{\delta}_i \|_2 \\
\qquad & + \left\| \nabla_c \hat{\mathcal{L}}_{\tilde{S}_t}(\overline{c}^0, [\delta^*_i]_t) \right\|_2 \cdot \| \Delta_c \|_2 &
\end{aligned}
\end{equation}

By Lemma \ref{lemma:eta_gradient_corrupt_bound}, setting
\[
\eta \ge \frac{2D^3 + 9p + 2p^2 \sqrt{D^2 + 3} + 6 \sqrt{D^2 + 3}}{27}
\ge \max_{i \in [t]} \left\| \nabla_{\delta_i}^{(i)} \hat{\mathcal{L}}_{\tilde{S}_t}(\overline{c}^0, [\delta^*_i]_t) \right\|_2
\]
completes the proof.

\end{proof}

\subsection{Proof of Lemma \ref{lemma:eta_gradient_corrupt_bound}}
\label{sec:proof_eta_gradient_corrupt_bound}
\begin{proof}

For each sample $\delta_i$, the gradient with respect to $\delta_i$ is:
\[
\nabla_{\delta_i}^{(i)} \tilde{\mathcal{L}}_{\tilde{S}_t}(c, [\delta]_t) 
= - \sum_{j=1}^{m_i} \frac{2}{D} \left[ \log\left( 
\frac{\exp(c_{i}(j) + \delta_{i}(j))}{\sum_{k=j}^{D} \exp(c_{i}(k) + \delta_{i}(k))} 
\right) \right].
\]

Using the gradient of softmax log-likelihood, we get:
\[
\nabla_{\delta_i}^{(i)} \tilde{\mathcal{L}}_{\tilde{S}_t}(c, [\delta]_t) 
= - \sum_{j=1}^{m_i} \sum_{i' = i+1}^{D} \left( 
\frac{\exp(c_i(j) + \delta_i(j))}{\sum_{k=i}^{D} \exp(c_i(k) + \delta_i(k))} 
(e_{i}(j) - e_{i}(j')) 
\right).
\]

From this, we can bound each coordinate:
\[
\left| \nabla_{\delta_i}^{(i)} \tilde{\mathcal{L}}_{\tilde{S}_t}(c, [\delta]_t)(j) \right| 
\le 
\begin{cases}
1 - \frac{D}{\sum_{k=j}^{D} \exp(c_i(k) + \delta_i(k))} \exp(c_i(j) + \delta_i(j)) 
\le 1, & j \le m_i \\
\sum_{j'=1}^{m_i} 
\frac{\exp(c_i(j) + \delta_i(j))}{\sum_{k=j}^{D} \exp(c_i(k) + \delta_i(k))} 
\le m_i, & j > m_i
\end{cases}
\]

This implies the total $\ell_2$ norm is bounded by:
\[
\left\| \nabla_{\delta_i}^{(i)} \tilde{\mathcal{L}}_{\tilde{S}_t}(c, [\delta]_t) \right\|_2 
\le \sqrt{m_i + m_i^2 (D - m_i)} 
\le \frac{2D^3 + 9p + 2p^2 \sqrt{D^2 + 3} + 6 \sqrt{D^2 + 3}}{27},
\]
where the final inequality is the closed-form solution of
\[
\max_{m_i \in [0, D]} \sqrt{m_i + m_i^2 (D - m_i)}.
\]
    
\end{proof}

\section{Proof of Lemma \ref{lemma:ucb_c_BF_corrupt}}
\label{sec:pf_lemma_ucb_c_BF_corrupt}

Before the proof, we first restate Lemma 5 with a detained version.

\setcounter{templemma}{\value{lemma}} % 保存当前值

\setcounter{lemma}{4} % 想让下一个 lemma 显示为3
\begin{lemma}[Detailed Version]
Following the settings in Lemma \ref{lemma:qe_ucb_corrupt}, for any user $n \in [N]$, any $\rho \in (0,1)$, with probability at least $1-\frac{(1+D+2e^2) \rho^2 \pi^2}{6}$, for all sufficiently large $t \geq 2D(16e^{2B}+8)^2 \log(\frac{T}{\rho})$, the estimated preference of Eq.\ref{eq:c_BF} satisfies
\[
\begin{aligned}
\Vert \boldsymbol{\hat{c}}^{(\text{HB})}_{n,t} - \overline{\boldsymbol{c}}_n \Vert_{ \boldsymbol{V}_{n,t}}
& \leq
% \textstyle 
\alpha \frac{BD}{2} \sqrt{D \log\big(\frac{1+\alpha D(t-1) \ (\lambda+1)}{\delta} \big) + \log(\frac{1+\lambda}{\lambda \delta})} \\
& \quad +
\sqrt{1-\alpha}
\Bigg( \frac{D(8e^{4B}+4e^{2B})}{\sqrt{m_t^{\downarrow}}} \sqrt{ \eta^2 \Big( \epsilon + 2 \big( \frac{\epsilon \log(t/\rho)}{t} \big) ^{1/2} \Big) + \frac{4D \log(t/\rho)}{t} } + B\sqrt{2 \lambda} \Bigg)
.
\end{aligned}
\]
\end{lemma}

\setcounter{lemma}{\value{templemma}} % 恢复

\begin{proof}

The proof can be greatly simplified by applying the results in our previous proofs. 

Specifically, by Eq. \ref{eq: proof_bf_error} in uncorrupted case, we have
\[
\| \hat{c}_t^{\text{(HB)}} - \overline{c} \|_{V_t} 
\le 
(1 - \alpha) \underbrace{\Big\| V_t^{-1} U_\lambda (\hat{c}_t^{\text{(QE)}} - \overline{c}) \Big\|_{V_t} }_{A}
+ \alpha  \underbrace{\Big\| V_t^{-1} \sum_{\tau=1}^{t-1} r_{a_\tau, \tau} r_{a_\tau, \tau}^{\top} \xi_\tau \Big\|_{V_t}}_{B},
\]
where the term $B$ is exactly same with the uncorrupted case since the independence with the corruption $\epsilon$. For term 
$A$, by the derivation in Eq. \ref{eq:proof_QE_BF_error_A}, we have 
\[
A \le \frac{1}{\sqrt{1 - \alpha}} \left\| \hat{c}_t^{\text{(QE)}} - \overline{c} \right\|_{U_\lambda}
\leq
\frac{1}{\sqrt{1 - \alpha}}
\sqrt{ (\hat{c}_t^{\text{(QE)}} - \overline{c}^0)^\top U (\hat{c}_t^{\text{(QE)}} - \overline{c}^0) } + B \sqrt{2 \lambda D }.
\]

Note the only difference is the new QE-estimator error bound induced by the $\epsilon$-corruption model, as has been characterized in Lemma \ref{lemma:ucb_c_BF_corrupt}. Pluging the new QE-estimator error bound of Lemma \ref{lemma:ucb_c_BF_corrupt} yields
\[
\begin{aligned}
& \left\| \hat{c}_t^{(\text{HB})} - \bar{c} \right\|_{V_t} 
\leq 
\alpha \frac{BD}{2} \sqrt{D \log\big(\frac{1+\alpha D(t-1) \ (\lambda+1)}{\delta} \big) + \log(\frac{1+\lambda}{\lambda \delta})} \\
& \qquad + 
\sqrt{1-\alpha}
\Big( \frac{D(8e^{4B}+4e^{2B})}{\sqrt{m_t^{\downarrow}}} \sqrt{ \eta^2 \big( \epsilon + 2 \big( \frac{\epsilon \log(t/\rho)}{t} \big) ^{1/2} \big) + \frac{4D \log(t/\rho)}{t} } + B\sqrt{2 \lambda D} \Big)
,
\end{aligned}
\]
completing the proof.

\end{proof}

\section{Proof of Theorem \ref{thm:user_spe_reg_corrupt}}
\label{sec:pf_thm_user_spe_reg_corrupt}

Before the detailed proofs, we first restate Theorem 2 with a detained version and show how the optimal choices of balance-weight $\alpha$ and regularization $\lambda$ are derived.

\setcounter{templemma}{\value{theorem}} % 保存当前值

\setcounter{theorem}{3} % 想让下一个 lemma 显示为3

\begin{theorem}[Detailed Version]
For Algorithm \ref{alg:MO_PQUCB} with group-wise Lasso MLE, by setting $\beta_t$ as the UCB of $\boldsymbol{\hat{c}}^{n(\text{BF})}_t$ in Lemma \ref{lemma:qe_ucb_corrupt}, setting $\eta$ as Lemma \ref{lemma:qe_ucb_corrupt}, 
then for any user $n \in [N]$, any sparse corruption rate $\epsilon \in [0,1]$, with probability at least $1-\frac{(2+D+2e^2) \rho^2 \pi^2}{6}$, we have
\[
R^{n}(T) \leq 
\Psi^{n(\hat{c})}(T)
+
\Psi^{n(\hat{r})}(T)
+
M, 
\quad \text{where }
\]
\[
% \textstyle
M = \max
\Big\{
\frac{4D^2 K \alpha^2}{\nu^4_{r \downarrow}} \log (\frac{T}{\rho})
,
\frac{4 \epsilon \eta^4}{\alpha^2} \log(\frac{T}{\rho})
, 
\frac{4 D}{\alpha} \log(\frac{T}{\rho})
,
2D(16e^{2B}+8)^2 \log(\frac{T}{\rho})
\Big\}
\]
\[
% \textstyle
\Psi^{n(\hat{c})}(T)= 
O \Big( 
\Big( \frac{D^{2} \sqrt{D + \eta \epsilon/\alpha}}{\sqrt{m^{\downarrow}_T}}
+ BD\sqrt{\frac{ \lambda}{\alpha} }
+ 
BD^{\frac{3}{2}} \sqrt{\alpha \log(\frac{T}{\rho}) + \alpha\log(\frac{1+\lambda}{\lambda \rho})} \Big) \sqrt{T \log(\frac{\alpha T}{\lambda})} \Big) 
;
\]
\[
% \textstyle
\Psi^{n(\hat{r})}(T) =
O \Big( 
\sqrt{\frac{K D^{4}}{m^{\downarrow}_T \lambda}} \log^{2} T
+
\Big( BD+
\eta \sqrt{\frac{D^3 \epsilon }{m^{\downarrow}_T \lambda}}
+
\frac{\alpha BD^{\frac{3}{2}}}{\sqrt{\lambda}} \sqrt{\log(\frac{T}{\rho}) + \log(\frac{1+\lambda}{\lambda \rho})} \Big) \sqrt{ KT \log (\frac{T}{\rho})}
\Big).
\]
\end{theorem}
\setcounter{theorem}{\value{templemma}} % 恢复

\begin{remark}
\label{remark:user_spe_reg_corrupt}
The parameters $\alpha$ and $\lambda$ can be optimally tuned. By setting $\alpha = \lambda = O(\epsilon + \log^{-1}(T/\rho))$, we have
\[
% \textstyle
R^{n}(T) \leq 
O\bigg(
\Big(D^2\sqrt{\frac{D+\eta}{m_T^{\downarrow}}} + BD + \eta \sqrt{\frac{D^3 K}{m_T^{\downarrow}}} \Big) \sqrt{T \log(T)}
+
B \sqrt{D^3 K (\epsilon + \frac{1}{\log(T/\rho)})T} \cdot \log (\frac{T}{\rho})
\bigg)
\]
\end{remark}

\begin{proof}

The proof can be greatly simplified by applying the results in our previous proofs. 
Specifically, by Eq. \ref{eq: trunc_regret} and Eq. \ref{eq: hidden_regret_c}, we have the regret as following formula.

\begin{equation}
\begin{aligned}
\label{eq: regret_corrupt}
R(T) 
& \leq
O(M)
+
\underbrace{
\sum_{t=M+1}^{T} 
\min \left( 2 B_{a_t,t}^{c}, BD \right)
}_{R_{M+1:T}^{\tilde{\boldsymbol{c}}}}+
\underbrace{
\sum_{t=M+1}^{T} 2 \Vert\boldsymbol{\overline{c}}\Vert_1 \sqrt{\frac{ \log(t/\alpha)}{N_{a,t}}} 
}_{R_{M+1:T}^{\tilde{\boldsymbol{r}}}} \\
& \leq 
O(M) 
+
2\sqrt{2} 
\underbrace{
\sum_{t=M+1}^{T}
\min \left( \beta_t
\Vert \boldsymbol{\mu}_{a_t} \Vert_{\mathbb{E}\left[\boldsymbol{V}_{t}\right]^{-1}},
\frac{BD}{2\sqrt{2}}  \right)
}_{\text{Reg}^{(1)}}
+
4 \underbrace{
\sum_{t=M+1}^{T}
\beta_t \sqrt{\frac{D}{\lambda}} \gamma_{a_t,t}
}_{\text{Reg}^{(2)}} +
R_{M+1:T}^{\tilde{\boldsymbol{r}}}.
\end{aligned}
\end{equation}

Note that the only difference is the new $\beta_t$ induced by additional corruption component, leading to different $\text{Reg}^{(1)}$ and $\text{Reg}^{(2)}$. In the following, we show the derivation of these two terms while omit others since they are exactly the same with uncorrupted case.

\textbf{Step-1.1. (Upper-Bound over $\text{Reg}_{t}^{(1)}$)}

Plugging the $\beta_t = 
\sqrt{1-\alpha}
\Big( \frac{D(8e^{4B}+4e^{2B})}{\sqrt{m_t^{\downarrow}}} \sqrt{ \eta^2 \big( \epsilon + 2 \big( \frac{\epsilon \log(t/\rho)}{t} \big) ^{1/2} \big) + \frac{4D \log(t/\rho)}{t} } + B\sqrt{2 \lambda D})
+
\alpha \frac{BD}{2} \sqrt{D \log\big(\frac{1+\alpha D(t-1) \ (\lambda+1)}{\delta} \big) + \log(\frac{1+\lambda}{\lambda \delta})} 
$ (defined in Lemma \ref{lemma:ucb_c_BF_corrupt}), and by Cauchy–Schwarz inequality, we have

\begin{equation}
\label{eq: Reg_1}
\begin{aligned}
& \text{Reg}^{(1)}
\leq 
D (32 e^{4B} + 16 e^{2B}) \sqrt{ \frac{ D(1-\alpha) }{m_t^{\downarrow}} } 
\sqrt{\sum_{t=M+1}^{T} 1 \cdot \underbrace{\sum_{t=M+1}^{T} \min \left(\frac{\log(t/\rho)}{t} \Vert \boldsymbol{\mu}_{a_t} \Vert^2_{\mathbb{E}\left[\boldsymbol{V}_{t} \right]^{-1}}, \frac{B^2 D^2}{8} \right)}_{I_1}} \\
& +
D (32 e^{4B} + 16 e^{2B}) \sqrt{ \frac{ (1-\alpha) }{m_t^{\downarrow}} } 
\sqrt{\sum_{t=M+1}^{T} 1 \cdot \underbrace{\sum_{t=M+1}^{T} \min \left( 2 \eta^2 \sqrt{\frac{\epsilon \log(t/\rho)}{t}} \Vert \boldsymbol{\mu}_{a_t} \Vert^2_{\mathbb{E}\left[\boldsymbol{V}_{t} \right]^{-1}}, \frac{B^2 D^2}{8} \right)}_{I_{*}}}\\
& +
\left(\eta D (32e^{4B} + 16e^{2B}) \sqrt{\frac{(1-\alpha)\epsilon}{m_T^{\downarrow}}} + B\sqrt{2\lambda(1-\alpha)} + \alpha \frac{BD}{2} \sqrt{D \log\big(\frac{1+\alpha DT / (\lambda+1)}{\rho} \big) + \log(\frac{1+\lambda}{\lambda \rho})} \right) \\
& \quad \quad 
\times \sqrt{\sum_{t=M+1}^{T} 1 \cdot \underbrace{\sum_{t=M+1}^{T} \min \left( \Vert \boldsymbol{\mu}_{a_t} \Vert_{\mathbb{E}\left[\boldsymbol{V}_{t}\right]^{-1}}^2, \frac{B^2 D^2}{8} \right)}_{I_2}}.
\end{aligned}
\end{equation}

Note for $I_1$ and $I_2$, we have obtained the results in previous uncorrupted case. Next we show the bound of term $I_*$.
When  $\frac{4 \epsilon \eta^4  \log\frac{t}{\rho}}{\alpha^2} \leq t$, we have $ 2 \eta^2 \sqrt{\frac{\epsilon \log \frac{t}{\rho}}{t}} \leq \alpha$, which gives
\[
\begin{aligned}
I_*
& = \sum_{t=1}^T \min\left( 2 \eta^2 \sqrt{ \frac{\epsilon \log(t/\rho)}{t}} \|x\|^2_{V_t^{-1}}, \frac{B^2 D^2}{8} \right)\\
& \leq
\sum_{t=1}^T \min\left( \alpha \|x\|^2_{V_t^{-1}}, \frac{B^2 D^2}{8} \right) \\
& \leq
c_0 D \log \left( 1 + \frac{\alpha}{\lambda (1-\alpha)}(T-M) \right),
\end{aligned}
\]

where the last inequality holds by Eq. \ref{eq: regret_I_1}. Putting together, we have for $t \ge \max\{ \tfrac{\log (T/\rho)}{\alpha}, \frac{4 \epsilon \eta^4  \log\frac{T}{\rho}}{\alpha^2} \leq t \}$,

\begin{equation}
\label{eq: Reg_1_result_corrupt}
\begin{aligned}
 \text{Reg}^{(1)}
& \leq 
D (32 e^{4B} + 16 e^{2B}) (D+\sqrt{D}) \sqrt{ \frac{ c_0 (1-\alpha) }{m_T^{\downarrow}} } \sqrt{ (T-M) \log \left( 1 + \frac{\alpha}{\lambda (1-\alpha)}(T-M) \right) } \\
& \quad +
\eta D (32e^{4B} + 16e^{2B}) \sqrt{\frac{(1-\alpha)\epsilon}{\alpha m_T^{\downarrow}}} \sqrt{ (T-M) \log \left( 1 + \frac{\alpha}{\lambda (1-\alpha)}(T-M) \right) } \\
& \quad +
\left( 
B\sqrt{\frac{2\lambda(1-\alpha) c_0 D}{\alpha}} +  \frac{BD}{2} \sqrt{\alpha c_0 D} \sqrt{D \log\big(\frac{1+\alpha DT / (\lambda+1)}{\rho} \big) + \log(\frac{1+\lambda}{\lambda \rho})} \right) \\
& \quad \quad \quad \times 
\sqrt{ (T-M) \log \left( 1 + \frac{\alpha}{\lambda (1-\alpha)}(T-M) \right) }.
\end{aligned}
\end{equation}

\textbf{Step-1.2. (Upper-Bound over $\text{Reg}_{t}^{(2)}$)}

Plugging the new $\beta_t$ (defined in Lemma \ref{lemma:ucb_c_BF_corrupt}), we have

\begin{equation}
\begin{aligned}
\label{eq: Reg_2_result_corrupt}
& \text{Reg}^{(2)}
\leq 
D^2 (32 e^{4B} + 16 e^{2B}) \sqrt{ \frac{(1-\alpha) }{ m_T^{\downarrow} \lambda } } \sum_{t=M+1}^{T} \sqrt{\frac{ \log^2(t/\rho)}{t N_{a_t,t}}}  \\
& + 
D (16 e^{4B} + 8 e^{2B}) \sqrt{ \frac{ D(1-\alpha) }{ m_T^{\downarrow} \lambda } } \eta \sum_{t=M+1}^{T}
\underbrace{\sqrt{\epsilon + 2 \big( \frac{\epsilon \log(t/\rho)}{t} \big) ^{1/2}}}_{A} \sqrt{\frac{ \log(t/\rho)}{N_{a_t,t}}}\\
& +
\left(B\sqrt{2\lambda(1-\alpha)} + \alpha \frac{BD}{2} \sqrt{D \log\big(\frac{1+\alpha DT / (\lambda+1)}{\rho} \big) + \log(\frac{1+\lambda}{\lambda \rho})} \right) 
\sqrt{\frac{D}{\lambda}}
\sum_{t=M+1}^{T} \sqrt{\frac{ \log(t/\rho)}{N_{a_t,t}}}.
\end{aligned}
\end{equation}

Note term $A$ is the additional term introduced by corruption. We can see that $O(1) \leq O(A) \leq O(\frac{\log(t/\rho)}{t})$, thus we have the second term can be bounded as 
\[
O \left(D (16 e^{4B} + 8 e^{2B}) \sqrt{ \frac{ D \epsilon (1-\alpha) }{ m_T^{\downarrow} \lambda } } \eta \sum_{t=M+1}^{T} \sqrt{\frac{ \log(t/\rho)}{N_{a_t,t}}} \right) = O \left(\eta D\sqrt{ \frac{ D \epsilon }{\lambda } } \sqrt{KT \log(\frac{T}{\rho})} \right).
\]

Pluging above result back to Eq. \ref{eq: Reg_2_result_corrupt}, and combining Eq. \ref{eq: Reg_2_result_corrupt}, Eq. \ref{eq: Reg_2_result_corrupt}, 
Eq. \ref{eq: regret_corrupt} and Eq. \ref{eq: Reg_r_result} conclude the proof of $R(T)$.
\end{proof}

\section{Auxiliary Lemmas}

\begin{lemma}[Variant of Lemma 7 in \citep{jun2018adversarial}]
\label{lemma: R_N_relation}
Assume that a bandit algorithm enjoys a sub-linear regret bound, then 
$\mathbb{E}[ N_{i,T} ] = o(T), \forall i \neq a^*$.
\end{lemma}

\begin{lemma}[Theorem 1.8 of \citep{hayeslarge}]
\label{lemma:martingale_central_bd}
Suppose that $S_r = \sum_{t=1}^r X_t$ is a martingale where $X_1, X_2, \dots, X_m$ take values in $\mathbb{R}^n$ and are such that $\mathbb{E}[X_t] = 0$ and $\|X_t\|_2 \le D$ for all $t$, for $D > 0$. Then, for every $x > 0$,
\[
\mathbb{P}\left[ \|S_r\|_2 > x \right] \le 2e^2 \exp\left( -\frac{x^2}{2rD^2} \right).
\]
\end{lemma} 

\begin{lemma}[Hoeffding’s inequality for general bounded random variables, Theorem
2.2.6 of ~\citep{vershynin2018high}]
\label{lemma: Hoeffding}
 Given independent random variables $\{X_1, ..., X_m \}$ where $a_i \leq X_i \leq b_i$ almost surely (with probability 1) we have:
\[
\mathbb{P} \left(\frac{1}{m} \sum_{i=1}^{m} X_i - \frac{1}{m} \sum_{i=1}^{m} \mathbb{E}[X_i] \geq \epsilon \right) \leq \exp \left(\frac{-2 \epsilon^2 m^2}{\sum_{i=1}^{m} (b_i-a_i)^2} \right).
\]
\end{lemma}

\begin{lemma}[Self-Normalized Bound for Vector-Valued Martingales \citep{abbasi2011improved}]
\label{lemma:self_norm_martingale}
Let $\{\mathcal{F}_t\}_{t=0}^\infty$ be a filtration.  
Let $\{\eta_t\}_{t=1}^\infty$ be a real-valued stochastic process such that $\eta_t$ is $\mathcal{F}_t$-measurable and conditionally $R$-sub-Gaussian for some $R \ge 0$, i.e.,
\[
\forall \lambda \in \mathbb{R}, \quad \mathbb{E} \left[ e^{\lambda \eta_t} \mid \mathcal{F}_{t-1} \right] \le \exp \big( \frac{\lambda^2 R^2}{2} \big).
\]
Let $\{X_t\}_{t=1}^\infty$ be an $\mathbb{R}^d$-valued stochastic process such that $X_t$ is $\mathcal{F}_{t-1}$-measurable.  
Assume that $V$ is a $d \times d$ positive definite matrix. For any $t \ge 0$, define
\[
\overline{V}_t = V + \sum_{s=1}^t X_s X_s^\top, \qquad
S_t = \sum_{s=1}^t \eta_s X_s.
\]

Then, for any $\delta > 0$, with probability at least $1 - \delta$, for all $t \ge 0$,
\[
\| S_t \|^2_{\overline{V}_t^{-1}} 
\le 2 R^2 \log \left( \frac{ \det(\overline{V}_t)^{1/2} \det(V)^{-1/2} }{\delta} \right).
\]
\end{lemma}

\begin{lemma}[Determinant of Symmetric PSD Matrices Sum]
\label{lemma: symmetric_det}
Let $\boldsymbol{A} \in \mathbb{R}^{n \times n}$ be a symmetric and positive definite matrix, and $\boldsymbol{B} \in \mathbb{R}^{n \times n}$ be a symmetric and positive (semi-) definite matrix. Then we have 

\[
\emph{det} \left( \boldsymbol{A} + \boldsymbol{B} \right) 
\geq
\emph{det} \left( \boldsymbol{A}\right) + \emph{det} \left( \boldsymbol{B}\right) 
\]
\end{lemma}

\begin{proof}

\begin{equation}
\begin{aligned}
\label{eq: symmetric_det1}
\text{det} \left( \boldsymbol{A} + \boldsymbol{B} \right) 
=
\text{det} \left( \boldsymbol{A}\right) 
\text{det} \left( \boldsymbol{I} + \boldsymbol{A}^{-\frac{1}{2}} \boldsymbol{B} \boldsymbol{A}^{-\frac{1}{2}} \right).
\end{aligned}
\end{equation}

Let $\lambda_1, ..., \lambda_n$ be the eigenvalues of $\boldsymbol{A}^{-\frac{1}{2}} \boldsymbol{B} \boldsymbol{A}^{-\frac{1}{2}}$. 
Since $\boldsymbol{A}^{-\frac{1}{2}} \boldsymbol{B} \boldsymbol{A}^{-\frac{1}{2}}$ is positive (semi-) definite, we have $\lambda_i \geq 0, \forall i \in [n]$, which implies

\begin{equation}
\begin{aligned}
\label{eq: symmetric_det2}
\text{det} \left( \boldsymbol{I} + \boldsymbol{A}^{-\frac{1}{2}} \boldsymbol{B} \boldsymbol{A}^{-\frac{1}{2}} \right)
=
\prod_{i=1}^{n} (1 + \lambda_i) 
\geq 
1 + \prod_{i=1}^{n} \lambda_i
= 
\text{det} (\boldsymbol{I}) + \text{det} \left( \boldsymbol{A}^{-\frac{1}{2}} \boldsymbol{B} \boldsymbol{A}^{-\frac{1}{2}} \right).
\end{aligned}
\end{equation}

Combining Eq.\ref{eq: symmetric_det1} with Eq.~\ref{eq: symmetric_det2} concludes the proof.

\end{proof}

\begin{lemma}[Lemma 10 in \citep{cao2025provably}]
\label{lemma: hidden_sum_reg_c_expectation_bd}
Let $M = \left\lfloor \min \big \{ t^{\prime} \mid t  \nu_{r^{\downarrow}}^2 + \lambda \geq 2D \omega \sqrt{Kt\log \frac{t}{\alpha} }, \forall t \geq t^{\prime} \big \} \right \rfloor$, under the conditions of Claim \ref{claim_1} and Claim \ref{claim_2}, for any $t \geq M+1$, and any $\boldsymbol{\mu} \in \mathbb{R}^{D}$, we have 
\[
\boldsymbol{\mu}^{\top} \boldsymbol{V}_{t-1}^{-1} \boldsymbol{\mu}
\leq
2 \boldsymbol{\mu}^{\top} \mathbb{E}[\boldsymbol{V}_{t-1}]^{-1} \boldsymbol{\mu}.
\]
\end{lemma}

\begin{lemma}[Lemma 12 in \citep{cao2025provably}]
\label{lemma: hidden_sum_reg_c_2}
For any $M \geq 0$, we have
\[
\sum_{t=M+1}^{T} \sqrt{\frac{ \log \left( \frac{t}{\rho} \right)}{N_{a_t,t}}}
\leq
2 \sqrt{ K(T-M) \log \left( \frac{T}{\rho} \right)}.
\]
\end{lemma}

\begin{lemma}
\label{lemma: hidden_sum_reg_c_3}
For any $M \geq 0$, we have
\[
\sum_{t=M+1}^{T} \sqrt{ \frac{ \log^2(t/\rho)}{t N_{a,t}} } = \mathcal{O}\left( \sqrt{K} \cdot \log^2(T/\rho) \right).
\]
\end{lemma}

\begin{proof}

By the Cauchy-Schwarz inequality, we have:
\[
\sum_{t=M+1}^{T} \sqrt{ \frac{ \log^2(t/\rho)}{t N_{a_t,t}} }
\leq \sqrt{ \sum_{t=M+1}^{T} \frac{ \log^2(t/\rho) }{t} } \cdot \sqrt{ \sum_{t=M+1}^{T} \frac{1}{N_{a_t,t}} }.
\]

We bound the first sum term by an integral:
\[
\sum_{t=M+1}^{T} \frac{ \log^2(t/\rho) }{t} \leq \int_{M}^{T} \frac{ \log^2(x/\rho) }{x} \, dx.
\]

Let \( u = \log(x/\rho) \Rightarrow x = \rho e^u, dx = \rho e^u du \). Then:
\[
\int_{M}^{T} \frac{ \log^2(x/\rho) }{x} dx = \int_{\log(M/\rho)}^{\log(T/\rho)} u^2 \, du = \left[ \frac{u^3}{3} \right]_{\log(M/\rho)}^{\log(T/\rho)}.
\]

So:
\[
\sum_{t=M+1}^{T} \frac{ \log^2(t/\rho) }{t} = \mathcal{O}\left( \log^3(T/\rho) \right).
\]

We then bound the second term.
Note that \( a_t \) is the selected arm at time \( t \), and for each time \( t \), \( N_{a_t,t} \) is incremented by 1.
Thus:
\[
\sum_{t=1}^{T} \frac{1}{N_{a_t,t}} \leq \sum_{i=1}^{K} \sum_{s=1}^{N_{i,T}} \frac{1}{s} \leq \sum_{i=1}^{K} \log N_{i,T} \leq K \log T,
\]
which gives:
\[
\sum_{t=M+1}^{T} \frac{1}{N_{a_t,t}} \leq \mathcal{O}(K \log T).
\]

Putting both bounds together:
\[
\sum_{t=M+1}^{T} \sqrt{ \frac{ \log^2(t/\rho)}{t N_{a_t,t}} }
\leq \sqrt{ \mathcal{O}( \log^3(T/\rho) ) } \cdot \sqrt{ \mathcal{O}( K \log T ) }
= \mathcal{O}( \sqrt{K} \cdot \log^2(T/\rho) ).
\]

\end{proof}

\begin{lemma}[Standard Chernoff Bound (Additive Form)]
\label{lemma: chernoff_bound}
Let $S_T$ be the sum of $t$ i.i.d. Bernoulli random variables with mean $\varepsilon$.  
Then for any $\rho > 0$, the following inequality holds:
\[
\mathbb{P} \left( S_t \ge \varepsilon t + \sqrt{ 2 \varepsilon t \log\left( \tfrac{1}{\rho} \right) } \right) \le \rho.
\]
\end{lemma}

\end{document}